\documentclass{article}

\usepackage[main,final]{neurips_2026}

\usepackage[utf8]{inputenc}     
\usepackage[T1]{fontenc}        
\usepackage{hyperref}           
\hypersetup{hypertexnames=false}
\usepackage{url}                
\usepackage{booktabs}           
\usepackage{amsfonts}           
\usepackage{nicefrac}           
\usepackage{microtype}          
\usepackage{caption}            
\usepackage[dvipsnames]{xcolor} 
\definecolor{darkgreen}{rgb}{0,0.5,0}
\usepackage{colortbl}
\usepackage{graphicx}
\usepackage{wrapfig}
\usepackage{subfigure}
\usepackage{multirow}
\usepackage{float}
\usepackage{placeins}
\usepackage{adjustbox}
\usepackage{tabularx}
\usepackage{amsmath}
\usepackage{amssymb}
\usepackage{mathtools}
\usepackage{amsthm}
\usepackage{listings}
\usepackage{tcolorbox}
\usepackage{makecell}
\tcbuselibrary{breakable,skins}
\usepackage{enumitem}
\usepackage{multicol}
\usepackage{algorithm}
\usepackage{algorithmic}
\newcommand{\TO}{\textbf{to}~}
\newcommand{\RETURN}{\STATE \textbf{return}~}
\usepackage{pifont}
\newcommand{\cmark}{\ding{51}}
\newcommand{\xmark}{\ding{55}}

\newcommand{\flowr}{FlowSteer}
\newcommand{\canvas}{Workflow Canvas}
\newcommand{\director}{Flow-Director}

\makeatletter
\def\thm@space@setup{%
  \thm@preskip=3pt plus 1pt minus 1pt
  \thm@postskip=3pt plus 1pt minus 1pt
}
\makeatother

\newtheorem{proposition}{Proposition}

\usepackage[capitalize,noabbrev]{cleveref}

\theoremstyle{plain}

\theoremstyle{definition}

\theoremstyle{remark}

\makeatletter
\renewenvironment{proof}[1][\proofname]{\par
  \pushQED{\qed}%
  \normalfont \topsep3\p@\@plus1\p@\@minus1\p@
  \trivlist
  \item[\hskip\labelsep
        \itshape
    #1\@addpunct{.}]\ignorespaces
}{%
  \popQED\endtrivlist\@endpefalse
}
\makeatother

\title{FlowSteer: Towards Agents Designing Agentic Workflows via Reinforced Progressive Canvas Editing}

\author{%
  Mingda Zhang$^{1}$\thanks{Equal contribution.}, Wenjin Liu$^{2}$\footnotemark[1], Tiesunlong Shen$^{3}$\thanks{Co-corresponding authors.}, Qika Lin$^{4}$ \\
  \bfseries Rui Mao$^{2}$, Erik Cambria$^{2}$, Xiaoying Tang$^{5,6,7}$, Haoran Luo$^{2}$\footnotemark[2] \\[4pt]
  \mdseries\footnotesize $^{1}$School of Artificial Intelligence, The Chinese University of Hong Kong, Shenzhen, China \\
  \mdseries\footnotesize $^{2}$College of Computing and Data Science, Nanyang Technological University, Singapore \\
  \mdseries\footnotesize $^{3}$Centre for Biomedical Data Science, Duke-NUS Medical School, Singapore \\
  \mdseries\footnotesize $^{4}$Saw Swee Hock School of Public Health, National University of Singapore, Singapore \\
  \mdseries\footnotesize $^{5}$School of Science and Engineering, The Chinese University of Hong Kong, Shenzhen, China \\
  \mdseries\footnotesize $^{6}$Shenzhen Future Network of Intelligence Institute (FNii-Shenzhen), China \\
  \mdseries\footnotesize $^{7}$Guangdong Provincial Key Laboratory of Future Networks of Intelligence, CUHK(SZ), China
}

\begin{document}

\setlength{\textfloatsep}{8pt plus 2pt minus 2pt}
\setlength{\floatsep}{8pt plus 2pt minus 2pt}
\setlength{\intextsep}{8pt plus 2pt minus 2pt}
\setlength{\dbltextfloatsep}{8pt plus 2pt minus 2pt}
\setlength{\dblfloatsep}{8pt plus 2pt minus 2pt}

\maketitle
\suppressfloats[t] 

\raggedbottom

\makeatletter
\renewcommand{\section}{\@startsection{section}{1}{\z@}%
  {-1.6ex \@plus -0.3ex \@minus -0.2ex}%
  { 1.0ex \@plus  0.1ex}%
  {\large\bf\raggedright}}
\renewcommand{\subsection}{\@startsection{subsection}{2}{\z@}%
  {-1.4ex \@plus -0.2ex \@minus -0.1ex}%
  { 0.5ex \@plus  0.1ex}%
  {\normalsize\bf\raggedright}}
\renewcommand{\subsubsection}{\@startsection{subsubsection}{3}{\z@}%
  {-1.2ex \@plus -0.2ex \@minus -0.1ex}%
  { 0.4ex \@plus  0.1ex}%
  {\normalsize\bf\raggedright}}
\makeatother

\begin{abstract}
In recent years, agentic workflows have been widely applied to complex human tasks. However, existing workflow construction still faces key challenges, including human-dependent construction, the lack of graph-level execution feedback, and the inability to repair errors in-loop during long-horizon construction. To address these challenges, we propose \flowr{}, a new paradigm of Agent Designing Agentic Workflows---a single agent end-to-end designs the workflow that a downstream executor runs. To support this paradigm, we introduce the Workflow Canvas, an executable graph-state environment that returns syntax-checked execution feedback for every atomic edit. Built on the canvas, we propose Reinforced Progressive Canvas Editing, in which a lightweight policy agent issues one atomic edit per turn conditioned on real canvas feedback and is trained end-to-end via reinforcement learning. Moreover, \flowr{} provides a plug-and-play framework supporting diverse operator libraries and interchangeable LLM backends. Experiments on twelve datasets show that \flowr{} significantly outperforms baselines across various tasks. Our code is available at \url{https://github.com/SenticNet/FlowSteer}.
\end{abstract}

\section{Introduction}

\begin{wrapfigure}{r}{0.5\textwidth}
\vspace{-4.5mm}
\centering
\includegraphics[width=6.75cm]{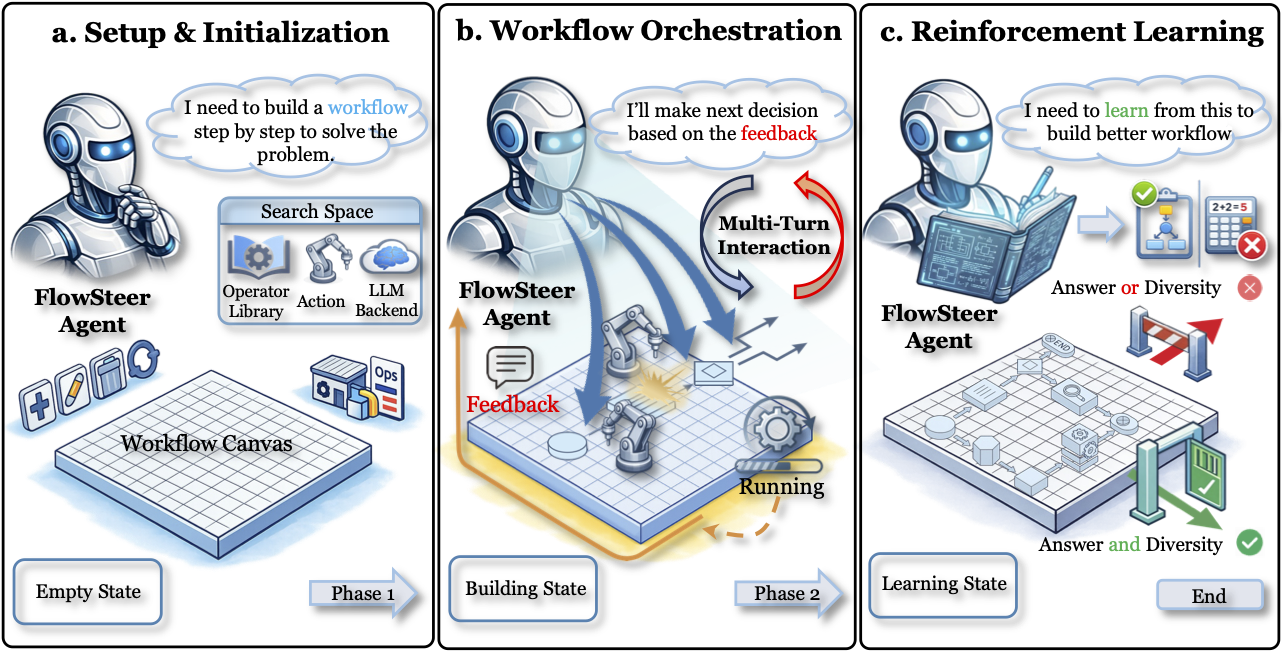}
\caption{Overview of the \flowr{} framework. The agent initializes with the task, then through multi-turn interaction with the canvas iteratively builds and refines the workflow, finally learning from diversity-constrained rewards.}
\label{fig:example}
\vspace{-4mm}
\end{wrapfigure}

In recent years, a variety of powerful agentic systems have been applied to solve a wide range of human problems~\citep{tptu2023,hugginggpt2023,metagpt2024,wu2025knowledge,peng2024tong}, gradually moving beyond single-turn question answering (QA) toward executable end-to-end task completion.

\begin{figure}[t]
\centering
\includegraphics[width=\linewidth]{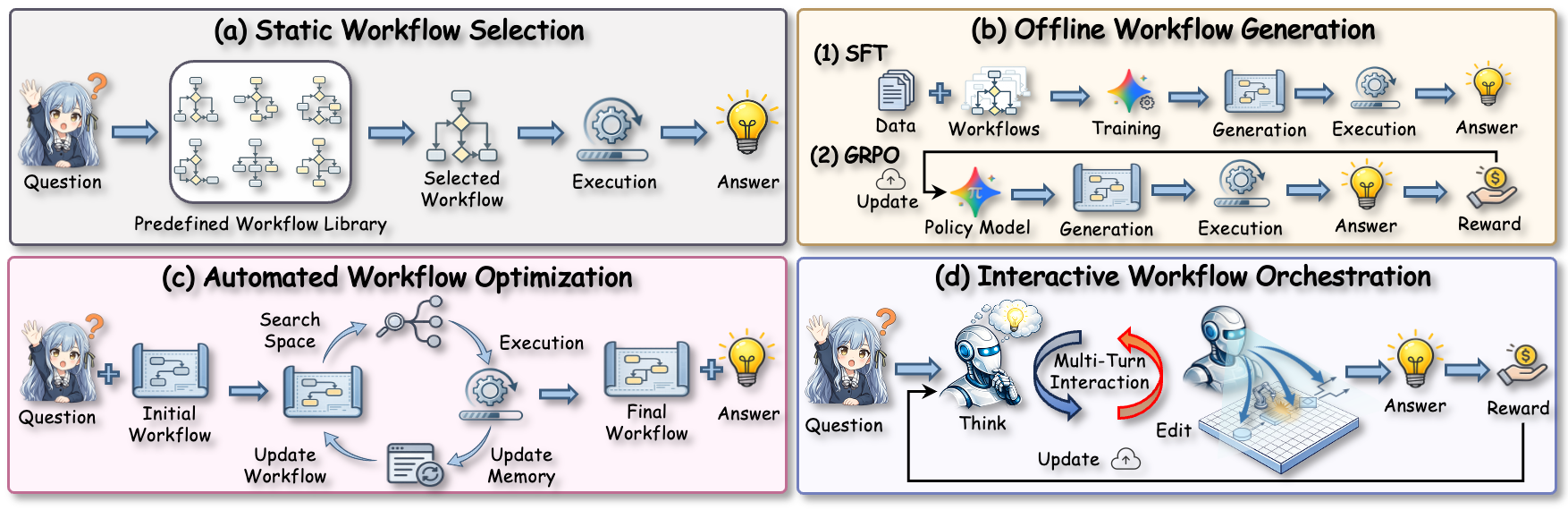}
\caption{Comparison of workflow orchestration paradigms: static selection, offline generation, automated optimization, and our interactive orchestration framework \flowr{}.}
\label{fig:comparison}
\end{figure}

In this process, workflow orchestration has become a key bridge from task goals to reproducible execution: by organizing operators into an executable workflow graph, systems can complete complex tasks with improved controllability, debuggability, and reusability~\citep{flowmind2023,chatdev2024}, as shown in Figure~\ref{fig:example}. However, in practice, workflow construction still heavily relies on manual drag-and-drop and rule-based configuration~\citep{aflow2024,scoreflow2025}, making it costly to transfer across new tasks, new operator libraries, new model backends, or different application domains.

To address these issues, three main paradigms of workflow orchestration have emerged, as shown in Figure~\ref{fig:comparison}. First, static workflow selection retrieves pre-defined workflows from a library based on task similarity~\citep{awm2024}. Second, offline workflow generation trains policy models via supervised fine-tuning (SFT) or group relative policy optimization (GRPO)~\citep{deepseekmath2024} to generate workflows. Third, automated workflow optimization---such as AFlow~\citep{aflow2024}, GPTSwarm~\citep{gptswarm2024}, RobustFlow~\citep{robustflow2025}, and LATS~\citep{lats2023}---combines search and execution feedback to iteratively improve workflow structures.

However, these methods still face several challenges:
\textbf{(i) Human-dependent workflow construction}---workflow graphs are still produced primarily through manual drag-and-drop, rule templates, or one-shot LLM prompting~\citep{aflow2024,awm2024}, requiring continuous human effort to maintain and adapt across new tasks, which severely limits scalability and generalization;
\textbf{(ii) Lack of graph-level execution feedback}---existing agent environments operate at the level of text generation or external tool calls~\citep{react2022,toolformer2023}, lacking an executable abstraction that can provide structured feedback on the workflow graph itself; without such feedback, the policy cannot diagnose or repair structural errors during construction;
\textbf{(iii) Inability to repair errors in-loop during long-horizon construction}---one-shot generation or external search-based construction cannot iteratively repair errors as they emerge across many editing steps~\citep{steptool2024,deepseekr12025}, leading to sparse, delayed learning signals, shortcut behaviors (e.g., premature termination, oversimplified graphs), and unstable long-horizon credit assignment.

To address these challenges, we propose \flowr{}, a new paradigm of \emph{Agent Designing Agentic Workflows}, in which a single agent end-to-end designs the workflow that a downstream executor runs---replacing the manual-design loop. To support this paradigm, we introduce the \emph{\canvas{}}, a novel executable graph-state environment that maintains the workflow graph and returns syntax-checked execution feedback for every atomic edit, lifting agent--environment interaction from prompt rewriting to graph-level operations. Built on this canvas, we further propose \emph{Reinforced Progressive Canvas Editing}: a lightweight policy agent (\director{}) issues one atomic edit per turn---inserting, deleting, modifying, or configuring an operator node---conditioned on real canvas feedback, decomposing long-horizon workflow construction into checkable and repairable local decisions. The policy is trained end-to-end via reinforcement learning on canvas trajectories. Moreover, \flowr{} supports plug-and-play deployment across operator libraries and LLM backends.

We perform experiments on twelve datasets covering mathematical reasoning, question answering, and code generation, across six modern LLM backbones. Experimental results demonstrate that \flowr{} outperforms static-workflow, search-based, and agent-RL baselines in task accuracy, orchestration efficiency (token consumption and interaction turns), and cross-backbone transferability. As shown in Figure~\ref{fig:comparison}, the reinforced progressive canvas editing strategy guides the agent through iterative graph-level edits with real execution feedback, effectively bridging the gap between human-designed workflows and self-designing agentic systems. This work lays a foundation for building the next generation of agent-designing-agent systems for autonomous agentic-workflow construction.

\section{Related Work}

\textbf{Agent Workflows.} Before the rise of LLM agents, workflow automation was largely rule- or template-driven, with operators and control flow specified by hand. In the era of LLM agents, agent workflows improve long-horizon reliability via a plan--act--feedback loop~\citep{planandact2025,hiplan2025,metagpt2024}. Existing studies can be grouped into three lines: single-agent decision making, which models tool use as sequential decisions~\citep{planandact2025} or interleaved reasoning and acting~\citep{agentoccam2025}; orchestration, which uses LLM controllers for tool/model routing~\citep{agentsquare2025,xrouter2025} or constrained API planning to ground intent~\citep{awm2024}; and multi-agent collaboration, which relies on standard operating procedures (SOPs)/roles~\citep{metagpt2024} or cross-team orchestration~\citep{agentverse2023,huang2026collabflow}. For sustained capability, agentic tool reasoning~\citep{agenticreasoning2025}, reusable workflow memory, and evolving skill libraries further improve efficiency~\citep{awm2024,legomem2025,zheng2025machine,wu2025knowledge,zhang2026r2flow,zhang2026skillflow,zhang2026evosteer}, while modular agent search automates compositional design~\citep{agentswift2025}.

\textbf{Reinforcement Learning for Agents.} Recent agent RL frames interaction as a long-horizon Markov decision process (MDP), where coupling value learning with token-level policy learning eases delayed credit assignment in hierarchical multi-turn settings~\citep{archer2024,multiturnrl2025}. For tool-chain control, step-grained shaping~\citep{steptool2024,istar2025} and outcome feedback improve tool selection and self-correction~\citep{retool2025,verltool2025}. When retrieval is treated as an action, search rollouts learn think-then-search behaviors~\citep{searchr12025}. End-to-end RL also trains a small model to prompt a larger one~\citep{liu2026promptr1}. Moreover, large-scale RL motivates GRPO-style objectives based on verifiable or group-relative signals~\citep{deepseekr12025,deepseekmath2024,treegrpo2025}. Concurrently, Workflow-R1~\citep{workflowr12026} trains a policy calling operators in a think--action--observation history and may answer directly, while \flowr{} edits an explicit executable graph and lets the executor answer (Appendix~\ref{app:workflowr1}).

\section{Preliminaries}
\label{sec:preliminaries}

\textbf{Definition 1: Workflow Graph.} A workflow graph is a directed acyclic graph $\mathcal{G}=(V,E,\mathrm{attr})$, where $V=\{v_1,\dots,v_n\}$ is a set of $n$ operator nodes, $E \subseteq V \times V$ encodes data dependencies and execution order, and $\mathrm{attr}(v) = (\mathrm{op}(v), \mathrm{param}(v), \mathrm{prompt}(v))$ specifies the operator type from library $\mathcal{O}$, parameter configuration, and execution prompt for each node $v \in V$ in the workflow.

\textbf{Definition 2: Orchestration Trajectory.} An orchestration trajectory is a complete sequence of $T$ interaction steps that uniquely determines a workflow graph $\mathcal{G}_\tau$:
\begin{equation}
\tau = \{(a_t^{\mathrm{think}}, a_t, o_t^{\mathrm{exec}})\}_{t=1}^{T} \quad \Rightarrow \quad \mathcal{G}_\tau,
\label{eq:trajectory}
\end{equation}
where $T$ is the total number of interaction steps before termination, $a_t^{\mathrm{think}}$ denotes the reasoning reflection at step $t$, $a_t = (\alpha_t, a_t^{\mathrm{out}})$ is the editing action with type $\alpha_t \in \mathcal{A}_{\mathrm{type}}$ and content $a_t^{\mathrm{out}}$, and $o_t^{\mathrm{exec}}$ is the execution feedback from the canvas environment. Table~\ref{tab:operators} lists all operators and actions.

\textbf{Problem Statement.} Given a task $q$ drawn from a task distribution $\mathcal{D}_Q$, an operator library $\mathcal{O}$, and a pluggable backend LLM $\mathcal{M}_{\mathrm{exec}}$, the workflow orchestration problem aims to learn a policy $\pi_\theta$ that generates trajectory $\tau$. The corresponding workflow $\mathcal{G}_\tau$ is then executed to produce the answer:
\begin{equation}
y_q = \mathrm{Execute}(\mathcal{G}_\tau, q, \mathcal{M}_{\mathrm{exec}}).
\label{eq:execute}
\end{equation}
The overall learning objective is to maximize expected reward over the task distribution:
\begin{equation}
\max_\theta \ \mathcal{J}(\theta) = \mathbb{E}_{q \sim \mathcal{D}_Q} \ \mathbb{E}_{\tau \sim P_\theta(\cdot \mid q)} \left[ R(\tau) \right],
\label{eq:objective}
\end{equation}
where $P_\theta(\tau \mid q)$ is the trajectory distribution induced by policy $\pi_\theta$, and $R(\tau)$ is the trajectory-level reward measuring both structural quality and answer correctness.

\begin{table}[t]
\centering
\caption{Operator library $\mathcal{O}$ and action space $\mathcal{A}$ in \flowr{}. Each row shows a category of operators, their outputs, action types, and the corresponding graph update operations.}
\label{tab:operators}
\fontsize{8pt}{10pt}\selectfont
\setlength{\tabcolsep}{3pt}
\begin{adjustbox}{max width=\linewidth}
\begin{tabular}{@{}llllll@{}}
\toprule
\textbf{Category} & \textbf{Operator} $o \in \mathcal{O}$ & \textbf{Output} & \textbf{Type} & \textbf{Action} & \textbf{Graph Update} \\
\midrule
Planning & \texttt{Plan}, \texttt{Decompose} & strategy, sub-problems & Editing & \texttt{add} & $V_t \leftarrow V_{t-1} \cup \{v\}$ \\
Solving & \texttt{Programmer}, \texttt{Custom}, \texttt{AnswerGen} & code, response, answer & Editing & \texttt{delete} & $V_t \leftarrow V_{t-1} \setminus \{v\}$ \\
Verification & \texttt{Test}, \texttt{Review}, \texttt{Verify} & correct, feedback & Editing & \texttt{modify} & $\mathrm{op}(v) \leftarrow o'$ \\
Revision & \texttt{Revise} & revised solution & Config & \texttt{set\_prompt} & $\mathrm{prompt}(v) \leftarrow p$ \\
Ensemble & \texttt{ScEnsemble}, \texttt{Aggregate} & voted, combined & Terminal & \texttt{finish} & $y_q = \mathrm{Execute}(\mathcal{G}_T, q)$ \\
Formatting & \texttt{Format} & final answer $y_q$ & Control & \texttt{parallel}, \texttt{cond}, \texttt{loop} & branch/merge in $E_t$ \\
\textit{User-defined} & \textit{$o^{+} \in \mathcal{O}, \dots$} & \textit{task-specific} & \textit{Editing/Config} & \textit{\texttt{add}, \texttt{modify}, \texttt{set\_prompt}} & \textit{(as above)} \\
\bottomrule
\end{tabular}
\end{adjustbox}
\end{table}

\section{Methodology: \flowr{}}
\label{sec:method}

As illustrated in Figure~\ref{fig:framework}, \flowr{} comprises three components, each addressing one of the three challenges identified in the introduction: the canvas environment and operator/action interface that supply graph-level execution feedback (Section~\ref{subsec:canvas}), the \emph{Agent Designing Agentic Workflows} paradigm that replaces human-driven workflow construction with a designer agent (Section~\ref{subsec:paradigm}), and \emph{Reinforced Progressive Canvas Editing} that turns long-horizon construction into a sequence of in-loop, repairable edits (Section~\ref{subsec:rl}).

\subsection{\canvas{} and Agent Initialization}
\label{subsec:canvas}

\flowr{} follows a ReAct-based agent paradigm~\citep{react2022}: a lightweight policy model (\director{}) interacts with an executable canvas environment (\canvas{}) to construct the workflow graph $\mathcal{G}$ (Definition 1). The canvas, operator library, action space, and state space below jointly fix the agent--environment interface that the rest of the section builds on.

\textbf{\canvas{} $\mathcal{C}$.} The \canvas{} is the environment that maintains the workflow graph state $\mathcal{G}_t$ and provides executable feedback at each orchestration step:
\begin{equation}
\mathcal{C} = (\mathcal{G}_t, \mathcal{O}, \mathcal{M}_{\mathrm{exec}}, d^{\mathrm{lib}}),
\label{eq:canvas}
\end{equation}
where $\mathcal{G}_t = (V_t, E_t, \mathrm{attr})$ is the workflow graph at step $t$, $\mathcal{O}=\{o_1,\dots,o_K\}$ is the operator library with $K$ operators, $\mathcal{M}_{\mathrm{exec}}$ is the pluggable LLM backend, and $d^{\mathrm{lib}}$ is the operator library description that enables the policy to learn available operators.

\textbf{Operator Library $\mathcal{O}$ and Action Space $\mathcal{A}$.} The 12 functional operators and the corresponding editing/control actions are summarized in Table~\ref{tab:operators}. The library $\mathcal{O}$ is plug-and-play: operators can be freely added, removed, or replaced without affecting the framework, since the policy interacts with each operator only through its standardized action interface and structured output. At each step $t$, the agent emits a reflection $a_t^{\mathrm{think}}$ followed by an editing action $a_t = (\alpha_t, a_t^{\mathrm{out}})$ with type $\alpha_t \in \mathcal{A}_{\mathrm{type}}$; orchestration terminates when $\alpha_t = \texttt{finish}$.

\textbf{State Space $\mathcal{S}$.} The agent state $H_t \in \mathcal{S}$ is defined by its complete interaction history. Given task $q$, operator description $d^{\mathrm{lib}}$, and template $a^{\mathrm{tmpl}}$, the initial state is $H_0 = [q \oplus d^{\mathrm{lib}} \oplus a^{\mathrm{tmpl}}]$, where $\oplus$ denotes sequence concatenation. The state updates as:
\begin{equation}
H_t = H_{t-1} \oplus (a_t^{\mathrm{think}}, a_t, o_t^{\mathrm{exec}}).
\label{eq:state_update}
\end{equation}

\textbf{Orchestration Target.} The agent interacts with canvas $\mathcal{C}$ until reaching \texttt{finish} or maximum turns $T_{\max}$. Following Definition 2, the complete trajectory $\tau$ determines workflow $\mathcal{G}_\tau$, and the answer is obtained via execution (Eq.~\ref{eq:execute}).

\begin{proposition}
The operator-action space $(\mathcal{O}, \mathcal{A})$ covers diverse workflow patterns across task types through canvas-grounded orchestration (proof in Appendix~\ref{app:proof1}; evidence in Sections~\ref{subsec:main-results} and~\ref{subsec:generalization}).
\end{proposition}

\begin{figure}[!t]
\centering
\includegraphics[width=\linewidth]{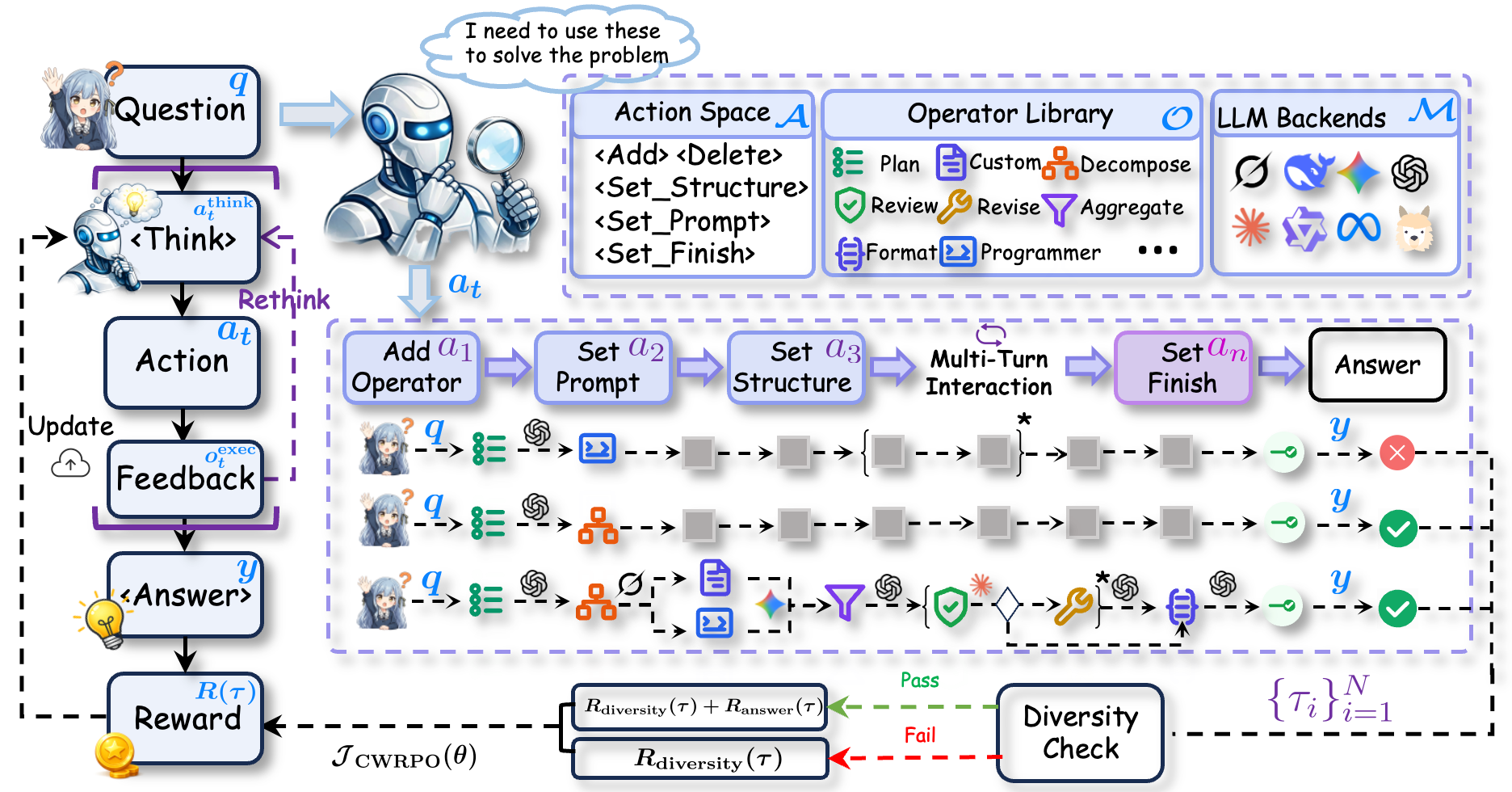}
\caption{An overview of the \flowr{} framework. The designer agent progressively edits the workflow graph in \canvas{} through atomic actions, conditioned on canvas execution feedback, and is trained end-to-end with reinforcement learning over the trajectory-level outcome reward $R(\tau)$.}
\label{fig:framework}
\end{figure}

\subsection{Agent Designing Agentic Workflows}
\label{subsec:paradigm}

A defining feature of \flowr{} is the explicit separation between \emph{workflow design} and \emph{workflow execution}: the trained agent $\pi_\theta$ is responsible only for designing a workflow, while task execution is delegated to a separate backbone LLM operating on the designed workflow. We formalize this designer--executor paradigm below.

\textbf{Designer--Executor Decoupling.} Given a task $q \in \mathcal{D}_Q$, the policy $\pi_\theta$ (\director{}) samples an orchestration trajectory $\tau$, which is decoded into a workflow graph $\mathcal{G}_\tau$; the answer is produced by an executor backbone $\mathcal{M}_{\mathrm{exec}}$ running the graph,
\begin{equation}
y_q = \mathrm{Execute}(\mathcal{G}_\tau,\ q,\ \mathcal{M}_{\mathrm{exec}}),\quad \mathcal{G}_\tau = \mathrm{Decode}(\tau),\quad \tau \sim \pi_\theta(\cdot \mid q),
\label{eq:design-execute}
\end{equation}
where $\mathrm{Decode}(\cdot)$ deterministically extracts the final workflow graph from the editing trajectory $\tau$ by sequentially applying its atomic edits, and $\mathrm{Execute}(\cdot)$ runs the resulting graph $\mathcal{G}_\tau$ on the executor backbone $\mathcal{M}_{\mathrm{exec}}$ to produce the answer $y_q$. Throughout training, $\pi_\theta$ is optimized only to produce $\tau$, while $\mathcal{M}_{\mathrm{exec}}$ is held fixed.

\textbf{Multi-Agent Workflow Execution.} The executor runs $\mathcal{G}_\tau$ as a multi-agent system: each node is an LLM-based agent on $\mathcal{M}_{\mathrm{exec}}$ with its own role, prompt, and upstream inputs, directed edges pass results between agents, and \texttt{parallel}, \texttt{cond}, and \texttt{loop} structures schedule concurrent, branching, and iterative execution. \director{} only edits the graph and issues \texttt{finish}, while subproblem solving and the final answer come from these executor agents (e.g., the Decompose $\to$ [Custom$\times$3] $\to$ Verify $\to$ Aggregate workflow in Appendix~\ref{app:case_study_2}).

\textbf{Marginal Likelihood over Workflow Designs.} Equivalently, the answer distribution induced by \flowr{} marginalizes over all possible workflow designs sampled from $\pi_\theta$,
\begin{equation}
P\!\left(y_q \mid q;\,\pi_\theta, \mathcal{M}_{\mathrm{exec}}\right) = \mathbb{E}_{\tau \sim \pi_\theta(\cdot \mid q)}\!\Big[\,P_{\mathrm{exec}}\!\left(y_q \mid \mathcal{G}_\tau, q, \mathcal{M}_{\mathrm{exec}}\right)\Big],
\label{eq:marginal-likelihood}
\end{equation}
where $P_{\mathrm{exec}}(y_q \mid \mathcal{G}_\tau, q, \mathcal{M}_{\mathrm{exec}})$ is the executor's stochastic answer distribution when running workflow $\mathcal{G}_\tau$ on backbone $\mathcal{M}_{\mathrm{exec}}$ for task $q$. The unit of decision is a workflow design $\tau$ at the level of operator composition; $\pi_\theta$ thereby learns orchestration strategies that transfer across operator implementations and backbone LLMs.

\textbf{Design Space and Decoupled Gradient.} For every node $v \in V_\tau$, a trajectory $\tau$ specifies the per-node attributes $\mathrm{attr}_\tau(v) = (\mathrm{op}(v),\,\mathrm{param}(v),\,\mathrm{prompt}(v))$ through the editing actions \texttt{add}, \texttt{modify}, and \texttt{set\_prompt} (Table~\ref{tab:operators}); $\mathrm{prompt}(v)$ in particular can be rewritten by $\pi_\theta$ online during the editing loop. The training objective and its policy gradient are
\begin{equation}
\mathcal{J}(\theta) = \mathbb{E}_{q \sim \mathcal{D}_Q,\,\tau \sim \pi_\theta(\cdot|q)}\!\big[R(\tau)\big], \qquad \nabla_\theta \mathcal{J}(\theta) = \mathbb{E}_{q,\,\tau \sim \pi_\theta}\!\big[\nabla_\theta \log \pi_\theta(\tau \mid q)\cdot R(\tau)\big].
\label{eq:designer-grad}
\end{equation}
The executor $\mathcal{M}_{\mathrm{exec}}$ influences the gradient solely through its scalar contribution to $R(\tau)$ via the executed answer $y_q$. Consequently, swapping $\mathcal{M}_{\mathrm{exec}}$, editing the library $\mathcal{O}$, or rewriting an operator's prompt preserves the form of Eq.~\ref{eq:designer-grad} and only updates the realized reward of a given workflow design---this underlies the transferability of $\pi_\theta$ across operator libraries and backbones.

\begin{proposition}
The designer--executor decoupling allows the trained policy $\pi_\theta$ to generalize across operator libraries $\mathcal{O}$ and backbone LLMs $\mathcal{M}_{\mathrm{exec}}$ without retraining (proof in Appendix~\ref{app:proof2}; evidence with six target and three training executors in Section~\ref{subsec:transferability}).
\end{proposition}

\begin{table}[!t]
\centering
\fontsize{9pt}{11pt}\selectfont
\renewcommand{\arraystretch}{1.0}
\caption{The system prompt template utilized by \director{} to interact with the \canvas{}.}
\label{tab:prompt}
\begin{tabularx}{\linewidth}{@{}X@{}}
\toprule
Build a workflow step by step to solve the problem. In each turn, output \textbf{EXACTLY ONE} XML action (\texttt{add}/\texttt{delete}/\texttt{modify}/\texttt{set\_prompt}/\texttt{finish} or a structure action). Keep thinking brief; let operators do the computation. Format every turn as: \textcolor{orange}{\textbf{<think>}}(reasoning, conditioned on the latest \textcolor{red}{\textbf{<feedback>}}...\textcolor{red}{\textbf{</feedback>}})\textcolor{orange}{\textbf{</think>}} \textcolor{green!50!black}{\textbf{<action>}}(one editing action that builds on previous edits)\textcolor{green!50!black}{\textbf{</action>}}. Let the workflow evolve naturally (e.g., plan $\to$ solve $\to$ verify $\to$ format). When the workflow is complete, output \textcolor{green!50!black}{\textbf{<action>}}\textcolor{blue}{\textbf{finish}}\textcolor{green!50!black}{\textbf{</action>}}. Task: \textcolor{purple}{\textbf{\{task\}}}. \tabularnewline
\bottomrule
\end{tabularx}
\end{table}

\subsection{Reinforced Progressive Canvas Editing}
\label{subsec:rl}

We instantiate the above paradigm by training $\pi_\theta$ via \emph{progressive canvas editing}: an interaction loop where the agent commits one atomic edit per turn and the canvas returns execution feedback that conditions the next edit, following the prompt template in Table~\ref{tab:prompt}.

\textbf{Progressive Graph Construction.} Each turn submits one atomic editing action, decomposing long-horizon design into checkable and repairable local decisions. Starting from the empty initial graph $\mathcal{G}_0 = (\emptyset, \emptyset, \emptyset)$, after $T$ atomic edits the workflow is the cumulative composition
\begin{equation}
\mathcal{G}_T = \mathcal{G}_0 \,\oplus\, a_1 \,\oplus\, a_2 \,\oplus\, \cdots \,\oplus\, a_T, \qquad \oplus:\ \mathcal{G}\times \mathcal{A} \to \mathcal{G},
\label{eq:graph-compose}
\end{equation}
where $\oplus$ denotes the graph-update operator instantiated by Table~\ref{tab:operators} (i.e.\ \texttt{add}, \texttt{delete}, \texttt{modify}, \texttt{set\_prompt}, control structures), which maps a graph and an atomic edit to a new graph.

\textbf{Hierarchical Editing Policy.} At each step $t$, \director{} factorizes the editing action into a reflection, an action type, and a structured output, conditioned on the interaction history $H_{t-1}$:
\begin{equation}
\log \pi_\theta(a_t^{\mathrm{think}}, a_t \mid H_{t-1}) =
\begin{cases}
\begin{aligned}[t]
& \log \pi_\theta(a_t^{\mathrm{think}} \mid H_{t-1}) \\
& \,+\, \log \pi_\theta(\alpha_t \mid a_t^{\mathrm{think}}, H_{t-1}) \\
& \,+\, \log \pi_\theta(a_t^{\mathrm{out}} \mid \alpha_t, a_t^{\mathrm{think}}, H_{t-1}),
\end{aligned} & \alpha_t \neq \texttt{finish}, \\[10pt]
\begin{aligned}[t]
& \log \pi_\theta(a_t^{\mathrm{think}} \mid H_{t-1}) \\
& \,+\, \log \pi_\theta(\texttt{finish} \mid a_t^{\mathrm{think}}, H_{t-1}),
\end{aligned} & \alpha_t = \texttt{finish}.
\end{cases}
\label{eq:hierarchical-policy}
\end{equation}

\textbf{Canvas-Feedback-Driven Trajectory.} Given action $a_t$, \canvas{} executes it and returns feedback $o_t^{\mathrm{exec}}$, forming a closed-loop ``diagnose--edit--verify'': \textit{(i)} the canvas performs syntax parsing and constraint checking, sampling $o_t^{\mathrm{exec}} \sim \mathcal{C}_{\mathrm{exec}}(\cdot \mid \mathcal{G}_{t-1}, a_t)$, where $\mathcal{C}_{\mathrm{exec}}$ denotes the canvas's stochastic feedback distribution conditioned on the current graph and edit, carrying success status, failure reasons, and repair suggestions; \textit{(ii)} the graph updates as $\mathcal{G}_t = \mathrm{Update}(\mathcal{G}_{t-1}, a_t, o_t^{\mathrm{exec}})$ and the state evolves via Eq.~\ref{eq:state_update}; \textit{(iii)} \director{} continues until $\alpha_t = \texttt{finish}$ or $t = T_{\max}$, the maximum number of editing turns. The full trajectory distribution factorizes as
\begin{equation}
\resizebox{\linewidth}{!}{$
P_\theta(\tau \mid q) = \underbrace{\pi_\theta(a_1^{\mathrm{think}}, a_1 \mid H_0)}_{\text{first edit}} \cdot \prod_{t=2}^{T}\underbrace{\Big(\,\mathcal{C}_{\mathrm{exec}}(o_{t-1}^{\mathrm{exec}} \mid \mathcal{G}_{t-2}, a_{t-1}) \cdot \pi_\theta(a_t^{\mathrm{think}}, a_t \mid H_{t-1})\,\Big)}_{\text{feedback-conditioned subsequent edits}},
$}
\label{eq:trajectory-dist}
\end{equation}
where only $\pi_\theta$ is optimized.

\textbf{End-to-End Reinforcement Learning.} We adopt the standard GRPO objective~\citep{deepseekmath2024}, modified by one canvas-specific adaptation: a \emph{token-level mask} $\mathrm{mask}_t^{(i)} \in \{0,1\}$ that restricts the policy gradient to tokens generated by \director{}, excluding canvas-feedback tokens that share the same context window. Given $N$ rollouts $\{\tau_i\}_{i=1}^{N}$ for task $q$ sampled under behavior policy $\pi_{\theta_{\mathrm{old}}}$, the masked objective is
\begin{equation}
\resizebox{\linewidth}{!}{$
\mathcal{J}(\theta) = \mathbb{E}_{q \sim \mathcal{D}_Q,\,\{\tau_i\} \sim \pi_{\theta_{\mathrm{old}}}}\!\Bigg[\frac{1}{N}\sum_{i=1}^{N}\frac{1}{|\tau_i|_{\mathrm{mask}}}\sum_{t=1}^{|\tau_i|}\mathrm{mask}_t^{(i)} \cdot \min\!\Big(\rho_\theta^{(i,t)}\,\hat{A}_i,\ \mathrm{clip}\!\big(\rho_\theta^{(i,t)},\,1\pm\epsilon\big)\,\hat{A}_i\Big) - \beta\, D_{\mathrm{KL}}\!\big(\pi_\theta\,\|\,\pi_{\mathrm{ref}}\big)\Bigg],
$}
\label{eq:rl-objective}
\end{equation}
where $|\tau_i|_{\mathrm{mask}} = \sum_t \mathrm{mask}_t^{(i)}$ counts policy-generated tokens in $\tau_i$, and the importance ratio $\rho_\theta^{(i,t)}$, the group-standardized trajectory-level advantage $\hat{A}_i$ computed from the trajectory-level reward $R(\tau_i)$, the clip range $\epsilon$, the KL strength $\beta$, the reference policy $\pi_{\mathrm{ref}}$, and the behavior policy $\pi_{\theta_{\mathrm{old}}}$ all follow the standard GRPO formulation~\citep{deepseekmath2024}. The outcome reward $R(\tau)$ combines a structural reward over the workflow graph and an answer reward that is released only when the structural reward is satisfied; full forms are deferred to Appendix~\ref{app:reward}.

\begin{proposition}
End-to-end training of the canvas-feedback-driven multi-turn loop yields stable orchestration policies, suppresses shortcut behaviors, and improves long-horizon credit assignment (proof in Appendix~\ref{app:proof3}; evidence in Sections~\ref{subsec:ablation} and~\ref{subsec:rl-compare}).
\end{proposition}

\begin{table}[t]
\centering
\small
\caption{Main results on twelve IID and OOD benchmarks. SFT and GRPO are applied on Qwen3-8B; all workflow methods use GPT-4o-mini as the evaluation executor (\director{} is trained with GPT-OSS-120B). $\Delta\uparrow$ is the gap to GPT-4o-mini.}
\label{tab:main-results}
\begin{adjustbox}{max width=\linewidth}\tabcolsep=3pt
\renewcommand{\arraystretch}{1.18}
\begin{tabular}{ll|cccccccc|c}
\toprule
\textbf{} & \textbf{} & \multicolumn{2}{c}{\textbf{Baseline}} & \textbf{SFT} & \textbf{GRPO} & \textbf{AFlow} & \multicolumn{3}{c}{\textbf{Agent+RL (4o-mini)}} & \textbf{Ours (4o-mini)} \\
\cmidrule(lr){3-4}\cmidrule(lr){5-5}\cmidrule(lr){6-6}\cmidrule(lr){7-7}\cmidrule(lr){8-10}\cmidrule(lr){11-11}
\textbf{Dataset} & \textbf{Metric} & \textbf{Qwen3-8B} & \textbf{4o-mini} & \textbf{Qwen3-8B} & \textbf{Qwen3-8B} & \textbf{4o-mini} & \textbf{Agentflow} & \textbf{Router-R1} & \textbf{Orchestrator} & \textbf{\flowr{} ($\Delta\uparrow$)} \\
\midrule
\multicolumn{11}{l}{\cellcolor{gray!10}\textit{Part I: In-Distribution (IID) Benchmarks}} \\
\midrule
\textbf{HotPotQA} & \textbf{EM} & 67.19 & 63.28 & 70.31 & 59.38 & 68.75 & 67.19 & 72.00 & 67.97 & \textbf{78.12} \textbf{\textcolor{darkgreen}{(+14.84)}} \\
\textbf{} & \textbf{F1} & 74.05 & 73.03 & 75.25 & 64.95 & 77.90 & 77.88 & 79.84 & 75.61 & \textbf{84.98} \textbf{\textcolor{darkgreen}{(+11.95)}} \\
\textbf{SQuAD v2} & \textbf{EM} & 54.69 & 47.66 & 73.44 & 66.41 & 73.44 & 64.06 & 59.84 & 70.34 & \textbf{78.12} \textbf{\textcolor{darkgreen}{(+30.46)}} \\
\textbf{} & \textbf{F1} & 61.54 & 59.42 & 77.31 & 72.00 & 82.41 & 72.45 & 65.29 & 75.24 & \textbf{83.67} \textbf{\textcolor{darkgreen}{(+24.25)}} \\
\textbf{GSM8K} & \textbf{Acc.} & 91.41 & 92.97 & 92.19 & 92.97 & 94.53 & 93.75 & 94.01 & 93.94 & \textbf{96.09} \textbf{\textcolor{darkgreen}{(+3.12)}} \\
\textbf{MATH} & \textbf{Acc.} & 66.41 & 60.94 & 61.72 & 68.75 & 70.31 & 71.87 & 76.56 & 72.26 & \textbf{81.25} \textbf{\textcolor{darkgreen}{(+20.31)}} \\
\textbf{MBPP} & \textbf{Pass@1} & 63.28 & 64.84 & 57.03 & 77.34 & 83.20 & 79.69 & 73.43 & 74.22 & \textbf{84.38} \textbf{\textcolor{darkgreen}{(+19.54)}} \\
\textbf{HumanEval} & \textbf{Pass@1} & 81.25 & 82.81 & 61.72 & 86.72 & 90.62 & 87.50 & 85.15 & 89.06 & \textbf{92.96} \textbf{\textcolor{darkgreen}{(+10.15)}} \\
\cmidrule(lr){1-11}
\textbf{Avg. (IID)} & \textbf{EM} & 60.94 & 55.47 & 71.88 & 62.90 & 71.10 & 65.63 & 65.92 & 69.16 & \textbf{78.12} \textbf{\textcolor{darkgreen}{(+22.65)}} \\
\textbf{} & \textbf{F1} & 67.80 & 66.23 & 76.28 & 68.48 & 80.16 & 75.17 & 72.57 & 75.43 & \textbf{84.33} \textbf{\textcolor{darkgreen}{(+18.10)}} \\
\textbf{} & \textbf{Acc./Pass} & 75.59 & 75.39 & 68.17 & 81.45 & 84.67 & 83.20 & 82.29 & 82.37 & \textbf{88.67} \textbf{\textcolor{darkgreen}{(+13.28)}} \\
\midrule
\multicolumn{11}{l}{\cellcolor{gray!10}\textit{Part II: Out-of-Distribution (OOD) Benchmarks}} \\
\midrule
\textbf{TriviaQA} & \textbf{EM} & 60.16 & 71.09 & 60.94 & 59.38 & 73.44 & 75.00 & 75.78 & 77.36 & \textbf{79.69} \textbf{\textcolor{darkgreen}{(+8.60)}} \\
\textbf{} & \textbf{F1} & 69.17 & 81.40 & 69.88 & 69.23 & 82.50 & 81.47 & 80.43 & 83.23 & \textbf{84.11} \textbf{\textcolor{darkgreen}{(+2.71)}} \\
\textbf{NaturalQuestions} & \textbf{EM} & 39.84 & 39.84 & 46.09 & 43.75 & 42.97 & 45.70 & 49.22 & 50.00 & \textbf{54.69} \textbf{\textcolor{darkgreen}{(+14.85)}} \\
\textbf{} & \textbf{F1} & 50.75 & 51.42 & 53.40 & 53.24 & 49.92 & 55.98 & 52.79 & 55.41 & \textbf{62.56} \textbf{\textcolor{darkgreen}{(+11.14)}} \\
\textbf{MathQA} & \textbf{Acc.} & 75.00 & 79.69 & 61.71 & 60.15 & 83.59 & 82.81 & 80.47 & 82.03 & \textbf{88.67} \textbf{\textcolor{darkgreen}{(+8.98)}} \\
\textbf{AIME 2025} & \textbf{Acc.} & 16.66 & 10.00 & 0.00 & 8.33 & 13.33 & 10.00 & 10.00 & 20.00 & \textbf{26.67} \textbf{\textcolor{darkgreen}{(+16.67)}} \\
\textbf{APPS} & \textbf{Pass@1} & 39.84 & 40.62 & 26.56 & 34.38 & 42.97 & 41.41 & 42.95 & 44.53 & \textbf{49.21} \textbf{\textcolor{darkgreen}{(+8.59)}} \\
\textbf{DS-1000} & \textbf{Pass@1} & 34.38 & 45.31 & 25.78 & 38.28 & 53.91 & 46.88 & 42.97 & 51.56 & \textbf{58.59} \textbf{\textcolor{darkgreen}{(+13.28)}} \\
\cmidrule(lr){1-11}
\textbf{Avg. (OOD)} & \textbf{EM} & 50.00 & 55.47 & 53.52 & 51.57 & 58.21 & 60.35 & 62.50 & 63.68 & \textbf{67.19} \textbf{\textcolor{darkgreen}{(+11.72)}} \\
\textbf{} & \textbf{F1} & 59.96 & 66.41 & 61.64 & 61.24 & 66.21 & 68.73 & 66.61 & 69.32 & \textbf{73.34} \textbf{\textcolor{darkgreen}{(+6.93)}} \\
\textbf{} & \textbf{Acc./Pass} & 41.47 & 43.91 & 28.51 & 35.29 & 48.45 & 45.28 & 44.10 & 49.53 & \textbf{55.79} \textbf{\textcolor{darkgreen}{(+11.88)}} \\
\bottomrule
\end{tabular}
\end{adjustbox}
\label{tab:ood-results}
\end{table}

\section{Experiments}
\label{sec:experiments}

We evaluate \flowr{} through the following research questions (RQs):
\textbf{RQ1:} Can \flowr{} outperform existing workflow orchestration methods?
\textbf{RQ2:} How does \flowr{} generalize to out-of-distribution benchmarks?
\textbf{RQ3:} How transferable is \flowr{} across different LLM backends?
\textbf{RQ4:} What are the contributions of core components such as the multi-turn interaction paradigm and RL modules?
\textbf{RQ5:} How does our canvas-aware RL objective compare against other RL algorithms (GRPO, DAPO)?

\subsection{Experimental Setup}

We evaluate \flowr{} on diverse in-distribution (IID) and out-of-distribution (OOD) benchmarks, compare against strong baselines, and report standard task-specific metrics.

\textbf{Datasets.}
We use six IID benchmarks covering three task categories: \textit{(i)} mathematical reasoning (GSM8K~\citep{cobbe2021gsm8k}, MATH~\citep{hendrycks2021math}), \textit{(ii)} question answering (HotPotQA~\citep{yang2018hotpotqa}, SQuAD v2~\citep{rajpurkar2018squad2}), and \textit{(iii)} code generation (MBPP~\citep{austin2021mbpp}, HumanEval~\citep{chen2021humaneval}). To assess generalization, we also test on six OOD benchmarks: TriviaQA~\citep{joshi2017triviaqa}, NaturalQuestions~\citep{kwiatkowski2019nq}, MathQA~\citep{amini2019mathqa}, AIME 2025, APPS~\citep{hendrycks2021apps}, and DS-1000~\citep{lai2022ds1000}. All methods are evaluated on the same fixed set of 1,438 test instances, with 128 instances sampled once from each dataset and all 30 AIME 2025 problems. More details in Appendix~\ref{app:datasets}.

\textbf{Baselines.}
We compare \flowr{} against the following baselines: \textit{(i)} direct LLM baselines (Qwen3-8B~\citep{qwen3}, GPT-4o-mini~\citep{openai2024gpt4o}), \textit{(ii)} standard fine-tuning methods (SFT, GRPO~\citep{deepseekmath2024}), \textit{(iii)} workflow-based methods (AFlow~\citep{aflow2024}), and \textit{(iv)} agent with RL methods (AgentFlow~\citep{agentflow2025}, Router-R1~\citep{routerr12025}, Orchestrator~\citep{orchestrator2025}). Unless noted otherwise, \director{} is trained with GPT-OSS-120B, and all workflow methods are evaluated with the GPT-4o-mini executor. More details are provided in Appendix~\ref{app:baselines}.

\textbf{Evaluation Metrics.}
Following standard practice, we use Exact Match (EM) and F1 score for question answering tasks, Accuracy (Acc.) for mathematical reasoning tasks, and Pass@1 for code generation tasks. More details of the evaluation metrics are shown in Appendix~\ref{app:metrics}.

\begin{table}[!t]
\centering
\fontsize{8.5pt}{10pt}\selectfont
\caption{Ablation of \flowr{} on twelve datasets (GPT-4o-mini backend). Component-level rows remove the agent, multi-turn interaction, canvas, or RL; RL-objective rows remove the token mask, diversity-constrained reward, or conditional release.}
\label{tab:ablation}
\begin{adjustbox}{max width=\linewidth}\tabcolsep=2.5pt
\begin{tabular}{@{} l | cccccccc | cccccccc @{}}
\toprule
\textbf{Method} & \multicolumn{8}{c|}{\textbf{IID}} & \multicolumn{8}{c}{\textbf{OOD}} \\
\cmidrule(lr){2-9}\cmidrule(lr){10-17}
\textbf{} & \textbf{GSM8K} & \textbf{MATH} & \multicolumn{2}{c}{\textbf{HotPotQA}} & \multicolumn{2}{c}{\textbf{SQuAD v2}} & \textbf{MBPP} & \textbf{HEval} & \multicolumn{2}{c}{\textbf{TriviaQA}} & \multicolumn{2}{c}{\textbf{NQ}} & \textbf{MathQA} & \textbf{AIME} & \textbf{APPS} & \textbf{DS-1k} \\
\cmidrule(lr){2-9}\cmidrule(lr){10-17}
\textbf{} & \textbf{Acc.} & \textbf{Acc.} & \textbf{EM} & \textbf{F1} & \textbf{EM} & \textbf{F1} & \textbf{Pass@1} & \textbf{Pass@1} & \textbf{EM} & \textbf{F1} & \textbf{EM} & \textbf{F1} & \textbf{Acc.} & \textbf{Acc.} & \textbf{Pass@1} & \textbf{Pass@1} \\
\midrule
\multicolumn{17}{l}{\textit{Component-level ablations}} \\
\midrule
\textbf{w/o Agent}        & 94.53 & 70.31 & 72.66 & 78.91 & 66.41 & 70.63 & 60.94 & 83.59 & 69.53 & 75.08 & 53.12 & 57.43 & 80.47 & 23.33 & 41.40 & 40.63 \\
\textbf{w/o Multi-turn}   & 91.41 & 75.00 & 72.66 & 81.09 & 56.25 & 66.48 & 57.81 & 88.28 & 63.28 & 68.28 & 43.75 & 55.00 & 78.13 & 20.00 & 38.84 & 41.41 \\
\textbf{w/o Canvas}       & 94.53 & 73.44 & 70.31 & 79.53 & 59.38 & 69.53 & 63.28 & 85.16 & 57.81 & 65.63 & 53.13 & 61.48 & 78.91 & 10.00 & 42.96 & 45.31 \\
\textbf{w/o RL}           & 91.41 & 71.81 & 72.66 & 82.19 & 59.38 & 67.66 & 64.06 & 89.06 & 64.84 & 71.80 & 53.13 & 60.72 & 80.47 & 20.00 & 38.28 & 49.22 \\
\midrule
\multicolumn{17}{l}{\textit{RL-objective ablations}} \\
\midrule
\textbf{w/o Token Mask}   & 93.75 & 67.97 & 70.31 & 76.32 & 75.78 & 80.23 & 81.25 & 89.84 & 71.88 & 76.31 & 53.12 & 59.94 & 78.12 & 23.33 & 39.06 & 46.88 \\
\textbf{w/o Div.\ Reward} & 94.53 & 70.31 & 69.53 & 75.44 & 75.00 & 79.43 & 82.81 & 88.28 & 71.09 & 75.05 & 51.56 & 55.43 & 78.91 & 16.67 & 37.50 & 46.09 \\
\textbf{w/o Cond.\ Rel.}  & 91.41 & 74.21 & 72.66 & 78.79 & 75.00 & 79.69 & 82.03 & 89.84 & 70.31 & 75.40 & 46.88 & 50.69 & 83.59 & 23.33 & 42.19 & 50.78 \\
\midrule
\textbf{\flowr{} (Full)}  & \textbf{96.09} & \textbf{81.25} & \textbf{78.12} & \textbf{84.98} & \textbf{78.12} & \textbf{83.67} & \textbf{84.38} & \textbf{92.96} & \textbf{79.69} & \textbf{84.11} & \textbf{54.69} & \textbf{62.56} & \textbf{88.67} & \textbf{26.67} & \textbf{49.21} & \textbf{58.59} \\
\bottomrule
\end{tabular}
\end{adjustbox}
\end{table}

\begin{figure}[!t]
\centering
\begin{minipage}[c]{0.712\linewidth}
\centering
\vspace{0pt}
\subfigure[Radar charts on different backends]{
\label{fig:rq3:radar}
\includegraphics[width=\linewidth,height=0.671217166\linewidth]{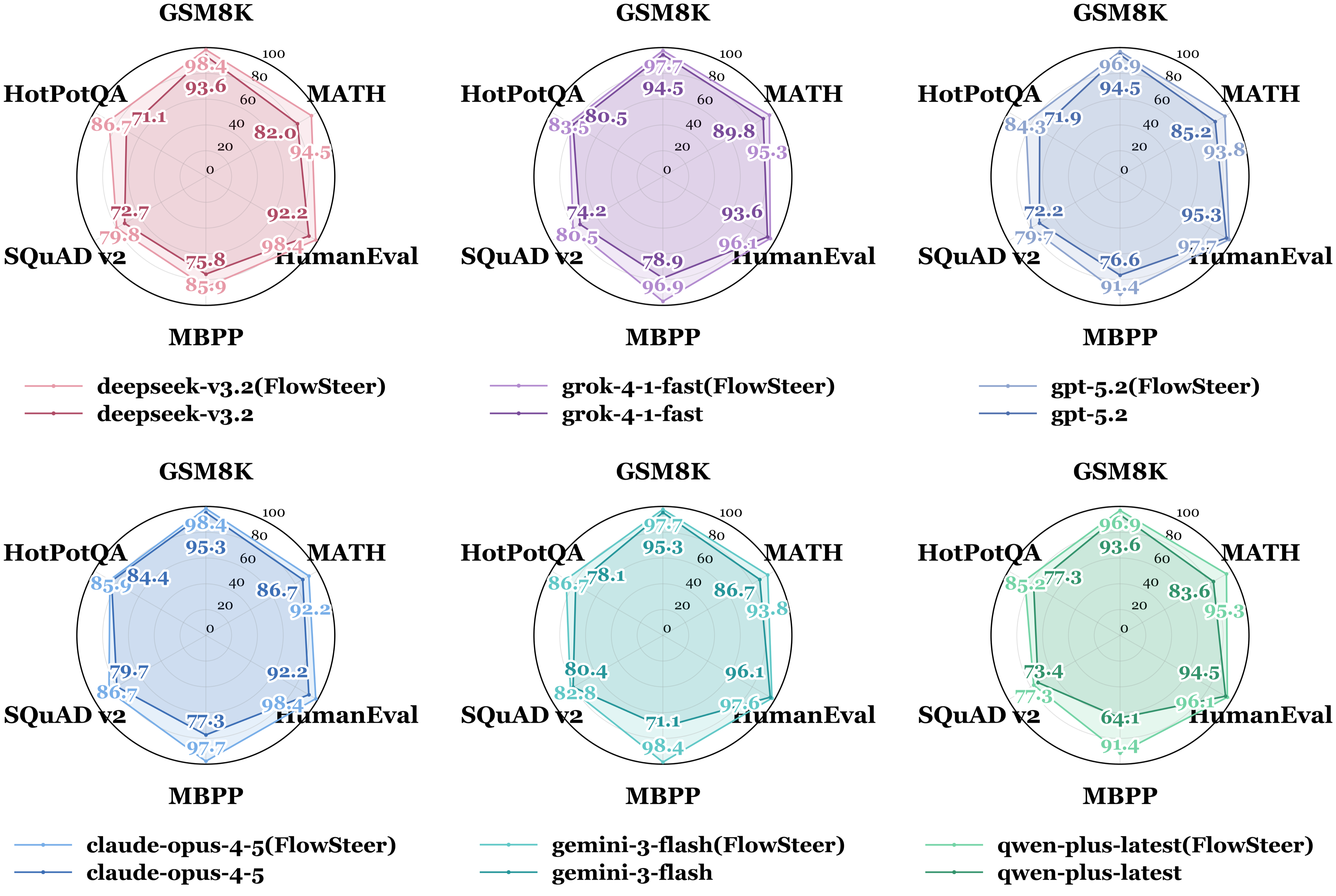}}
\end{minipage}\hfill
\begin{minipage}[c]{0.277\linewidth}
\centering
\subfigure[Training dynamics]{
\label{fig:rq3:dynamics}
\begin{tabular}{@{}c@{}}
\includegraphics[width=\linewidth]{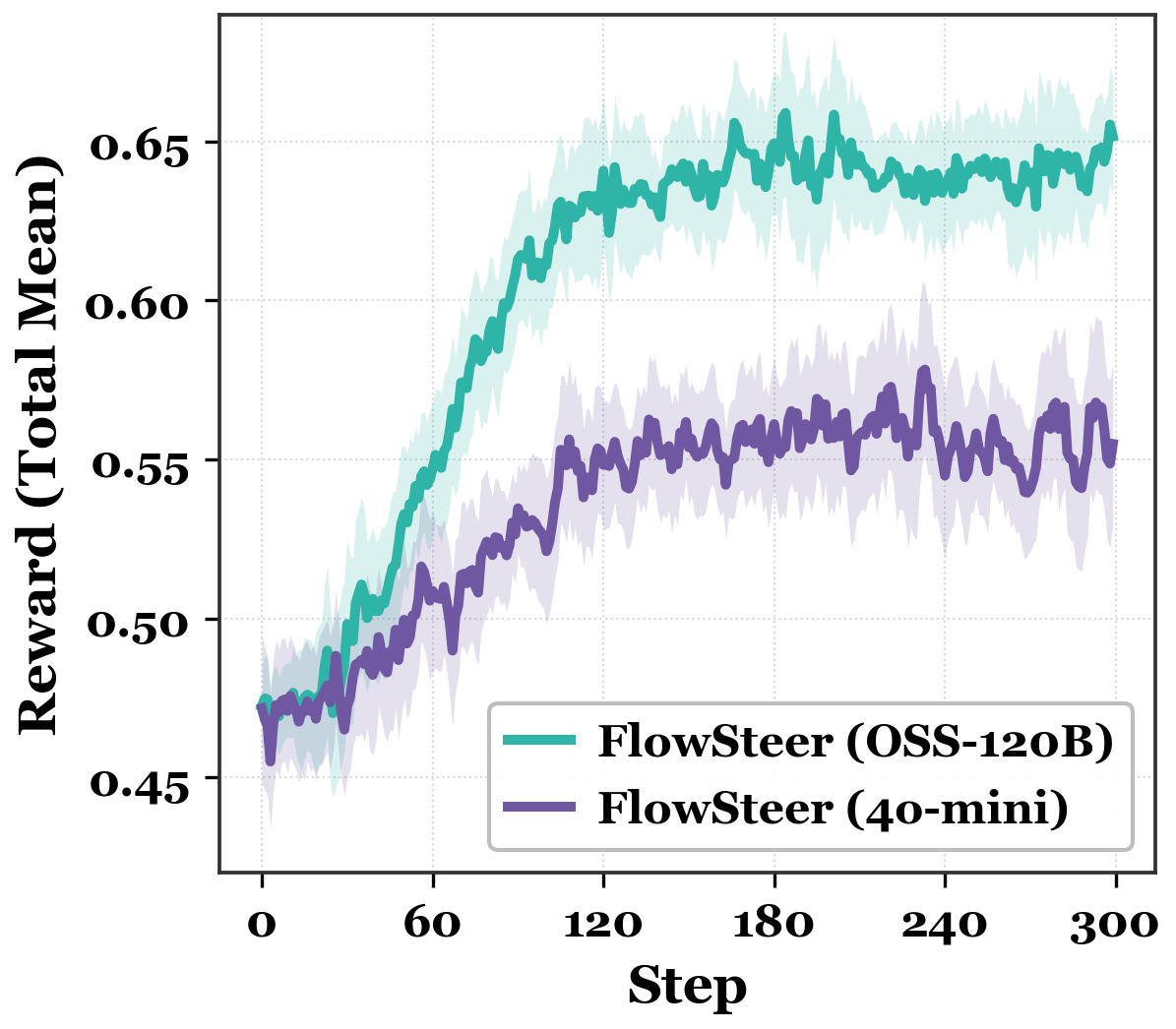}\\[-1pt]
\includegraphics[width=\linewidth]{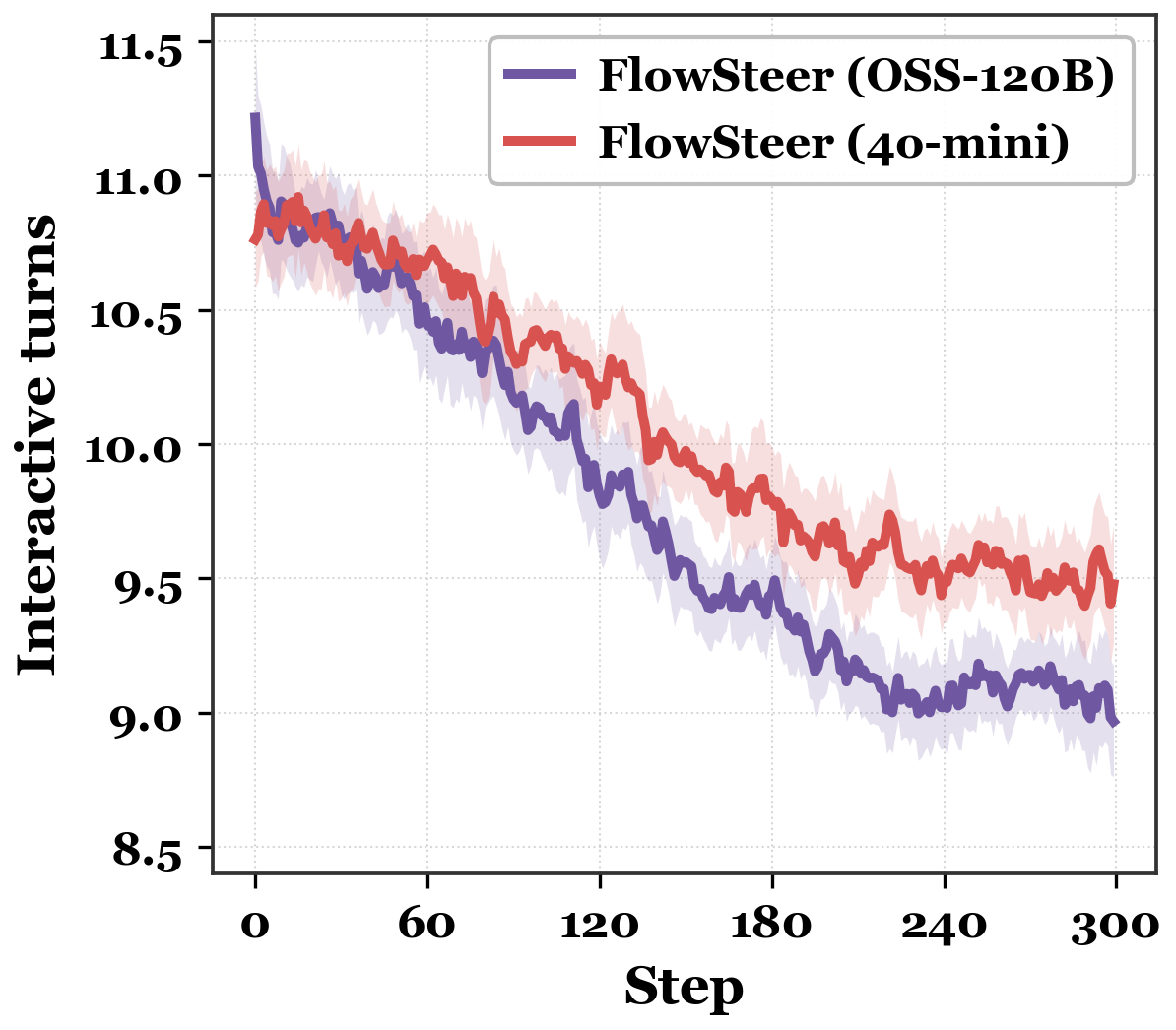}
\end{tabular}}
\end{minipage}\\[2pt]
\subfigure[Aggregated performance by task type]{
\label{fig:rq3:bar}
\includegraphics[width=\linewidth,height=0.180976220\linewidth]{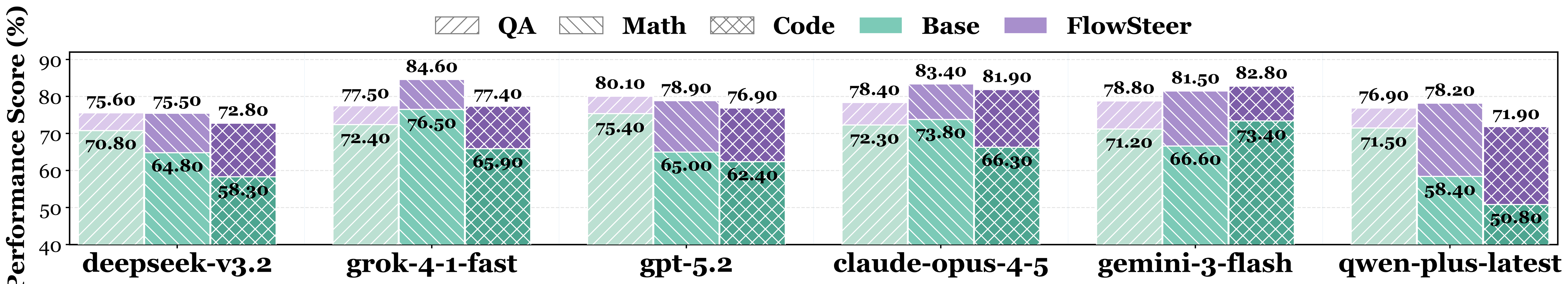}}
\caption{Transferability analysis of \flowr{} across six LLM backends (RQ3), using one frozen \director{} trained with GPT-OSS-120B. (a) Per-backend radar comparison on six IID benchmarks. (b) Training dynamics with GPT-4o-mini vs.\ GPT-OSS-120B as training executor. (c) Aggregated performance by task type (QA, math, and code), Base versus \flowr{}.}
\label{fig:rq3}
\end{figure}

\subsection{Main Results (RQ1)}
\label{subsec:main-results}

As illustrated in Table~\ref{tab:main-results}, \flowr{} achieves consistent improvements over GPT-4o-mini across six IID benchmarks and outperforms multiple baseline categories. We attribute these gains to three interlocking design choices. First, the \canvas{} exposes graph-level execution feedback after every atomic edit, so structural errors are repaired \emph{in-loop} rather than only at the end of generation, which drives the largest gain on long-horizon mathematical reasoning. Second, the \director{}--executor decoupling lets a lightweight policy focus on transferable orchestration while the pluggable backend handles task-specific computation, allowing a small policy to \emph{amplify} a strong backend without backend fine-tuning. Third, progressive canvas editing decomposes long construction into checkable, repairable local decisions, yielding cleaner credit assignment and consistent Pass@1 gains on code generation.

\subsection{Generalization Results (RQ2)}
\label{subsec:generalization}
As shown in Table~\ref{tab:main-results}, \flowr{} maintains its superiority over baselines on six OOD benchmarks, with stable improvements in question answering and mathematical reasoning. The OOD generalization stems from the same decoupled design that drives the IID gains: because the \director{} learns to compose task-agnostic cognitive operators rather than task-specific solution patterns, the orchestration strategy fitted on IID training carries over naturally to novel task families such as AIME~2025 and DS-1000. In contrast, large-scale LLM backends show strong zero-shot capability but lack transferable orchestration, while search-based workflows are inefficient in OOD settings. These results show that \flowr{} achieves robust cross-task and cross-distribution gains without task-specific fine-tuning of the backend LLM.

\begin{table}[!t]
\centering
\caption{Transfer across training and target executors, averaged over the twelve datasets. Each row is a frozen \director{} trained with the listed executor; Base runs the target executor without \flowr{}; $\dag$ marks matched training and target executors.}
\label{tab:train-exec}
\small
\setlength{\tabcolsep}{5pt}
\begin{tabular}{@{}l|cccccc|c@{}}
\toprule
\textbf{Train $\downarrow$ / Target $\rightarrow$} & \textbf{DS-V3.2} & \textbf{Opus-4.5} & \textbf{Gemini-3-F} & \textbf{OSS-120B} & \textbf{4o-mini} & \textbf{Qwen3.5-9B} & \textbf{Avg.} \\
\midrule
Base (no \flowr{}) & 72.93 & 78.37 & 77.51 & 66.42 & 58.25 & 60.56 & 69.01 \\
\midrule
GPT-OSS-120B    & 82.63 & 87.23 & 86.93 & \textbf{80.60}$^\dag$ & 72.37 & 76.81 & 81.10 \\
DeepSeek-V3.2   & \textbf{85.77}$^\dag$ & 89.86 & 87.25 & 79.98 & 73.06 & 77.32 & 82.21 \\
Claude-Opus-4.5 & 83.59 & \textbf{91.98}$^\dag$ & \textbf{89.46} & 80.11 & \textbf{73.54} & \textbf{77.58} & \textbf{82.71} \\
\bottomrule
\end{tabular}
\end{table}

\begin{figure}[!t]
\centering
\captionsetup[subfigure]{font=footnotesize}
\setlength{\subfigcapskip}{-3pt}
\begin{minipage}[c]{0.68\linewidth}
\centering
\renewcommand{\arraystretch}{1.30}
\setlength{\tabcolsep}{1.5pt}
\fontsize{11.4pt}{13pt}\selectfont
\begin{adjustbox}{max width=\linewidth}
\begin{tabular}{@{}l@{\hspace{2pt}}l@{\hspace{2.5pt}}cccccccccccc@{}}
\toprule
\textbf{Group} & \textbf{Setting} & \textbf{GSM8K} & \textbf{MATH} & \textbf{HotPot} & \textbf{SQuAD2} & \textbf{MBPP} & \textbf{HEval} & \textbf{TrivQA} & \textbf{NQ} & \textbf{MathQA} & \textbf{AIME} & \textbf{APPS} & \textbf{DS1K} \\
\midrule
\textbf{Full} & --- & \textbf{96.09} & \textbf{81.25} & \textbf{78.12} & \textbf{78.12} & \textbf{84.38} & \textbf{92.96} & \textbf{79.69} & \textbf{54.69} & \textbf{88.67} & \textbf{26.67} & \textbf{49.21} & \textbf{58.59} \\
\midrule
\multirow{4}{*}{\textit{Removal}}
 & $-$Review/Revise   & 96.09 & 81.25 & 71.09 & 73.44 & 84.38 & 92.96 & 79.69 & 50.78 & 88.67 & 26.67 & 49.21 & 58.59 \\
 & $-$ScEnsem/Aggr    & 96.09 & 81.25 & 75.78 & 78.12 & 84.38 & 92.96 & 79.69 & 53.12 & 88.67 & 20.00 & 49.21 & 58.59 \\
 & $-$Verify/Test     & 94.53 & 70.31 & 78.12 & 75.78 & 78.91 & 86.72 & 79.69 & 54.69 & 85.16 & 20.00 & 35.16 & 38.28 \\
 & Minimal (4 ops)    & 90.62 & 68.75 & 65.63 & 71.09 & 78.91 & 86.72 & 78.12 & 48.44 & 85.16 & 16.67 & 35.16 & 38.28 \\
\midrule
\multirow{3}{*}{\textit{Subst.}}
 & Prog$\to$Jupyter      & 96.09 & 84.38 & 78.12 & 78.12 & 85.94 & 92.96 & 79.69 & 54.69 & 88.67 & 30.00 & 52.34 & 62.50 \\
 & Custom$\to$GenSkills  & 96.09 & 81.25 & 84.38 & 78.91 & 84.38 & 92.96 & 79.69 & 57.03 & 88.67 & 26.67 & 49.21 & 58.59 \\
 & Review$\to$SelfCrit   & 96.09 & 81.25 & 71.88 & 77.34 & 84.38 & 92.96 & 73.44 & 50.00 & 88.67 & 26.67 & 49.21 & 58.59 \\
\midrule
\multirow{3}{*}{\textit{Addition}}
 & $+$Search       & 96.09 & 81.25 & 71.88 & 77.34 & 84.38 & 92.96 & 85.16 & 63.28 & 88.67 & 26.67 & 49.21 & 58.59 \\
 & $+$Calculator   & 96.09 & 80.47 & 78.12 & 78.12 & 84.38 & 92.96 & 79.69 & 54.69 & 88.28 & 30.00 & 49.21 & 58.59 \\
 & $+$Debugger     & 96.09 & 81.25 & 78.12 & 78.12 & 84.38 & 92.96 & 79.69 & 54.69 & 88.67 & 33.33 & 51.56 & 62.50 \\
\bottomrule
\end{tabular}
\end{adjustbox}
\\[5pt]
{\footnotesize (a) Operator-level ablation across 12 datasets}
\end{minipage}\hfill
\begin{minipage}[c]{0.31\linewidth}
\centering
\addtocounter{subfigure}{1}%
\subfigure[Overall pairwise]{\includegraphics[width=\linewidth]{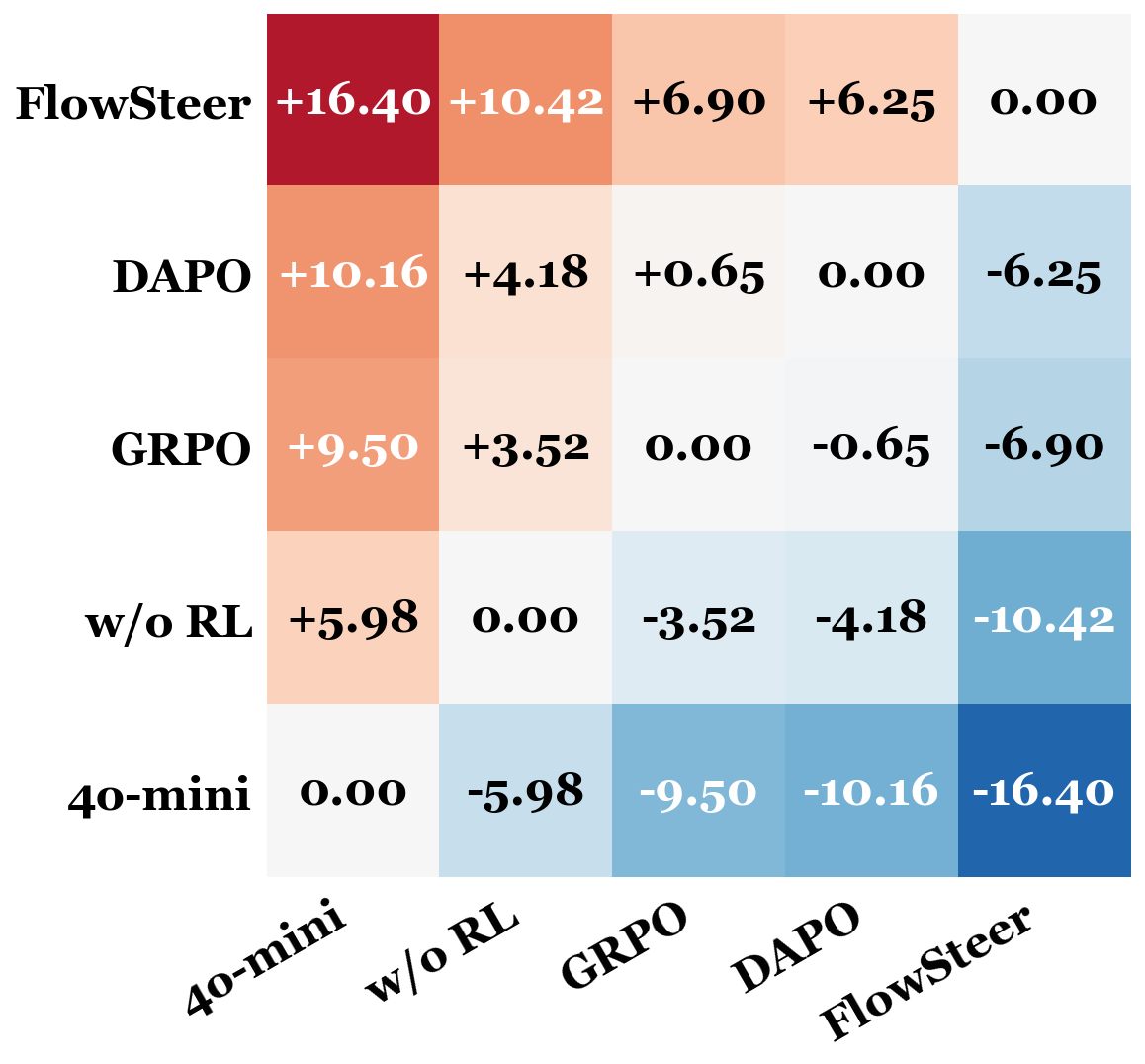}}
\end{minipage}\\[-2pt]
\noindent
\subfigure[Token consumption]{\includegraphics[width=0.333\linewidth]{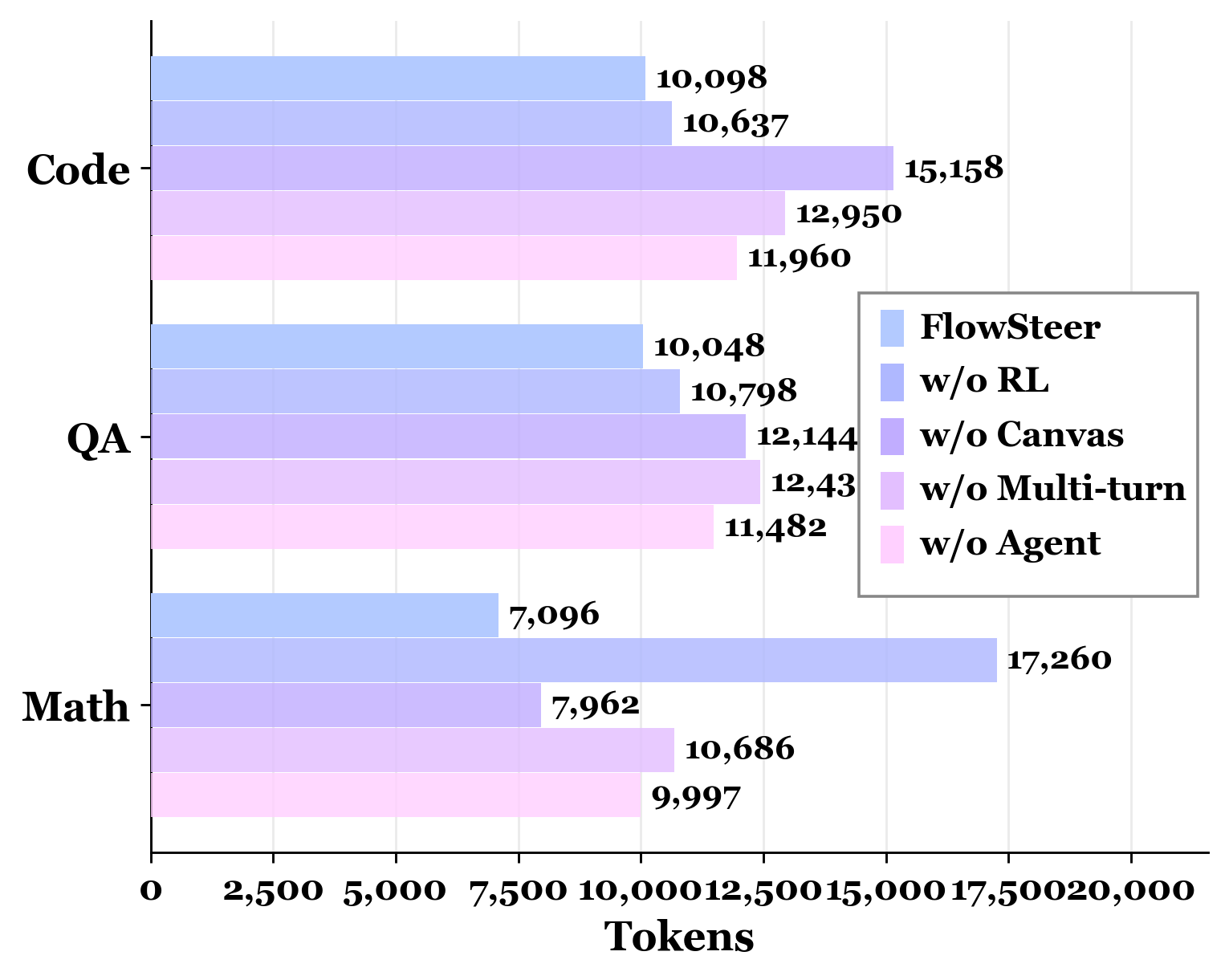}}\hfill%
\subfigure[Interactive turns]{\includegraphics[width=0.333\linewidth]{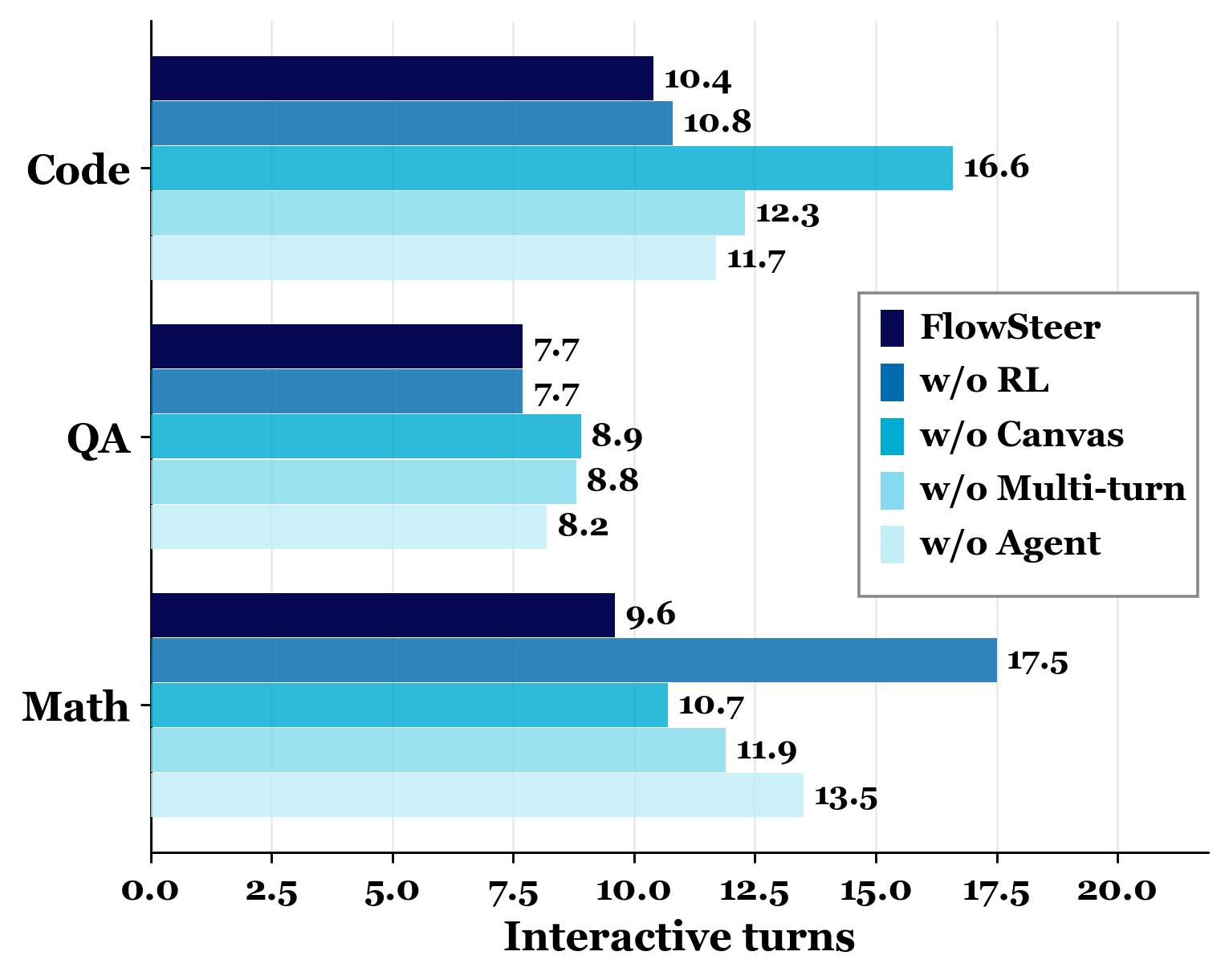}}\hfill%
\subfigure[Performance scaling]{\includegraphics[width=0.333\linewidth]{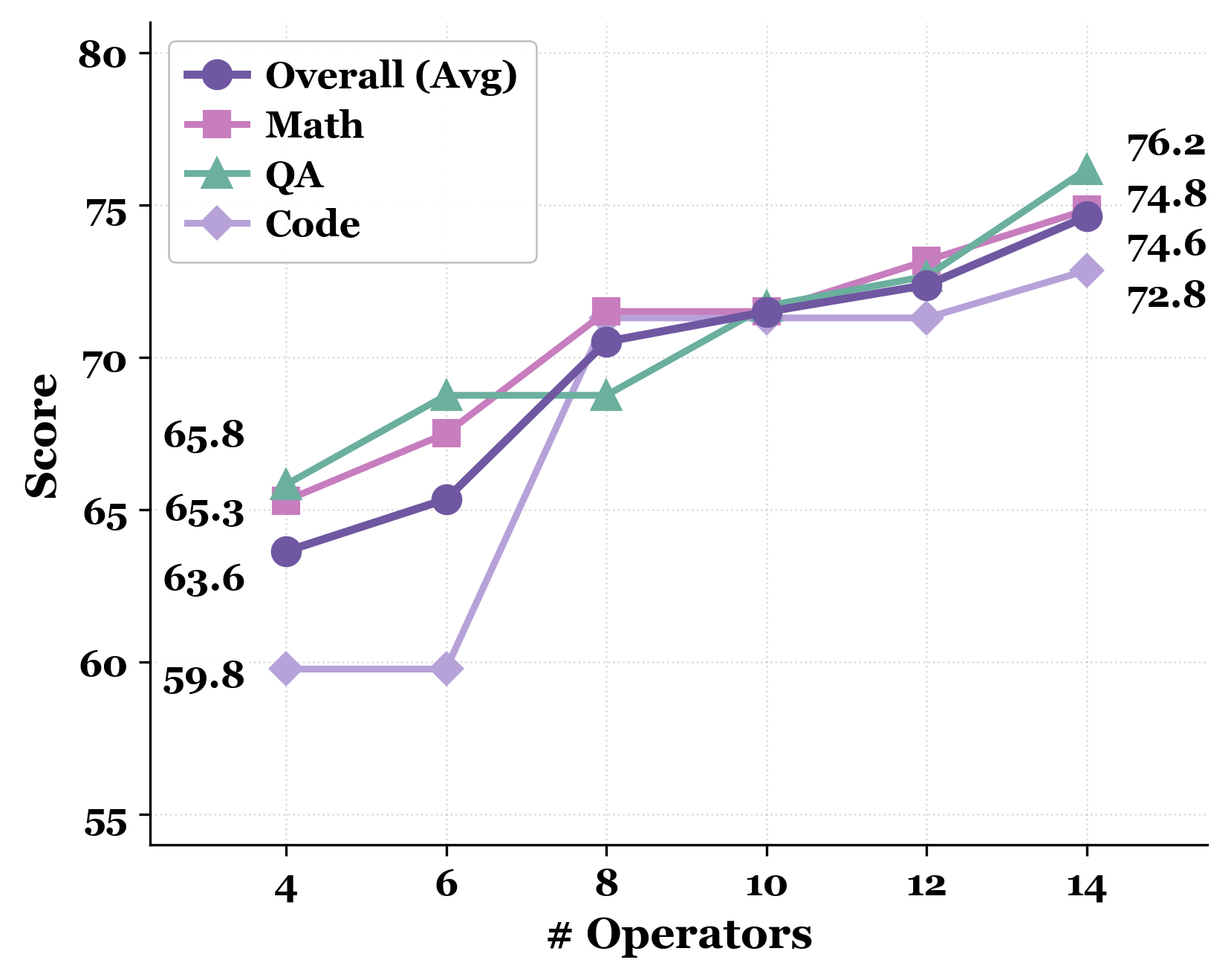}}
\caption{Ablation, scaling, and RL algorithm analysis. (a) Operator-level ablation. (b) Pairwise comparison across RL algorithms (averaged across IID and OOD benchmarks; row $-$ col). (c--d) Token consumption and interaction turns. (e) Performance scaling with operator-library size.}
\label{fig:rq5}
\end{figure}

\subsection{Transferability across LLM Backends (RQ3)}
\label{subsec:transferability}
As shown in Figure~\ref{fig:rq3}, a single \director{} trained with GPT-OSS-120B and then frozen yields consistent improvements across all six target backends on IID benchmarks (Figure~\ref{fig:rq3}(a)), with weaker baselines benefiting more, and stable gains across math, QA, and code on the aggregated OOD evaluation (Figure~\ref{fig:rq3}(c)). The transfer holds because the frozen policy conditions every edit on canvas feedback returned by the target executor (Eq.~\ref{eq:trajectory-dist}), so it carries over orchestration patterns such as decomposition, verification, and revision, and rebuilds the graph for the new executor rather than replaying a fixed one. Table~\ref{tab:train-exec} further varies the training executor itself: every frozen \director{} improves every target executor, including transfer from the strongest Claude-Opus-4.5 to the weaker GPT-4o-mini and Qwen3.5-9B. A stronger training executor raises the overall ceiling because its more reliable execution and verification feedback separates the effects of different edits more clearly, and a matched training and target executor adds a further gain ($\dag$). The training dynamics (Figure~\ref{fig:rq3}(b)) further show that locally deployed GPT-OSS-120B reaches higher reward with more stable optimization than GPT-4o-mini while converging to similar interaction-turn counts, indicating that lower-cost local executors can effectively replace API backbones for training, trailing a Claude-Opus-4.5-trained \director{} by only 1.6 points on average (Table~\ref{tab:train-exec}; details in Appendix~\ref{app:cross-exec}).

\subsection{Ablation Study (RQ4)}
\label{subsec:ablation}

As shown in Table~\ref{tab:ablation} and Figure~\ref{fig:rq5}, removing any component (agent, multi-turn interaction, canvas, or RL) degrades performance across all twelve datasets, with the largest drops on math/code (RL) and QA (canvas feedback); the three RL-objective adaptations (token mask, diversity-constrained reward, conditional release) are likewise complementary and each removal substantially hurts one or more task families. The control-structure check $R_{\mathrm{control}}$ serves as a training-time exploration signal rather than a template: without it, \flowr{} still exceeds both the fixed graph (w/o Agent) and selection among predefined templates on average, the full model is best on all twelve datasets (Appendix~\ref{app:control-prior}), and the \director{} still builds plain sequential graphs for simple QA at inference (Appendix~\ref{app:case_study_4}). Operator-level ablations (Figure~\ref{fig:rq5}(a)) show that adding or substituting operators maintains or improves performance while removing verification operators or shrinking the library to four operators causes the largest drops, so the \director{} composes operators by function rather than by memorized names; scaling the operator library from 4 to 14 operators steadily lifts Overall score from 63.6 to 74.6 (Figure~\ref{fig:rq5}(e)). The full model also reduces token consumption and interaction turns relative to all ablation variants (Figure~\ref{fig:rq5}(c--d)) while maintaining accuracy. Since w/o Agent and w/o Multi-turn correspond to static workflows and single-shot workflow generation, progressive editing adds short orchestration turns without raising total token cost, and workflow execution rather than \director{} orchestration accounts for most wall-clock time on every dataset (Appendix~\ref{app:efficiency}).

\subsection{Comparison of RL Algorithms (RQ5)}
\label{subsec:rl-compare}

As shown in Figure~\ref{fig:rq5}(b), \flowr{} outperforms GRPO and DAPO under identical training settings. The canvas-aware token mask removes biased gradient signals from environment-feedback tokens, and the structural-then-answer reward gate suppresses shortcut behaviors --- two failure modes GRPO and DAPO inherit from treating the trajectory as a structurally opaque sequence.

\section{Conclusion}

In this work, we propose \flowr{}, an end-to-end reinforcement learning framework for interactive workflow orchestration. By training a lightweight policy model to interact with an executable \canvas{}, \flowr{} learns transferable orchestration strategies from real execution feedback. Recent work points to further directions for such orchestrators, including evolving skill libraries, learned agent collaboration, self-play self-improvement, and continual self-evolution of science agents~\citep{zhang2026skillflow,huang2026collabflow,zhang2026r2orch,zhang2026scienceclaw}.

\begin{ack}
This work is supported in part by the Guangdong Basic and Applied Basic Research Foundation under Grant No. 2025A1515012968, in part by the Shenzhen Science and Technology Program under Grant No. JCYJ20240813113502004, in part by the National Natural Science Foundation of China under Grant Nos. 62001412 and 62366057, in part by the National Social Science Fund of China (NSSFC) under Grant No. 25CFX009, in part by Shenzhen Stability Science Program 2023, in part by the Guangdong Provincial Key Laboratory of Future Networks of Intelligence (Grant No. 2022B1212010001), in part by the Shenzhen Key Lab of Crowd Intelligence Empowered Low-Carbon Energy Network (Grant No. ZDSYS20220606100601002), in part by the Key Research and Development Program of Yunnan Province under Grant Nos. 202502AD080017, 202502AD080013, and 202402AA310056, and in part by the Special Fund for the Central Government to Guide Local Scientific and Technological Development under Grant No. 202407AB110003. The authors declare no competing interests.
\end{ack}

\bibliographystyle{plainnat}
\bibliography{example_paper}

\clearpage
\appendix

\begin{figure}[ht]
\centering
\begin{tcolorbox}[breakable, colback=gray!5!white, colframe=gray!75!black, title=System Prompt for Flow-Director, fonttitle=\bfseries\small, boxrule=0.5pt, arc=2pt, left=3pt, right=3pt, top=1pt, bottom=1pt]
\small

You are building a workflow step by step to solve the problem. In each turn, output EXACTLY ONE XML action (add/delete/modify/set\_prompt/finish or a structure add). The goal is to build a reliable workflow, not a single-shot answer. Keep your thinking brief (under 200 words) and focus on choosing the next action, not solving the whole problem yourself. Let the operators do the computation. If you use $<$think$>$...$<$/think$>$, you MUST output an $<$action$>$ tag AFTER it.

\vspace{0.2em}
\textbf{Available Operators (12 total).} Programmer: write and execute Python code. Plan: create solution strategy. Custom: natural language reasoning. Decompose: break into sub-problems. Test: run test cases. Review: evaluate quality. Verify: double-check result. Revise: fix issues. ScEnsemble: multiple solutions voting. Aggregate: combine results. AnswerGenerate: format final answer. Format: extract concise answer.

\vspace{0.2em}
\textbf{Actions (8 types).} The system supports 8 action types: add (add a single operator), finish (complete workflow building), parallel (add parallel branches), conditional (add conditional branch), loop (add loop structure), delete (remove a node), modify (change operator type), and set\_prompt (set custom prompt for an operator). All actions use XML format for reliable parsing.

\vspace{0.2em}
\textbf{Finish Policy.} Always add Format as the last step before finishing to extract concise answer. Before finishing, add Format to extract concise answer from the solution. When Format has extracted the answer and you are satisfied, output $<$action$>$finish$<$/action$>$. When the result is wrong or needs improvement, add more operators.

\end{tcolorbox}
\caption{The system prompt template for \director{}. The prompt instructs the agent to build workflows step by step through multi-turn interactions with the \canvas{}.}
\label{fig:system_prompt}
\end{figure}

\begin{table}[ht]
\centering
\caption{Summary of the 12 operators in $\mathcal{O}$ with input/output specifications.}
\label{tab:operator_summary}
\scriptsize
\renewcommand{\arraystretch}{0.95}
\setlength{\tabcolsep}{4pt}
\begin{tabular}{@{}p{1.7cm}p{1.2cm}p{1.8cm}p{2.0cm}p{6.1cm}@{}}
\toprule
\textbf{Operator} & \textbf{Category} & \textbf{Input} & \textbf{Output} & \textbf{Description} \\
\midrule
Plan & Planning & problem & approach, plan & Creates high-level solution strategy by identifying the overall approach and breaking it into actionable steps. \\
Decompose & Planning & problem & sub\_problems & Breaks complex problems into smaller, independent sub-problems that can be solved separately. \\
\midrule
Programmer & Solving & problem, analysis & code, output & Writes and executes Python code to compute answers; used for mathematical calculations and code generation tasks. \\
Custom & Solving & input, context & response & Performs natural language reasoning without code execution; used for QA, analysis, and explanations. \\
AnswerGenerate & Solving & input & thought, answer & Generates structured answers with explicit reasoning chains; similar to Custom but with formatted output. \\
\midrule
Test & Verification & code, tests & pass/fail, solution & Executes unit tests against generated code and triggers automatic revision if tests fail; code-specific. \\
Review & Verification & solution & is\_correct, feedback & Evaluates solution quality through critique and provides detailed feedback for potential revision. \\
Verify & Verification & answer & is\_correct, answer & Independently re-solves the problem to verify correctness; does not execute code, uses logical reasoning. \\
\midrule
Revise & Revision & solution, feedback & revised\_solution & Fixes issues in solutions based on feedback from Review or Test operators. \\
\midrule
ScEnsemble & Ensemble & solutions & selected\_solution & Implements majority voting across multiple candidate solutions to select the most consistent answer. \\
Aggregate & Ensemble & sub\_answers & aggregated & Combines results from parallel branches or sub-problems into a unified answer. \\
\midrule
Format & Formatting & solution & final\_answer & Extracts concise final answer from verbose solution; required as the last step before finishing. \\
\bottomrule
\end{tabular}
\end{table}


\section{Prompts and Operator Library in FlowSteer}
\label{app:prompts}

\subsection{System Prompt for Flow-Director}
\label{app:system_prompt}

As shown in Figure~\ref{fig:system_prompt}, we adopt the system prompt template for the Flow-Director, where the agent alternates between internal reasoning in \texttt{<think>} and communicates with the Workflow Canvas by \texttt{<action>}. The prompt instructs the agent to output exactly one XML action per turn, keeping thinking brief and focusing on action selection rather than problem-solving. The workflow evolves step by step until finalized with a \texttt{finish} action.


\subsection{Operator Library $\mathcal{O}$}
\label{app:operators}

As shown in Table~\ref{tab:operator_summary}, the operator library $\mathcal{O}$ includes 12 operators in six categories: planning (Plan, Decompose), solving (Programmer, Custom, AnswerGenerate), verification (Test, Review, Verify), revision (Revise), ensemble (ScEnsemble, Aggregate), and formatting (Format). The Programmer operator executes Python code in a sandboxed environment with scientific libraries under a 30-second timeout.


\begin{figure}[ht]
\centering
\begin{tcolorbox}[colback=gray!3!white, colframe=gray!60!black, title=Multi-turn Interaction Example, fonttitle=\bfseries\small, boxrule=0.5pt, arc=2pt]
\small

\textbf{Problem}: ``Find the sum of all positive integers $n$ such that $n^2 + 12n - 2007$ is a perfect square.''

\vspace{0.5em}
\hrule
\vspace{0.5em}

\textbf{Round 1} (BUILDING $\rightarrow$ AWAITING\_PROMPT)\\
\textit{Flow-Director}:
\begin{lstlisting}[basicstyle=\ttfamily\scriptsize, frame=none]
<think>This is a number theory problem. Start with Plan.</think>
<action>add</action><operator>Plan</operator>
\end{lstlisting}
\textit{Canvas Feedback}: ``Added Plan as node\_1. State: AWAITING\_PROMPT.''

\vspace{0.3em}
\textbf{Round 2} (AWAITING\_PROMPT $\rightarrow$ BUILDING)\\
\textit{Flow-Director}:
\begin{lstlisting}[basicstyle=\ttfamily\scriptsize, frame=none]
<action>set_prompt</action><target>node_1</target>
<prompt>Analyze the equation and identify the approach.</prompt>
\end{lstlisting}
\textit{Canvas Feedback}: ``Prompt set. State: BUILDING. DSL: Plan''

\vspace{0.3em}
\textbf{Round 3-4}: Add Programmer operator with prompt for algebraic computation.

\vspace{0.3em}
\textbf{Round 5-6}: Add Verify operator with prompt for result validation.

\vspace{0.3em}
\textbf{Round 7-8}: Add Format operator with prompt for answer extraction.

\vspace{0.3em}
\textbf{Round 9} (BUILDING $\rightarrow$ Finished)\\
\textit{Flow-Director}:
\begin{lstlisting}[basicstyle=\ttfamily\scriptsize, frame=none]
<think>Workflow complete. Ready to finish.</think>
<action>finish</action>
\end{lstlisting}
\textit{Canvas Feedback}: ``Workflow finished. Final DSL: Plan -> Programmer -> Verify -> Format. Answer: 89''

\end{tcolorbox}
\caption{A complete multi-turn interaction example showing the two-step mechanism (add + set\_prompt) for workflow construction.}
\label{fig:interaction_example}
\end{figure}

\subsection{Multi-turn Interaction Process}
\label{app:interaction}

The Flow-Director builds workflows through multi-turn interactions with the Workflow Canvas. Each interaction round consists of the Flow-Director generating an action and the Canvas providing execution feedback.

\subsubsection{Environment State Machine}

The Workflow Canvas maintains a finite state machine with two states:

\begin{itemize}
    \item \textbf{BUILDING}: The normal state where the Flow-Director can execute \texttt{add}, \texttt{delete}, \texttt{modify}, or \texttt{finish} actions.
    \item \textbf{AWAITING\_PROMPT}: After adding an operator, the Canvas transitions to this state, requiring the Flow-Director to specify a custom prompt via \texttt{set\_prompt} before returning to BUILDING.
\end{itemize}

This two-step mechanism reduces the cognitive load on the policy model by separating structural decisions (which operator to add) from content decisions (what prompt to use), improving the effectiveness of small-scale models.

\subsubsection{Action Types}

The 8 action types available to the Flow-Director are summarized in Table~\ref{tab:action_types}. The basic actions (add, finish, delete, modify, set\_prompt) enable sequential workflow construction, where the Flow-Director progressively builds the operator graph through iterative additions and modifications. The advanced actions (parallel, conditional, loop) extend the expressiveness to complex control structures, allowing the Flow-Director to construct workflows with branching, conditional execution, and iterative refinement patterns.

\begin{table}[ht]
\centering
\caption{Action types in FlowSteer.}
\label{tab:action_types}
\small
\begin{tabular*}{\columnwidth}{@{\extracolsep{\fill}}cc@{}}
\toprule
\textbf{Action} & \textbf{Description} \\
\midrule
\texttt{add} & Add a single operator \\
\texttt{finish} & Complete workflow building \\
\texttt{delete} & Remove a node \\
\texttt{modify} & Change operator type \\
\texttt{set\_prompt} & Set custom prompt \\
\midrule
\texttt{parallel} & Add parallel branches \\
\texttt{conditional} & Add conditional branch \\
\texttt{loop} & Add loop structure \\
\bottomrule
\end{tabular*}
\end{table}

\subsubsection{Complete Interaction Example}

A complete multi-turn interaction between the Flow-Director (small-scale policy model) and the Workflow Canvas (large-scale LLM backend) is illustrated in Figure~\ref{fig:interaction_example}. The process begins with the Flow-Director analyzing the problem and selecting an appropriate initial operator, which establishes the reasoning strategy for subsequent workflow construction. The Flow-Director then engages in iterative exchanges: after each canvas feedback, it reflects on the current workflow state, selects the next operator, and specifies task-specific prompts. Through successive rounds of construction and verification, the workflow gradually takes shape, and the Canvas executes the completed workflow to produce the final answer.

\section{Theoretical Proofs}
\label{app:proofs}
\setcounter{proposition}{0}

In this appendix, we provide detailed proofs for Propositions 1--3 stated in the main text. We first introduce the notation and assumptions, then present each proof in turn.

\paragraph{Notation.}
Given a task $q$, a workflow is represented as a directed graph $\mathcal{G}=(V,E,\mathrm{attr})$, where each node $v\in V$ is bound to an operator $\mathrm{op}(v)\in\mathcal{O}$ with attributes $\mathrm{attr}(v)=(\mathrm{param}(v),\mathrm{prompt}(v),\dots)$. The executor schedules nodes according to dependency edges $E$ and produces output $y=\mathrm{Execute}(\mathcal{G},q,\mathcal{M}_{\mathrm{exec}})$. During multi-turn orchestration, the policy model (\director{}) interacts with the canvas (\canvas{}), forming a trajectory
\begin{equation}
\tau=\{(a_t^{\mathrm{think}},a_t,o_t^{\mathrm{exec}})\}_{t=1}^{T},
\end{equation}
where $a_t^{\mathrm{think}}$ denotes the reflection text that summarizes the current state and identifies potential issues, $a_t$ denotes the edit action comprising an action type and its content, and $o_t^{\mathrm{exec}}$ denotes the execution and validation feedback returned by the canvas environment. The operator library consists of 12 functional operators:
\begin{equation}
\begin{split}
\mathcal{O}=\{&\texttt{Plan},\texttt{Decompose},\texttt{Programmer},\texttt{Custom},\texttt{AnswerGenerate},\texttt{Test}, \\
&\texttt{Review},\texttt{Verify},\texttt{Revise},\texttt{ScEnsemble},\texttt{Aggregate},\texttt{Format}\},
\end{split}
\end{equation}
where each operator implements a specific cognitive function in the problem-solving process. The action type set consists of 8 types:
\begin{equation}
\mathcal{A}_{\mathrm{type}}=\{\texttt{add},\texttt{delete},\texttt{modify},\texttt{set\_prompt},\texttt{finish},\texttt{parallel},\texttt{conditional},\texttt{loop}\},
\end{equation}
where the first five types can be understood as local editing actions, and the last three types are used to explicitly generate control structures (parallel, conditional, loop) that form long-range control flows.

\paragraph{Assumptions.}
To obtain formal proofs, we introduce three mild and interpretable assumptions. They do not require the real system to be perfect, but are sufficient to support the theoretical justification of why the propositions hold.

\textit{Assumption 1 (Cognitive Decomposability).} For the task families considered in this work (multi-hop QA, standard QA, mathematical reasoning, code generation), there exists a goal-directed problem-solving procedure that can be decomposed into a finite combination of cognitive primitives (defined below), and can be represented as a directed graph with conditionals and loops.

\textit{Assumption 2 (Informative Canvas Feedback).} The canvas feedback $o_t^{\mathrm{exec}}$ is informative about whether the current workflow is executable, whether it satisfies structural constraints, and which local modifications can fix errors. Formally, there exists non-zero probability that the feedback changes the agent's posterior distribution over the optimal next edit, i.e., $I(Z; o_t^{\mathrm{exec}} \mid H_{t-1}) > 0$, where $Z$ denotes the latent variable related to the solution (e.g., the correct answer or a feasible workflow).

\textit{Assumption 3 (Repairability).} When a workflow is in a non-executable or constraint-violating state, the canvas can provide sufficiently localized failure reasons or repair suggestions, such that there exist editing actions that can push it toward a more feasible state. Formally, there exists a potential function $\Phi(\mathcal{G})$ measuring the distance to the feasible set, and there exist actions such that $\mathbb{E}[\Phi(\mathcal{G}_t) \mid H_{t-1}] < \Phi(\mathcal{G}_{t-1})$.

\subsection{Proof of Proposition 1}
\label{app:proof1}

\begin{proposition}[Operator-Action Cognitive Completeness]
\label{prop:complete:app}
Let the operator library $\mathcal{O}$ and action type set $\mathcal{A}_{\mathrm{type}}$ be defined as above. Under Assumption 1, for any task $q$, there exists a finite-length action sequence $\{a_t\}_{t=1}^{T}$ with each action type belonging to $\mathcal{A}_{\mathrm{type}}$, such that starting from an empty canvas $\mathcal{G}_0$, the iterative updates $\mathcal{G}_t = \mathrm{Update}(\mathcal{G}_{t-1}, a_t, o_t^{\mathrm{exec}})$ can construct a terminal workflow $\mathcal{G}_T$ that implements a complete cognitive control loop, thereby covering the key problem-solving steps required for the task types considered in this work (mathematics, QA, code generation).
\end{proposition}

\begin{proof}
This proposition requires proving two things: (i) functional completeness, showing that the operator library covers the core cognitive modules required for goal-directed problem solving; and (ii) structural completeness, showing that the action space can organize these modules into control flow graphs with sufficient expressive power (sequential, branching, looping, parallel), thereby constructing executable workflow programs. We divide the proof into three parts: first defining the cognitive primitive set $\mathcal{C}$, then proving that $\mathcal{O}$ covers $\mathcal{C}$, and finally proving that $\mathcal{A}_{\mathrm{type}}$ can construct any structured workflow graph composed of these modules.

\textbf{(i) Cognitive primitives and functional coverage.} We first define a set of cognitive primitives consistent with goal-directed problem solving. Let the cognitive primitive set be
\begin{equation}
\mathcal{C}=\{\mathbf{P},\mathbf{D},\mathbf{E},\mathbf{M},\mathbf{R},\mathbf{I},\mathbf{O}\},
\end{equation}
where each element represents a fundamental cognitive process in problem solving. Specifically, $\mathbf{P}$ (Planning) forms goals, strategies, and resource budgets, determining what to do first, what to do later, and to what extent; $\mathbf{D}$ (Decomposition) breaks the overall task into operable subgoals with explicit dependency relationships; $\mathbf{E}$ (Execution) performs solving, reasoning, or externalized computation on a subgoal, including symbolic computation, code generation and execution, and other operations; $\mathbf{M}$ (Monitoring) verifies, reviews, unit tests, and checks constraints on intermediate products, determining whether to continue, backtrack, or redo; $\mathbf{R}$ (Revision) performs repairs based on monitoring results, including rewriting, supplementing, replacing, and adjusting prompts or parameters; $\mathbf{I}$ (Integration) fuses, votes, disambiguates, and ensures consistency of results from multiple branches or subproblems; and $\mathbf{O}$ (Output) outputs the final result in the format required by the task, including answer extraction and structured presentation.

We now construct a surjective mapping $\psi: \mathcal{O} \to \mathcal{C}$ to show that for each cognitive primitive $\mathbf{c} \in \mathcal{C}$, there exists at least one operator $o \in \mathcal{O}$ such that $\psi(o) = \mathbf{c}$. Consider the following correspondence:

For the planning primitive $\mathbf{P}$, the operator \texttt{Plan} directly implements high-level plan generation, including strategy formulation, step sequencing, budget allocation, and stopping condition specification. Therefore, we have
\begin{equation}
\psi(\texttt{Plan}) = \mathbf{P}.
\end{equation}

For the decomposition primitive $\mathbf{D}$, the operator \texttt{Decompose} expresses task decomposition in the form of subtask lists or subproblem graphs, explicitly representing the dependency structure among subtasks. Therefore, we have
\begin{equation}
\psi(\texttt{Decompose}) = \mathbf{D}.
\end{equation}

For the execution primitive $\mathbf{E}$, execution includes not only natural language reasoning but also externalized solving. The operator \texttt{Custom} covers general reasoning and retrieval-based processing, capable of hosting various tools and prompt templates according to implementation; the operator \texttt{Programmer} covers executable code generation and execution, typically handling symbolic computation, scripted reasoning, and data processing; the operator \texttt{AnswerGenerate} covers generating final natural language answers from obtained key evidence or intermediate conclusions, which can be viewed as decoding or expression. Therefore, we have
\begin{equation}
\psi(\texttt{Custom}) = \psi(\texttt{Programmer}) = \psi(\texttt{AnswerGenerate}) = \mathbf{E}.
\end{equation}

For the monitoring primitive $\mathbf{M}$, the operators \texttt{Verify} and \texttt{Review} perform consistency checking and critical evaluation on intermediate results; the operator \texttt{Test} performs executable testing on code solutions. These are all typical error monitoring and quality assessment processes. Therefore, we have
\begin{equation}
\psi(\texttt{Verify}) = \psi(\texttt{Review}) = \psi(\texttt{Test}) = \mathbf{M}.
\end{equation}

For the revision primitive $\mathbf{R}$, the operator \texttt{Revise} explicitly implements error correction and rewriting in a feedback-based repair manner. Therefore, we have
\begin{equation}
\psi(\texttt{Revise}) = \mathbf{R}.
\end{equation}

For the integration primitive $\mathbf{I}$, when there are multiple branches, multiple candidates, or multiple subproblems, fusion is needed. The operator \texttt{Aggregate} handles aggregation and consistency ensuring; the operator \texttt{ScEnsemble} handles selection and ensemble of diverse candidates. Therefore, we have
\begin{equation}
\psi(\texttt{Aggregate}) = \psi(\texttt{ScEnsemble}) = \mathbf{I}.
\end{equation}

For the output primitive $\mathbf{O}$, the operator \texttt{Format} extracts, structures, and presents results in the target format. Therefore, we have
\begin{equation}
\psi(\texttt{Format}) = \mathbf{O}.
\end{equation}

For each $\mathbf{c} \in \mathcal{C}$, the above construction provides at least one $o \in \mathcal{O}$ such that $\psi(o) = \mathbf{c}$. Hence $\psi$ is surjective onto $\mathcal{C}$, establishing that $\mathcal{O}$ achieves functional coverage of $\mathcal{C}$.

\textbf{(ii) Structural completeness of the action space.} Let $\mathfrak{G}(\mathcal{O})$ denote the set of all finite workflow graphs with nodes labeled by operators from $\mathcal{O}$ and composed using structured control constructs (sequential, conditional, loop, parallel). We show that for any target graph $\mathcal{G}^\star \in \mathfrak{G}(\mathcal{O})$, there exists a finite-length action sequence $\{a_t\}_{t=1}^{T}$ with action types belonging to $\mathcal{A}_{\mathrm{type}}$, such that starting from the empty graph $\mathcal{G}_0$, the iterative updates yield $\mathcal{G}_T = \mathcal{G}^\star$. We provide a constructive proof, which is equivalent to showing that these editing actions can assemble any target graph.

\textit{Step 1: Node construction and attribute assignment (basic editing closure).} For any target graph $\mathcal{G}^\star$ with node set $V^\star$, we traverse $v \in V^\star$ in any order. For each node, we use the action \texttt{add} to create a new node in the current canvas and specify its operator type $\mathrm{op}(v) \in \mathcal{O}$. We then use the action \texttt{set\_prompt} to set prompts and constraints for that node. If necessary, we use the action \texttt{modify} to write in parameters. If there are redundant or erroneous nodes, we use the action \texttt{delete} to remove them. Since \texttt{add}, \texttt{delete}, \texttt{modify}, and \texttt{set\_prompt} allow arbitrary finite discrete modifications to the node set and attributes, we can construct within finite steps a node set and attribute annotation isomorphic to $\mathcal{G}^\star$.

\textit{Step 2: Control structure construction (structured control closure).} Control structures in the target graph can be divided into three categories: conditional branching, looping, and parallel execution. If $\mathcal{G}^\star$ contains conditional structures (if/else or selective execution), we use the action \texttt{conditional} to introduce conditional gate nodes or conditional edges in the graph, and write in condition predicates via \texttt{set\_prompt} or \texttt{modify}. The condition can be a Boolean signal output by some checking or verification node. If $\mathcal{G}^\star$ contains loop structures (repeated execution until some criterion is met), we use the action \texttt{loop} to introduce iteration blocks. The stopping condition can similarly be produced by verification or test nodes, with the loop termination rule written in via \texttt{set\_prompt} or \texttt{modify}. If $\mathcal{G}^\star$ contains parallel structures (multiple subproblems or candidates expanded simultaneously), we use the action \texttt{parallel} to organize several branches into concurrent subgraphs, with subsequent integration via operators such as \texttt{Aggregate} or \texttt{ScEnsemble} to merge branch results.

\textit{Step 3: Termination (completing construction).} When the current canvas graph aligns with $\mathcal{G}^\star$ in nodes, attributes, and control structures, we use the action \texttt{finish} to terminate editing. Since $\mathcal{G}^\star$ is a finite graph, the above process is finite, hence there exists finite $T$ such that $\mathcal{G}_T = \mathcal{G}^\star$.

\textbf{(iii) Coverage of task types.} We now show that the operator library and action space can cover the task types considered in this work by demonstrating the existence of constructible cognitive program templates for each task type.

For mathematical reasoning tasks, a typical cognitive program follows the pattern: $\texttt{Plan} \to \texttt{Decompose} \to (\texttt{Custom} \text{ or } \texttt{Programmer}) \to \texttt{Verify} \to (\texttt{Revise} + \texttt{loop}) \to \texttt{Format}$. This template first formulates a solution strategy, decomposes the problem into solvable steps, executes reasoning or symbolic computation, verifies intermediate results, revises if errors are detected (potentially looping), and finally formats the answer.

For question answering tasks including multi-hop QA, a typical cognitive program follows the pattern: $\texttt{Plan} \to \texttt{Decompose} \to \texttt{parallel}(\texttt{Custom}/\texttt{AnswerGenerate}) \to \texttt{Aggregate} \to \texttt{Review}/\texttt{Verify} \to \texttt{Format}$. This template formulates a retrieval and reasoning strategy, decomposes into sub-questions that can be processed in parallel, aggregates evidence from multiple sources, reviews for consistency, and formats the final answer.

For code generation tasks, a typical cognitive program follows the pattern: $\texttt{Plan} \to \texttt{Decompose} \to \texttt{Programmer} \to \texttt{Test} \to (\texttt{Revise} + \texttt{loop}) \to \texttt{Format}$. This template plans the code structure, decomposes into implementable modules, generates code, tests for correctness, revises based on test feedback (looping until tests pass), and formats the output.

These templates use only operators from $\mathcal{O}$ and control structures from $\mathcal{A}_{\mathrm{type}}$. Therefore, under Assumption 1 (tasks can be decomposed into combinations of cognitive primitives), for any task $q$, there exists some $\mathcal{G}^\star$ implementing the corresponding cognitive program. By structural completeness established above, there exists an action sequence constructing $\mathcal{G}^\star$.

In summary, functional coverage ensures that the operator library implements all cognitive primitives required for goal-directed problem solving, while structural completeness ensures that the action space can construct any workflow graph composed of these primitives with arbitrary control structures. Together, they establish that for any task decomposable into cognitive primitive combinations, there exists an action sequence constructing a workflow that implements the complete cognitive control loop covering planning, decomposition, execution, monitoring, revision, integration, and output.
\end{proof}

\subsection{Proof of Proposition 2}
\label{app:proof2}

\begin{proposition}[Decoupled Transfer with Monotonic Improvement]
\label{prop:multiturn:app}
Let $\pi_\theta$ be trained on a canvas $\mathcal{C}=(\mathcal{G}_t,\mathcal{O},\mathcal{M}_{\mathrm{exec}},d^{\mathrm{lib}})$ and deployed without parameter updates on a target canvas $\mathcal{C}'=(\mathcal{G}_t,\mathcal{O}',\mathcal{M}'_{\mathrm{exec}},{d'}^{\mathrm{lib}})$ that shares the action interface $\mathcal{A}_{\mathrm{type}}$ and whose feedback satisfies Assumptions 2--3. Then $\pi_\theta$ is a valid editing policy on $\mathcal{C}'$, and its multi-turn orchestration possesses the monotonic improvement property in expectation: as the number of turns $t$ increases, (1) the agent's uncertainty about the correct output or feasible workflow does not increase; (2) consequently, the success probability of generating an executable and correct workflow (and final answer) does not decrease. Furthermore, multi-turn interaction under the same budget is at least as good as single-turn open-loop generation.
\end{proposition}

\begin{proof}
We first show that the decoupling makes the trained policy well defined on the target canvas, and then characterize the value of multi-turn interaction on that canvas using a Bayesian risk potential function: each turn obtains canvas feedback $o_t^{\mathrm{exec}}$ that provides information about the hidden correct solution, thereby concentrating the posterior distribution. The monotonic decrease of the potential function in expectation corresponds to the monotonic increase of reliability and success rate.

\textbf{(0) Validity of the frozen policy on the target canvas.} By Eq.~\ref{eq:design-execute}, the executor enters the interaction only through the feedback $o_t^{\mathrm{exec}}$ appended to $H_t$, and the operator library enters only through the description $d^{\mathrm{lib}}$ in $H_0$; by Eq.~\ref{eq:designer-grad}, both enter training only through the scalar reward $R(\tau)$. The policy therefore maps any history $H'_t$ produced on $\mathcal{C}'$ to a distribution over the same action space $\mathcal{A}_{\mathrm{type}}$, and each sampled edit is applied by the same graph-update operator. Swapping $\mathcal{M}_{\mathrm{exec}}$ for $\mathcal{M}'_{\mathrm{exec}}$ or $\mathcal{O}$ for $\mathcal{O}'$ thus changes the feedback the policy conditions on, while the policy itself, its editing interface, and the form of its gradient stay unchanged. Steps (i)--(v) below use only Bayes' rule and Assumptions 2--3 on the feedback, so they apply verbatim to the histories generated on $\mathcal{C}'$.

\textbf{(i) Posterior distribution and Bayesian risk potential.} Let $Z$ denote the latent variable related to the task solution. According to the problem setting, $Z$ can be taken as the correct answer $y_q^\star$, or as a representative element from the set of feasible workflows that lead to the correct answer. At interaction turn $t$, the history is defined as
\begin{equation}
H_t = \{(a_1^{\mathrm{think}}, a_1, o_1^{\mathrm{exec}}), \ldots, (a_t^{\mathrm{think}}, a_t, o_t^{\mathrm{exec}})\},
\end{equation}
where $H_t$ encodes all reflection texts, edit actions, and canvas feedback up to turn $t$. The initial history $H_0 = [q \oplus d^{\mathrm{lib}} \oplus a^{\mathrm{tmpl}}]$ holds the task input and contains no interaction step. Given history $H_t$, we define the posterior distribution over $Z$ as
\begin{equation}
\pi_t(z) \triangleq \mathbb{P}(Z = z \mid H_t),
\end{equation}
where $\pi_t(z)$ represents the probability that the true solution is $z$ given the interaction history up to turn $t$. We define the Bayes accuracy function as the maximum posterior probability:
\begin{equation}
A(H_t) \triangleq \max_{z} \pi_t(z),
\end{equation}
where $A(H_t)$ represents how concentrated the posterior is on the most likely solution. We then define the Bayes risk potential function as
\begin{equation}
V(H_t) \triangleq 1 - A(H_t) = 1 - \max_{z} \pi_t(z),
\end{equation}
where $V(H_t)$ measures the remaining uncertainty. Smaller $V(H_t)$ indicates more concentrated posterior and higher probability of the most likely correct solution.

\textbf{(ii) Supermartingale property of the risk potential.} We now show that the expected risk potential does not increase across turns, i.e., for any $t \geq 1$,
\begin{equation}
\mathbb{E}[V(H_t) \mid H_{t-1}] \leq V(H_{t-1}),
\end{equation}
with strict inequality when the feedback at turn $t$ provides non-zero information gain about $Z$.

The key observations are twofold. First, the posterior vector satisfies the martingale property. Since $H_t = H_{t-1} \oplus (a_t^{\mathrm{think}}, a_t, o_t^{\mathrm{exec}})$ is obtained by appending new observations to $H_{t-1}$, by Bayes' rule, for any $z$ we have
\begin{equation}
\mathbb{E}[\pi_t(z) \mid \mathcal{F}_{t-1}] = \pi_{t-1}(z),
\end{equation}
where $\mathcal{F}_{t-1} = \sigma(H_{t-1})$ is the natural filtration generated by the interaction history up to turn $t-1$. This martingale property states that the expected posterior at turn $t$, conditioned on information at turn $t-1$, equals the posterior at turn $t-1$.

Second, $V(\cdot)$ is a concave function over distributions. To see this, consider the potential function over the probability simplex:
\begin{equation}
\phi(p) \triangleq 1 - \max_{z} p_z, \quad p \in \Delta_{|Z|-1},
\end{equation}
where $p$ is a probability vector over the solution space, and $\Delta_{|Z|-1}$ is the probability simplex of such vectors. Since $\max_z p_z$ is a convex function of $p$ (as the pointwise maximum of linear functions), its negation $-\max_z p_z$ is concave, and hence $\phi(p) = 1 - \max_z p_z$ is also concave.

Applying Jensen's inequality for concave functions, we obtain the contraction of expected risk:
\begin{equation}
\mathbb{E}[V(\pi_t) \mid \mathcal{F}_{t-1}] \leq V(\mathbb{E}[\pi_t \mid \mathcal{F}_{t-1}]) = V(\pi_{t-1}),
\end{equation}
where the equality uses the martingale property established above. This yields
\begin{equation}
\mathbb{E}[V(H_t) \mid H_{t-1}] \leq V(H_{t-1}).
\end{equation}

When the feedback provides non-zero information gain about $Z$ (i.e., under Assumption 2), the posterior $\pi_t$ undergoes strict contraction with non-zero probability, meaning it does not merely preserve its mean but actually becomes more concentrated. In this case, Jensen's inequality becomes strict in expectation, yielding
\begin{equation}
\mathbb{E}[V(H_t) \mid H_{t-1}] < V(H_{t-1})
\end{equation}
with positive probability, leading to strict improvement.

\textbf{(iii) Monotonic improvement over multiple turns.} Taking unconditional expectation and iterating the supermartingale relation yields:
\begin{equation}
\mathbb{E}[V(H_t)] \leq \mathbb{E}[V(H_{t-1})] \leq \cdots \leq \mathbb{E}[V(H_0)].
\end{equation}
This establishes that the expected Bayes risk monotonically decreases (or stays constant) as the number of interaction turns increases.

To quantify the improvement, we define the accuracy gain at turn $t$ as
\begin{equation}
\Delta_t \triangleq \mathbb{E}[V(H_{t-1})] - \mathbb{E}[V(H_t)] \geq 0,
\end{equation}
where $\Delta_t$ represents the expected one-step reduction of the Bayes risk potential at turn $t$. Summing over all turns, the expected Bayes risk after $t$ turns satisfies:
\begin{equation}
\mathbb{E}[V(H_t)] = \mathbb{E}[V(H_0)] - \sum_{s=1}^{t} \Delta_s,
\end{equation}
where each $\Delta_s$ accumulates the expected risk decrease at turn $s$. Substituting into the definition of accuracy, we obtain:
\begin{equation}
\mathbb{E}[A(H_t)] = 1 - \mathbb{E}[V(H_0)] + \sum_{s=1}^{t} \Delta_s,
\end{equation}
where $\mathbb{E}[A(H_t)]$ represents the expected Bayes accuracy after $t$ turns. This shows that expected accuracy monotonically increases with the number of turns.

\textbf{(iv) From uncertainty reduction to error probability reduction.} Let $\hat{z}_t = \arg\max_z \pi_t(z)$ be the Bayes optimal estimate at turn $t$, i.e., the solution with maximum posterior probability. Under 0-1 loss (where we incur loss 1 if wrong and 0 if correct), the Bayes optimal decision rule minimizes expected loss by choosing $\hat{z}_t$. The error probability of this estimator satisfies
\begin{equation}
\mathbb{P}(\hat{z}_t \neq Z \mid H_t) = 1 - \max_z \pi_t(z) = V(H_t).
\end{equation}
This equality follows because the probability of error under the Bayes optimal rule equals one minus the probability of the chosen class, which is precisely $1 - \max_z \pi_t(z)$.

Therefore, monotonically non-increasing $\mathbb{E}[V(H_t)]$ is equivalent to monotonically non-increasing expected error rate, or equivalently, monotonically non-decreasing expected correctness rate. This proves conclusions (1) and (2) of the proposition.

\textbf{(v) Comparison with single-turn generation.} A single-turn strategy makes only one decision and directly outputs the final workflow or answer, which is equivalent to setting $T = 1$. A multi-turn strategy with $T > 1$ can still choose to execute the \texttt{finish} action at the first turn, thereby terminating immediately. Therefore, the strategy space of multi-turn interaction contains the single-turn strategy as a special case.

By this inclusion relationship of strategy spaces, the optimal value achievable by multi-turn interaction is no worse than that of single-turn interaction:
\begin{equation}
\max_{\pi \in \Pi_{\text{multi-turn}}} \mathbb{E}[A(H_T)] \geq \max_{\pi \in \Pi_{\text{single-turn}}} \mathbb{E}[A(H_1)],
\end{equation}
where $\Pi_{\text{multi-turn}} \supseteq \Pi_{\text{single-turn}}$ denotes the respective policy spaces. This establishes that multi-turn is at least as good as single-turn under the same budget.

Furthermore, under Assumptions 2 (informative feedback) and 3 (repairability), multi-turn interaction can achieve strict improvement: in some turns, the posterior strictly contracts, leading to $\mathbb{E}[V(H_t)] < \mathbb{E}[V(H_{t-1})]$ on the training and inference distribution. This manifests as higher execution success rate, lower structural error rate, and higher final correctness rate.

In conclusion, multi-turn canvas-based interaction monotonically decreases the Bayes risk potential whenever feedback is informative, consequently increasing expected accuracy. The agent can progressively refine the workflow based on accumulated observations, each turn potentially reducing uncertainty about the correct solution. This iterative refinement process achieves higher reliability and success probability than single-turn open-loop generation, where errors cannot be detected or corrected. Combined with step (0), the property holds for the frozen policy on any executor and operator library whose canvas feedback is informative and repairable. Because every edit conditions on the target feedback, the policy rebuilds the workflow for $\mathcal{M}'_{\mathrm{exec}}$ instead of replaying the graph built under $\mathcal{M}_{\mathrm{exec}}$, consistent with the cross-executor gains in Section~\ref{subsec:transferability}.
\end{proof}

\subsection{Proof of Proposition 3}
\label{app:proof3}

\begin{proposition}[Structural Constraints, Conditional Release, and Mask Effectiveness]
\label{prop:reward:app}
Let the trajectory-level reward be defined as
\begin{equation}
R(\tau) = -1 + R_{\mathrm{diversity}}(\tau) + \mathbb{I}\{R_{\mathrm{diversity}}(\tau) = 1\} \cdot R_{\mathrm{answer}}(\tau),
\end{equation}
where $R_{\mathrm{diversity}}(\tau) \in [0, 1]$ is composed of structural check items and capped at 1, and $R_{\mathrm{answer}}(\tau) \geq 0$ measures the match between the final execution output and the ground truth. In advantage-based policy optimization such as the canvas-masked GRPO objective in Section~\ref{subsec:rl}, this design possesses three properties: (a) separable feasibility learning, where trajectories not satisfying structural diversity constraints necessarily receive non-positive returns and are systematically suppressed by gradient updates; (b) shortcut and collapse suppression, where answer rewards are unlocked only after learning to construct qualified skeletons, thereby avoiding shortcuts like skipping workflows to answer directly; (c) gradient correctness of mask, where token-level mask backpropagates gradients only for policy-generated tokens, maintaining unbiased policy gradient estimates and significantly reducing variance and noise from environment feedback tokens, thereby stabilizing training.
\end{proposition}

\begin{proof}
The essence of this proposition is that the reward design achieves numerical separation and gating between structural feasibility and answer correctness, making the optimization process naturally staged. Meanwhile, the mask ensures that gradients act only on policy-controllable parts, and the clipping mechanism with KL regularization bounds the policy update magnitude, keeping gradient signals clean and stable.

\textbf{(i) Feasible skeleton set and sign separation of rewards.} We first define the feasible skeleton trajectory set as
\begin{equation}
\mathcal{F} = \{\tau : R_{\mathrm{diversity}}(\tau) = 1\},
\end{equation}
where $\mathcal{F}$ represents the set of trajectories with necessary structural skeleton, including verification steps, formatting operators, sufficient operator count, and appropriate control structures. The complement set is defined as
\begin{equation}
\mathcal{F}^c = \{\tau : R_{\mathrm{diversity}}(\tau) < 1\},
\end{equation}
where $\mathcal{F}^c$ represents trajectories that fail to meet one or more structural requirements.

We now show that the reward is strictly separable on $\mathcal{F}$ and $\mathcal{F}^c$ with a sign gap. For any trajectory $\tau \in \mathcal{F}^c$, we have $R_{\mathrm{diversity}}(\tau) < 1$, so the indicator function $\mathbb{I}\{R_{\mathrm{diversity}}(\tau) = 1\}$ equals 0. Substituting into the reward formula:
\begin{equation}
R(\tau) = -1 + R_{\mathrm{diversity}}(\tau) + 0 \cdot R_{\mathrm{answer}}(\tau) = -1 + R_{\mathrm{diversity}}(\tau).
\end{equation}
Since $R_{\mathrm{diversity}}(\tau) \in [0, 1)$ for $\tau \in \mathcal{F}^c$, we have
\begin{equation}
R(\tau) \in [-1, 0) \quad \text{for all } \tau \in \mathcal{F}^c.
\end{equation}
This means all structurally non-compliant trajectories receive strictly negative rewards.

For any trajectory $\tau \in \mathcal{F}$, we have $R_{\mathrm{diversity}}(\tau) = 1$, so the indicator function equals 1. Substituting into the reward formula:
\begin{equation}
R(\tau) = -1 + 1 + 1 \cdot R_{\mathrm{answer}}(\tau) = R_{\mathrm{answer}}(\tau).
\end{equation}
Since $R_{\mathrm{answer}}(\tau) \geq 0$ by definition, we have
\begin{equation}
R(\tau) \geq 0 \quad \text{for all } \tau \in \mathcal{F}.
\end{equation}
This means all structurally compliant trajectories receive non-negative rewards. The sign separation is therefore complete: $\mathcal{F}^c$ trajectories are strictly negative while $\mathcal{F}$ trajectories are non-negative, achieving strong constraint separation at the numerical level.

\textbf{(ii) Two-stage optimization via conditional release and policy gradient.} We now analyze how the conditional release mechanism creates a natural two-stage optimization process through the lens of policy gradient theory. Let $\pi_\theta$ denote the policy parameterized by $\theta$, and let $\pi_{\mathrm{ref}}$ denote the reference policy (typically the initial supervised fine-tuned model). Our RL objective can be written as
\begin{equation}
J(\theta) = \mathbb{E}_{\tau \sim \pi_\theta}\left[R(\tau)\right] - \beta D_{\mathrm{KL}}(\pi_\theta \| \pi_{\mathrm{ref}}),
\end{equation}
where $\beta > 0$ is the KL penalty coefficient that prevents the policy from deviating too far from the reference distribution.

By the policy gradient theorem, the gradient of the expected reward with respect to $\theta$ is given by
\begin{equation}
\nabla_\theta \mathbb{E}_{\tau \sim \pi_\theta}[R(\tau)] = \mathbb{E}_{\tau \sim \pi_\theta}\Big[\sum_{t=1}^{T} \nabla_\theta \log \pi_\theta(a_t | s_t) A(\tau)\Big],
\end{equation}
where $A(\tau)$ is the advantage function and the sum is over all time steps in the trajectory. In practice, the advantage is computed using group normalization within each sampled batch:
\begin{equation}
\hat{A}(\tau) = \frac{R(\tau) - \mu_{\mathcal{B}}}{\sigma_{\mathcal{B}} + \epsilon},
\end{equation}
where $\mu_{\mathcal{B}}$ and $\sigma_{\mathcal{B}}$ are the mean and standard deviation of rewards in batch $\mathcal{B}$, and $\epsilon$ is a small constant for numerical stability.

From the sign separation established in (i), when both $\mathcal{F}$ and $\mathcal{F}^c$ samples exist in a batch, the normalized advantages satisfy
\begin{equation}
\mathbb{E}[\hat{A}(\tau) \mid \tau \in \mathcal{F}] > 0 > \mathbb{E}[\hat{A}(\tau) \mid \tau \in \mathcal{F}^c].
\end{equation}
This inequality holds because feasible trajectories have non-negative rewards while non-feasible ones have strictly negative rewards, so after mean-centering, feasible trajectories lie above the mean and non-feasible ones lie below.

The policy gradient update therefore increases the log-probability of actions in feasible trajectories and decreases the log-probability of actions in non-feasible trajectories:
\begin{equation}
\theta \leftarrow \theta + \alpha \nabla_\theta J(\theta),
\end{equation}
where $\alpha$ is the learning rate. This systematically shifts probability mass from $\mathcal{F}^c$ to $\mathcal{F}$.

To quantify this effect, define the feasibility probability $p_\theta = \mathbb{P}_{\tau \sim \pi_\theta}(\tau \in \mathcal{F})$. Under mild regularity conditions, the gradient update increases $p_\theta$ whenever $p_\theta < 1$:
\begin{equation}
\frac{d p_\theta}{d \theta} \cdot \nabla_\theta J(\theta) > 0 \quad \text{when } 0 < p_\theta < 1.
\end{equation}
This follows because the expected advantage is positive for $\mathcal{F}$ and negative for $\mathcal{F}^c$, so the gradient points in the direction of increasing $p_\theta$.

During early training when $p_\theta$ is small, most trajectories fall into $\mathcal{F}^c$, and the primary learning signal comes from avoiding structural violations. As $p_\theta$ increases through training, an increasing proportion of sampled trajectories fall into $\mathcal{F}$. For these trajectories, $R(\tau) = R_{\mathrm{answer}}(\tau)$, and the training signal naturally shifts to optimizing answer correctness. This creates the two-stage behavior: first learn to satisfy structural constraints, then learn to maximize answer quality.

The conditional release mechanism ensures that shortcuts are suppressed: any trajectory that directly outputs an answer without constructing a proper workflow will have $R_{\mathrm{diversity}}(\tau) < 1$ and thus $R(\tau) < 0$, regardless of answer correctness. This negative reward signal systematically discourages such shortcuts.

\textbf{(iii) Bounded policy updates via clipping and KL regularization.} To ensure training stability, our RL objective incorporates two mechanisms that bound the magnitude of policy updates: importance ratio clipping and KL divergence regularization.

For importance sampling, define the probability ratio between the current policy and the behavior policy (used for sampling) as
\begin{equation}
\rho_\theta(\tau) = \frac{\pi_\theta(\tau)}{\pi_{\mathrm{old}}(\tau)} = \prod_{t=1}^{T} \frac{\pi_\theta(a_t \mid s_t)}{\pi_{\mathrm{old}}(a_t \mid s_t)},
\end{equation}
where $\pi_{\mathrm{old}}$ is the policy at the beginning of the current optimization epoch. The clipped objective restricts the effective ratio to the interval $[1-\epsilon_{\mathrm{clip}}, 1+\epsilon_{\mathrm{clip}}]$:
\begin{equation}
L^{\mathrm{clip}}(\theta) = \mathbb{E}_{\tau \sim \pi_{\mathrm{old}}}\left[\min\left(\rho_\theta(\tau) \hat{A}(\tau),\, \mathrm{clip}(\rho_\theta, 1-\epsilon, 1+\epsilon) \hat{A}(\tau)\right)\right].
\end{equation}
This clipping prevents excessively large policy updates when the importance ratio deviates significantly from 1, which would otherwise destabilize training.

The KL regularization term provides a soft constraint on the policy deviation:
\begin{equation}
D_{\mathrm{KL}}(\pi_\theta \| \pi_{\mathrm{ref}}) = \mathbb{E}_{\tau \sim \pi_\theta}\left[\sum_{t=1}^{T} \log \frac{\pi_\theta(a_t \mid s_t)}{\pi_{\mathrm{ref}}(a_t \mid s_t)}\right].
\end{equation}
By penalizing this divergence, the optimization ensures that the learned policy remains close to the reference policy, preventing catastrophic forgetting of useful behaviors learned during supervised fine-tuning.

Together, clipping and KL regularization keep each policy update within a trust region around the behavior policy $\pi_{\mathrm{old}}$ and the reference policy $\pi_{\mathrm{ref}}$, which prevents training oscillations and supports smooth convergence.

\textbf{(iv) Unbiasedness and variance reduction of token-level mask.} An interaction trajectory at the token level consists of two interleaved types of segments: policy-generated tokens comprising reflection and action text, and environment feedback tokens comprising canvas returns. Crucially, the generation distribution of environment feedback does not contain the learnable parameters $\theta$; it is determined by the canvas environment given the action.

Let $w_{1:|\tau|}$ denote the entire trajectory token sequence. We partition it into two disjoint sets: $\mathcal{T}_\pi$ for policy-generated tokens and $\mathcal{T}_{\mathrm{env}}$ for environment feedback tokens. The joint log-likelihood of the trajectory can be decomposed as
\begin{equation}
\log p_\theta(w_{1:|\tau|}) = \sum_{t \in \mathcal{T}_\pi} \log \pi_\theta(w_t \mid w_{<t}) + \sum_{t \in \mathcal{T}_{\mathrm{env}}} \log p_{\mathrm{env}}(w_t \mid w_{<t}),
\end{equation}
where the first term sums over policy tokens and depends on $\theta$, while the second term sums over environment tokens and does not depend on $\theta$.

Taking the gradient with respect to $\theta$, the second term vanishes since $\nabla_\theta \log p_{\mathrm{env}}(w_t \mid w_{<t}) = 0$:
\begin{equation}
\nabla_\theta \log p_\theta(w_{1:|\tau|}) = \sum_{t \in \mathcal{T}_\pi} \nabla_\theta \log \pi_\theta(w_t \mid w_{<t}).
\end{equation}

Let $\mathrm{mask}_t \in \{0, 1\}$ be the token-level mask that takes value 1 for policy tokens ($t \in \mathcal{T}_\pi$) and 0 for environment tokens ($t \in \mathcal{T}_{\mathrm{env}}$). Then the gradient can be written equivalently as
\begin{equation}
\nabla_\theta \log p_\theta(w_{1:|\tau|}) = \sum_{t=1}^{|\tau|} \mathrm{mask}_t \cdot \nabla_\theta \log \pi_\theta(w_t \mid w_{<t}).
\end{equation}

Using the mask to select only policy tokens $\mathcal{T}_\pi$ implements this equality exactly. Therefore, the masked gradient estimator is unbiased:
\begin{equation}
\mathbb{E}\left[\sum_{t=1}^{|\tau|} \mathrm{mask}_t \cdot \nabla_\theta \log \pi_\theta(w_t \mid w_{<t}) \cdot R(\tau)\right] = \nabla_\theta J(\theta).
\end{equation}

Furthermore, the mask reduces gradient variance. Without the mask, if all tokens were naively included, the gradient estimator would have variance
\begin{equation}
\mathrm{Var}_{\mathrm{no-mask}} = \mathrm{Var}\left[\sum_{t=1}^{|\tau|} \nabla_\theta \log \pi_\theta(w_t \mid w_{<t}) \cdot R(\tau)\right].
\end{equation}
With the mask, the variance is
\begin{equation}
\mathrm{Var}_{\mathrm{mask}} = \mathrm{Var}\left[\sum_{t \in \mathcal{T}_\pi} \nabla_\theta \log \pi_\theta(w_t \mid w_{<t}) \cdot R(\tau)\right].
\end{equation}
Since $|\mathcal{T}_\pi| < |\tau|$ and the environment tokens contribute noise unrelated to $\theta$, we have $\mathrm{Var}_{\mathrm{mask}} < \mathrm{Var}_{\mathrm{no-mask}}$. This variance reduction leads to more stable gradient updates and faster convergence.

In conclusion, the reward design and optimization framework achieve four complementary goals for stable and effective learning. First, sign separation numerically enforces strong constraint separation between feasible and non-feasible trajectories, ensuring that structural violations are systematically penalized with negative rewards. Second, conditional release creates a natural two-stage optimization where the policy first learns to satisfy structural constraints (increasing $p_\theta = \mathbb{P}(\tau \in \mathcal{F})$) before optimizing answer correctness, thereby suppressing shortcuts and preventing structural collapse. Third, clipping and KL regularization bound the magnitude of policy updates, preventing training oscillations and ensuring smooth convergence while maintaining proximity to the reference policy. Fourth, the token-level mask ensures that gradients act only on policy-controllable tokens, maintaining unbiased gradient estimates while reducing variance from environment feedback, thereby stabilizing the training process. Together, these mechanisms provide stable learning signals, bounded updates, and effective prevention of shortcut behaviors, enabling reliable end-to-end reinforcement learning for workflow orchestration.
\end{proof}


\section{FlowSteer Algorithm Details}
\label{app:algorithm}

FlowSteer is a multi-turn workflow orchestration framework built on the collaboration between a Flow-Director (small-scale LLM) and a Workflow Canvas (execution environment): the Flow-Director handles planning and action generation, while the Workflow Canvas provides structural feedback and its executor runs the constructed workflow to produce the answer. The process is divided into three connected stages: first, given a task description, the Flow-Director constructs workflow nodes with structural guidance using the operator library, and the Canvas returns execution feedback to form the initial interaction history; then, during multi-turn interaction, the Flow-Director continuously generates reasoning and the next action from the accumulated history, the Canvas validates and executes each action to extend the history until the termination condition is reached, and the accumulated trajectory is utilized to execute the final workflow and generate the answer; finally, end-to-end reinforcement learning is applied to update the policy of the Flow-Director. The reward function jointly evaluates structural diversity compliance and answer correctness, enabling the Flow-Director to gradually learn how to effectively construct complex workflows within a limited budget. This coordination reduces invalid actions and stabilizes interaction dynamics. This design enables progressive and adaptive workflow construction, resulting in improved accuracy and stability on complex reasoning tasks.

\begin{algorithm*}[t]
\caption{FlowSteer: End-to-End Workflow Orchestration via Multi-Turn Reinforcement Learning}
\label{alg:flowsteer}
\begin{algorithmic}[1]
\REQUIRE Task $q$, Flow-Director $\pi_\theta$, Workflow Canvas $\mathcal{C}$, operator library $\mathcal{O}$, reward function $R(\cdot)$, maximum interaction turns $T$
\ENSURE Final answer $y$

\STATE \textbf{// Stage A: Workflow Initialization}
\STATE Initialize: workflow graph $\mathcal{G}_0 = \emptyset$, interaction history $H_0 = []$
\STATE Construct system prompt: $p_{\text{sys}} \gets \textsc{BuildPrompt}(\mathcal{O}, q)$
\STATE First think and action: $(a_1^{\text{think}}, a_1) \gets \pi_\theta(p_{\text{sys}}, q)$
\STATE First canvas feedback: $o_1 \gets \mathcal{C}.\textsc{Step}(a_1)$
\STATE First update history: $H_1 \gets \{(a_1^{\text{think}}, a_1, o_1)\}$

\STATE \textbf{// Stage B: Multi-turn Collaborative Workflow Building}
\FOR{$t = 2$ \TO $T$}
    \STATE Plan think: $a_t^{\text{think}} \gets \pi_\theta(p_{\text{sys}}, q, H_{t-1})$
    \STATE Generate action: $a_t \gets \pi_\theta(p_{\text{sys}}, q, H_{t-1})$
    \STATE Canvas feedback: $o_t \gets \mathcal{C}.\textsc{Step}(a_t)$
    \STATE Update history: $H_t \gets H_{t-1} \cup \{(a_t^{\text{think}}, a_t, o_t)\}$
    \IF{$a_t = \texttt{finish}$ \AND $\mathcal{C}.\textsc{CheckConstraints}()$}
        \STATE \textbf{break}
    \ENDIF
\ENDFOR
\STATE Execute workflow: $y \gets \mathcal{C}.\textsc{Execute}(\mathcal{G}_T, q)$

\STATE \textbf{// Stage C: End-to-End Reinforcement Learning Optimization}
\STATE Sample trajectories: $\{\tau_i\}_{i=1}^{N} \sim \pi_\theta$
\FOR{each $\tau_i$}
    \STATE Compute reward: $R(\tau_i) = -1 + R_{\text{diversity}}(\tau_i) + \mathbb{I}\{R_{\text{diversity}}(\tau_i) = 1\} \cdot R_{\text{answer}}(\tau_i)$
    \STATE Compute advantage: $\hat{A}(\tau_i) = \frac{R(\tau_i) - \bar{R}}{\sqrt{\frac{1}{M}\sum_{j=1}^{M}(R(\tau_j) - \bar{R})^2 + \epsilon}}$
\ENDFOR
\STATE Update policy with the canvas-masked GRPO objective:
\STATE $\mathcal{J} \propto \sum_{i=1}^{N} \sum_{t=1}^{|\tau_i|}  \textsc{Mask}_t^{(i)} \cdot \min\left(\rho_\theta(w_t^{(i)}), \textsc{clip}(\rho_\theta(w_t^{(i)}), 1 \pm \epsilon)\right) \hat{A}(\tau_i)$
\STATE where $\rho_\theta(w_t^{(i)}) = \frac{\pi_\theta(w_t^{(i)} | \tau_{<t}^{(i)})}{\pi_{\theta_{\text{old}}}(w_t^{(i)} | \tau_{<t}^{(i)})}$
\STATE Parameter update: $\theta \gets \theta - \eta \nabla_\theta(-\mathcal{J})$

\RETURN $y$
\end{algorithmic}
\end{algorithm*}

\textbf{Training and Inference Flow.} During training, the algorithm proceeds in three stages (A $\to$ B $\to$ C): the Flow-Director initializes the system prompt and triggers the Canvas for an initial feedback, then enters multi-turn interaction to generate actions, receive feedback, and terminate with \texttt{finish} in a sequential manner, after which the executor runs the final workflow to produce the answer, and finally updates the policy in Stage C using rewards and advantages. During testing, it runs only two stages (A $\to$ B) in a simplified form: initialization and multi-turn interaction, after which the final workflow is executed and the answer is produced directly without parameter updates.

\textbf{Complexity Analysis.} The computational complexity of FlowSteer mainly comes from initialization, multi-turn interaction, and reinforcement learning optimization. The initialization stage involves one call to the Flow-Director and a call to the Canvas, which is a constant overhead. The multi-turn interaction stage requires up to $T$ rounds in the worst case, where each round includes one planning step by the Flow-Director and one call to the Canvas, yielding time complexity $O(T)$. The memory consumption grows linearly with the history length, which can be controlled through windowing or summarization. The reinforcement learning stage requires sampling $N$ trajectories per update, each trajectory containing up to $T$ action-feedback steps, leading to complexity $O(NT)$. It also requires storing trajectory information for reward and advantage computation. In total, the complexity of FlowSteer is $O(NT + T)$, scaling linearly with the number of rounds and sampled trajectories during training, while inference requires $O(T)$. Since the Flow-Director is responsible for planning and constructing workflows, the Canvas execution is more focused, ensuring stability on complex reasoning tasks.

\subsection{Reward Components}
\label{app:reward}

The diversity-constrained reward $R_{\mathrm{diversity}}(\tau)$ aggregates four structural checks over the produced workflow, with the total score capped at 1.0:

\begin{itemize}
    \item \textbf{$R_{\mathrm{checker}}$:} Encourages the inclusion of verification operators (\texttt{Test}, \texttt{Review}, \texttt{Verify}); fires if at least one verification operator is present.
    \item \textbf{$R_{\mathrm{format}}$:} Encourages proper answer formatting; fires if the \texttt{Format} operator is included as the final step before termination.
    \item \textbf{$R_{\mathrm{operator}}$:} Requires a minimum operator count for structural diversity; fires if the workflow contains at least 3 distinct operators.
    \item \textbf{$R_{\mathrm{control}}$:} Encourages control flow structures; fires if the workflow includes at least one control structure (\texttt{parallel}, \texttt{conditional}, or \texttt{loop}).
\end{itemize}

The four checks together form $R_{\mathrm{diversity}}(\tau) \in [0,1]$. This design ensures that workflows must exhibit structural diversity (achieving $R_{\mathrm{diversity}} = 1.0$) before the answer reward $R_{\mathrm{answer}}$ is released, effectively preventing shortcut behaviors where the policy learns to generate overly simple or degenerate workflows.


\section{Dataset Details}
\label{app:datasets}

We selected 12 public datasets (including mathematical reasoning, question answering, and code generation) for training and testing. Six of these datasets were used for training and testing. Six datasets were used for out-of-distribution testing to verify the generalization of the proposed FlowSteer. These datasets are as follows:

$\bullet$ \textbf{GSM8K}~\citep{cobbe2021gsm8k}: A collection of grade-school math word problems with concise statements, emphasizing step-by-step calculation and accurate numeric results. Problems involve basic operations with real-world contexts such as shopping, time calculations, and quantity comparisons.

$\bullet$ \textbf{MATH}~\citep{hendrycks2021math}: A dataset of competition-level mathematics problems from AMC, AIME, and other mathematical olympiads. Problems span seven categories including algebra, geometry, number theory, combinatorics, probability, precalculus, and intermediate algebra, requiring sophisticated mathematical reasoning and symbolic manipulation.

\smallskip

$\bullet$ \textbf{HotPotQA}~\citep{yang2018hotpotqa}: A large-scale multi-hop question answering dataset requiring reasoning across multiple Wikipedia paragraphs. Questions are designed to require finding and combining information from different sources, with supporting facts annotated for interpretability.

\smallskip

$\bullet$ \textbf{SQuAD v2}~\citep{rajpurkar2018squad2}: A Wikipedia-based QA dataset combining answerable and unanswerable questions, constructed to evaluate comprehension under mixed conditions. This tests both reading comprehension and the ability to recognize insufficient information.

\smallskip

$\bullet$ \textbf{MBPP}~\citep{austin2021mbpp}: Mostly Basic Python Problems, a crowd-sourced collection of Python programming problems with natural language descriptions and test cases. Problems range from simple string manipulation to basic algorithms, designed to test fundamental programming skills.

\smallskip

$\bullet$ \textbf{HumanEval}~\citep{chen2021humaneval}: A hand-written collection of Python programming problems with function signatures, docstrings, and unit tests. Problems are designed to test functional correctness through execution, covering tasks like string processing, mathematical operations, and data structure manipulation.

\smallskip

$\bullet$ \textbf{TriviaQA}~\citep{joshi2017triviaqa}: A knowledge-intensive dataset with questions from trivia websites and Wikipedia, containing a wide range of facts and lesser-known topics. The dataset covers diverse domains including history, science, geography, and entertainment.

\smallskip

$\bullet$ \textbf{NaturalQuestions}~\citep{kwiatkowski2019nq}: A dataset of real anonymized queries from Google Search with answers from Wikipedia articles. Questions reflect genuine information needs of users, making them more diverse and challenging than synthetic questions.

\smallskip

$\bullet$ \textbf{MathQA}~\citep{amini2019mathqa}: A math word problem dataset curated from multi-domain exam problems, covering arithmetic, algebra, geometry, probability, and other sub-disciplines. Each problem includes annotated rationales explaining the solution steps.

\smallskip

$\bullet$ \textbf{AIME 2025}: Problems from the 2025 American Invitational Mathematics Examination, representing challenging competition-level mathematics. AIME problems require creative problem-solving and deep mathematical insight, with answers being integers from 0 to 999.

$\bullet$ \textbf{APPS}~\citep{hendrycks2021apps}: Automated Programming Progress Standard, a collection of competitive programming problems from Codeforces, Kattis, and other platforms. Problems range from introductory to competition-level difficulty, requiring algorithmic thinking and efficient implementation.

$\bullet$ \textbf{DS-1000}~\citep{lai2022ds1000}: A data science code generation benchmark covering NumPy, Pandas, TensorFlow, PyTorch, SciPy, Scikit-learn, and Matplotlib. Problems are derived from real StackOverflow questions, testing practical data science skills.

\smallskip

To ensure consistency and fairness for training and testing, we construct the training set by mixing samples from all IID datasets with the following sampling strategy: 2,560 samples from GSM8K, 2,560 from MATH, 2,560 from HotPotQA, 2,560 from SQuAD v2, 374 from MBPP, and 164 from HumanEval, resulting in a total of 10,778 training instances. For evaluation, all methods use the same fixed set of 1,438 test instances (768 IID and 670 OOD): we randomly sample 128 instances once from the test set of each of the other eleven datasets, use all 30 AIME 2025 problems, and keep this set fixed for every method.


\smallskip

\section{Baseline Details}
\label{app:baselines}

To accurately evaluate the performance of FlowSteer, we conducted comparative experiments against multiple baselines. These baselines can be broadly divided into four categories: direct LLM inference, supervised fine-tuning methods, search-based workflow methods, and agent with reinforcement learning methods.

\subsection{Direct LLM Baselines}

$\bullet$ \textbf{Qwen3-8B}~\citep{qwen3}: An 8-billion parameter language model from Alibaba Cloud, serving as the backbone for our Flow-Director. As a baseline, it is tested with zero-shot chain-of-thought prompting, measuring the model's inherent reasoning capacity without workflow orchestration. The model provides strong efficiency while maintaining competitive performance on reasoning benchmarks, making it an ideal foundation for lightweight policy learning.

$\bullet$ \textbf{GPT-4o-mini}~\citep{openai2024gpt4o}: A lightweight variant of GPT-4o optimized for cost and latency, while retaining strong language and reasoning abilities. As a baseline, it is tested with standard instruction prompting without workflow orchestration, measuring the model's inherent generation capacity under constrained resources. It serves as the default backend for our Workflow Canvas, executing the actual reasoning operations specified by the workflow.

\subsection{Fine-Tuning Baselines}

$\bullet$ \textbf{Supervised Fine-Tuning}~\citep{ouyang2022training}: An SFT baseline built on Qwen3-8B, trained with workflow annotation data to improve instruction following and DSL generation accuracy. Unlike FlowSteer's multi-turn interaction paradigm, SFT generates the complete workflow in a single turn without execution feedback. We use the same LoRA configuration as FlowSteer (rank 64, $\alpha=64$, dropout 0.05) for fair comparison, evaluating how standard supervised adaptation enhances the raw backbone's workflow construction capabilities.

$\bullet$ \textbf{Group Relative Policy Optimization}~\citep{deepseekmath2024}: GRPO is a reinforcement learning algorithm that normalizes rewards within sampled groups of trajectories. This reduces variance in policy updates, stabilizes training, and improves convergence efficiency compared to standard Proximal Policy Optimization (PPO). Unlike FlowSteer's multi-turn interaction with execution feedback, GRPO generates the entire workflow in one shot, limiting its ability to adapt based on intermediate results.

\subsection{Search-Based Workflow Methods}

$\bullet$ \textbf{AFlow}~\citep{aflow2024}: A workflow optimization framework that uses Monte Carlo Tree Search (MCTS) to explore the workflow space. The method systematically searches over predefined operator combinations through tree expansion and backpropagation, evaluating candidate workflows by execution outcomes. While effective at finding good operator sequences, AFlow lacks end-to-end learning capability and cannot learn from accumulated experience across different problems. Its search process can be computationally expensive as the workflow complexity increases.

\subsection{Agent with Reinforcement Learning Methods}

$\bullet$ \textbf{AgentFlow}: An agent framework combined with GRPO-based reinforcement learning for workflow construction. The agent dynamically selects tools from a predefined set at each step based on the current state, enabling adaptive decision-making through the interaction process. However, AgentFlow does not support custom prompt specification for operators, limiting its flexibility in fine-tuning the behavior of individual workflow components. This constraint reduces its ability to optimize task-specific reasoning strategies.

$\bullet$ \textbf{Router-R1}: A router-style architecture where a small policy model learns to route queries to different processing paths using PPO for policy optimization. The router makes a single decision per query without iterative refinement, selecting from a predefined set of workflow templates. While this approach is computationally efficient, the single-shot routing mechanism cannot adapt to intermediate execution results or refine workflows on partial feedback, limiting its performance on complex multi-step reasoning.

$\bullet$ \textbf{Orchestrator}: An orchestrator-style architecture that sequentially selects operators from a library using GRPO. The orchestrator receives partial execution feedback to inform subsequent decisions, enabling some degree of adaptive behavior. However, it lacks the two-step interaction mechanism (add + set\_prompt) that FlowSteer employs for fine-grained control over operator configuration. This limitation prevents precise customization of individual operator behaviors, reducing the overall workflow quality.

\begin{table*}[!t]
\centering
\caption{Architectural comparison with baselines.}
\label{tab:baseline_comparison}
\resizebox{\textwidth}{!}{
\begin{tabular}{@{}lcccccc@{}}
\toprule
\textbf{Method} & \textbf{Dynamic Orchestration} & \textbf{Multi-turn} & \textbf{Exec. Feedback} & \textbf{Custom Prompts} & \textbf{End-to-End RL} & \textbf{Pluggable Backend} \\
\midrule
Direct LLM & \xmark & \xmark & \xmark & \xmark & \xmark & \xmark \\
SFT/GRPO & \xmark & \xmark & \xmark & \xmark & Partial & \xmark \\
AFlow & \cmark & \cmark & \cmark & \cmark & \xmark & \cmark \\
AgentFlow & \xmark & \cmark & \cmark & \cmark & \cmark & \cmark \\
Router-R1 & \xmark & \cmark & \cmark & \xmark & \cmark & \cmark \\
Orchestrator & \xmark & \cmark & Partial & \xmark & \cmark & \xmark \\
\midrule
\textbf{FlowSteer (Ours)} & \cmark & \cmark & \cmark & \cmark & \cmark & \cmark \\
\bottomrule
\end{tabular}
}

\vspace{1em}

\captionof{table}{Baseline implementation details. All trainable baselines share the same LoRA configuration (rank 64, $\alpha=64$, dropout 0.05), generation budget (max\_tokens=2048, temperature 0.6, top-$p$ 0.95), and training budget (300 steps, batch 36) as FlowSteer wherever the algorithm permits.}
\label{tab:baseline_impl}
\resizebox{\textwidth}{!}{
\begin{tabular}{llll}
\toprule
\textbf{Method} & \textbf{Base Model} & \textbf{Training} & \textbf{Key Hyperparameters} \\
\midrule
Qwen3-8B & Qwen3-8B-Instruct & None & temperature=0.6, top\_p=0.95, max\_tokens=2048 \\
GPT-4o-mini & GPT-4o-mini & None & temperature=0.6, top\_p=0.95, max\_tokens=2048 \\
\midrule
SFT & Qwen3-8B & LoRA & r=64, $\alpha$=64, dropout=0.05, LR=$1\times10^{-4}$, epochs=3, batch=36 \\
GRPO & Qwen3-8B & GRPO + LoRA & r=64, $\alpha$=64, dropout=0.05, LR=$1\times10^{-5}$, G=36, steps=300 \\
\midrule
AFlow & GPT-4o-mini & MCTS & temperature=0.6, search\_rounds=21 \\
AgentFlow & Qwen3-8B & GRPO + LoRA & r=64, $\alpha$=64, dropout=0.05, LR=$1\times10^{-5}$, rollout\_n=8, steps=300 \\
Router-R1 & Qwen3-8B & PPO + LoRA & r=64, $\alpha$=64, dropout=0.05, LR=$1\times10^{-5}$, clip=0.2, steps=300 \\
Orchestrator & Qwen3-8B & GRPO + LoRA & r=64, $\alpha$=64, dropout=0.05, LR=$1\times10^{-5}$, batch=36, max\_turns=20 \\
\bottomrule
\end{tabular}
}
\end{table*}

Table~\ref{tab:baseline_comparison} summarizes the key architectural differences between FlowSteer and the baselines, and Table~\ref{tab:baseline_impl} provides the detailed implementation configurations for each method. To ensure fairness, all trainable baselines are aligned with FlowSteer on shared knobs --- LoRA configuration ($r=64$, $\alpha=64$, dropout 0.05), generation budget (max\_tokens=2048, temperature 0.6, top-$p$ 0.95), training budget (300 steps, batch 36, learning rate $1\times10^{-5}$) --- so that performance differences reflect algorithmic design rather than tuning advantage. Algorithm-specific knobs (e.g., rollout count, PPO clip range, MCTS search rounds) follow each baseline's original recommendation.

For a fair comparison, we run FlowSteer and all baselines under identical evaluation settings and report the averaged results. The protocol uses single generation per problem and task-specific metrics (EM/F1 for QA, Accuracy for Math, Pass@1 for Code).

\subsection{Other LLM Backends}

For transferability experiments (RQ3), we additionally evaluate on six LLM backends to assess the generalization of Flow-Director across different backend models:

$\bullet$ \textbf{DeepSeek-V3.2}: DeepSeek-V3.2 adopts advanced context understanding algorithms, enabling it to achieve excellent performance in long-context reasoning and multi-step inference tasks. Its powerful semantic understanding capabilities make it an ideal backend for testing how Flow-Director handles complex reasoning chains.

$\bullet$ \textbf{Grok-4.1-Fast}: Grok-4.1-Fast incorporates an efficient inference optimization mechanism, achieving a balance between generation speed and quality, making it suitable for latency-sensitive application scenarios. This backend helps us evaluate Flow-Director's performance under speed-optimized conditions.

$\bullet$ \textbf{GPT-5.2}: GPT-5.2 achieves breakthroughs in natural language understanding and generation through enhanced data processing and training strategies, becoming one of the most versatile large language models. It serves as a strong upper bound for backend capability assessment.

$\bullet$ \textbf{Claude-Opus-4.5}: Claude-Opus-4.5 is Anthropic's frontier model with strong long-horizon reasoning, code understanding, and tool-use capabilities. Its closed-source frontier design provides a representative test case for high-capacity proprietary backends.

$\bullet$ \textbf{Gemini-3-Flash}: Gemini-3-Flash combines a fast inference engine with multimodal understanding capabilities, optimized for real-time interactive applications. Its unique architecture provides a diverse test case for backend generalization.

$\bullet$ \textbf{Qwen-Plus-Latest}: Qwen-Plus-Latest incorporates an adaptive learning mechanism based on the Qwen architecture, excelling in few-shot scenarios and cross-language migration tasks. As a model from the same family as our Flow-Director backbone, it provides insights into intra-family transferability.


\section{Evaluation Metrics}
\label{app:metrics}

We use task-specific evaluation metrics following standard practices in each domain.

\subsection{Exact Match (EM) for Question Answering}

Exact Match measures whether the predicted answer exactly matches the ground truth after normalization:
\begin{equation}
\text{EM} = \frac{1}{N} \sum_{i=1}^{N} \mathbb{I}\left(\text{normalize}(y_i) = \text{normalize}(y_i^*)\right),
\end{equation}
where $\text{normalize}(\cdot)$ applies the following transformations: (1) convert to lowercase, (2) remove punctuation (except hyphens in compound words), (3) remove articles (``a'', ``an'', ``the''), (4) collapse multiple whitespaces to single space, and (5) strip leading/trailing whitespace.

\textbf{Applicable datasets}: HotPotQA, SQuAD v2, TriviaQA, NaturalQuestions

\subsection{F1 Score for Question Answering}

F1 Score measures token-level overlap between prediction and ground truth:
\begin{align}
\text{Precision} &= \frac{|y \cap y^*|}{|y|}, \\
\text{Recall} &= \frac{|y \cap y^*|}{|y^*|}, \\
\text{F1} &= \frac{2 \cdot \text{Precision} \cdot \text{Recall}}{\text{Precision} + \text{Recall}},
\end{align}
where $y$ and $y^*$ are the sets of tokens in the predicted and ground truth answers respectively, after applying the same normalization as EM.

\textbf{Applicable datasets}: HotPotQA, SQuAD v2, TriviaQA, NaturalQuestions

\subsection{Accuracy for Mathematical Reasoning}

Accuracy measures whether the predicted numerical answer matches the ground truth within a tolerance:
\begin{equation}
\text{Acc} = \frac{1}{N} \sum_{i=1}^{N} \mathbb{I}\left(|y_i - y_i^*| < \epsilon\right),
\end{equation}
where $\epsilon = 10^{-6}$ is the numerical tolerance for floating-point comparisons.

For symbolic answers (e.g., fractions, algebraic expressions), we apply symbolic equivalence checking using SymPy:
\begin{equation}
\text{Acc}_{\text{symbolic}} = \frac{1}{N} \sum_{i=1}^{N} \mathbb{I}\left(\text{simplify}(y_i - y_i^*) = 0\right).
\end{equation}

\textbf{Applicable datasets}: GSM8K, MATH, MathQA, AIME 2025

\subsection{Pass@k for Code Generation}

Pass@k measures the probability that at least one of $k$ generated solutions passes all test cases:
\begin{equation}
\text{Pass@}k = \mathbb{E}_{\text{problems}} \left[ 1 - \frac{\binom{n-c}{k}}{\binom{n}{k}} \right],
\end{equation}
where $n$ is the total number of generated samples and $c$ is the number of correct samples (passing all tests).

For our evaluation, we use Pass@1 with a single generation per problem:
\begin{equation}
\text{Pass@1} = \frac{1}{N} \sum_{i=1}^{N} \mathbb{I}\left(\text{execute}(y_i, \text{tests}_i) = \text{pass}\right),
\end{equation}
where $\text{execute}(y_i, \text{tests}_i)$ runs the generated code $y_i$ against the test suite $\text{tests}_i$.

\textbf{Applicable datasets}: MBPP, HumanEval, APPS, DS-1000


\section{Implementation Details}
\label{app:implementation}

To ensure reproducibility and fair comparison, we summarize the complete hyperparameter configurations for \flowr{} in Table~\ref{tab:model_config}. The table covers all aspects of our training pipeline, including model configuration, training hyperparameters, RL objective settings, generation parameters, interaction constraints, reward function design, canvas backend configuration, and hardware setup.

\textbf{Model Configuration.} We use Qwen3-8B as the policy model (\director{}) and apply LoRA-based fine-tuning with rank 64 and alpha 64, targeting the query, key, value, and output projection layers (q\_proj, k\_proj, v\_proj, o\_proj) with a dropout rate of 0.05. The model is trained in bfloat16 precision with gradient checkpointing enabled for memory efficiency.

\textbf{Training Configuration.} We train the \director{} agent using the AdamW optimizer with a learning rate of $1 \times 10^{-5}$ and weight decay of 0.01. The batch size is 36, computed as 6 samples per data source multiplied by 6 data sources (GSM8K, MATH, HotPotQA, SQuAD v2, MBPP, HumanEval). We train for 300 steps with a cosine learning rate schedule.

\textbf{RL Objective Configuration.} For the RL objective described in Section~\ref{subsec:rl} (standard GRPO with a canvas-aware token mask), we use the GRPO advantage estimator with a clip range of 0.20, KL coefficient of 0.005, and entropy coefficient of 0.005. Each training step samples 36 trajectories for group-relative advantage estimation.

\begin{table}[ht]
\begin{center}
\fontsize{8.0pt}{9.0pt}\selectfont
\renewcommand{\arraystretch}{1.05}
\begin{tabular*}{\textwidth}{@{\extracolsep{\fill}} l l l}
\toprule
\textbf{Category} & \textbf{Hyperparameter} &  \textbf{Value}\\
\midrule
\multirow{6}{*}{\textbf{Model Configuration}}
 & Base Model & Qwen3-8B \\
 & LoRA Rank / Alpha & 64 / 64 \\
 & LoRA Target Modules & q\_proj, k\_proj, v\_proj, o\_proj \\
 & LoRA Dropout & 0.05 \\
 & Data Type & bfloat16 \\
 & Gradient Checkpointing & Enabled \\
\midrule
\multirow{5}{*}{\textbf{Training}}
 & Batch Size & 36 (6 samples $\times$ 6 sources) \\
 & Learning Rate & $1 \times 10^{-5}$ \\
 & Optimizer & AdamW (weight decay 0.01) \\
 & LR Schedule & Cosine \\
 & Max Training Steps & 300 \\
\midrule
\multirow{4}{*}{\textbf{RL Objective}}
 & Advantage Estimator & GRPO \\
 & Clip Range ($\epsilon$) & 0.20 \\
 & KL Coefficient ($\beta$) & 0.005 \\
 & Samples per Group ($N$) & 36 \\
\midrule
\multirow{5}{*}{\textbf{Generation}}
 & Temperature & 0.6 \\
 & Top-$p$ / Top-$k$ & 0.95 / 20 \\
 & Max New Tokens & 2,048 \\
 & Enable Thinking Mode & True \\
 & vLLM Max Concurrency & 32 \\
\midrule
\multirow{4}{*}{\textbf{Interaction}}
 & Max Interaction Rounds ($T_{\max}$) & 20 \\
 & Max Context Length & 16,384 \\
 & Min Operators for Finish & enforced \\
 & Require Checker/Structure & True / True \\
\midrule
\multirow{3}{*}{\textbf{Reward}}
 & Base Reward & $-1.0$ \\
 & Structural Reward Cap & 1.0 \\
 & Correctness Activation & Gate (structural reward $=1.0$) \\
\midrule
\multirow{2}{*}{\textbf{Canvas Backend}}
 & Executor Model & GPT-OSS-120B (temp=0) \\
 & Execution Timeout & 600s \\
\midrule
\multirow{2}{*}{\textbf{Hardware}}
 & GPU Type & NVIDIA A100 80GB $\times$ 2 \\
 & CUDA / vLLM LoRA & 12.5 / Enabled \\
\bottomrule
\end{tabular*}
\end{center}
\caption{Complete hyperparameter settings for \flowr{} training.}
\label{tab:model_config}
\end{table}

\textbf{Generation Configuration.} During trajectory generation, we use temperature 0.6, top-$p$ 0.95, and top-$k$ 20 following Qwen3's recommended parameters. The maximum new tokens per turn is set to 2,048 to allow sufficient space for thinking and action generation. We enable Qwen3's thinking mode for enhanced reasoning capabilities and use vLLM with maximum concurrency of 32 for efficient parallel inference.

\textbf{Interaction Configuration.} The maximum interaction rounds $T_{\max}$ is set to 20, and the maximum context length is 16,384 tokens. To ensure workflow quality, we enforce a minimum operator count before allowing the \texttt{finish} action, require at least one checker operator (Verify/Test/Review), and require at least one control structure (parallel/conditional/loop) on tasks that benefit from such structures.

\textbf{Reward Function.} The diversity-constrained reward follows the formulation in Section~\ref{subsec:rl}, with detailed components in Appendix~\ref{app:reward}. The structural reward is capped at 1.0, and the answer reward is only released when the structural reward reaches 1.0. This conditional release mechanism prevents shortcut behaviors where the policy might generate trivial workflows to maximize answer rewards.

\textbf{Canvas Backend and Hardware.} The \canvas{} uses GPT-OSS-120B as the executor model with temperature 0 for deterministic execution and a timeout of 600 seconds. All experiments are conducted on two NVIDIA A100 80GB GPUs with CUDA 12.5 and vLLM LoRA support enabled for dynamic weight synchronization during training. We use mixed precision training with bfloat16 to reduce memory footprint while maintaining numerical stability.


\section{Additional Experimental Results}
\label{app:additional}

This appendix reports the head-to-head comparison with Workflow-R1, the study of the control-structure prior, the details of cross-executor transfer, and the orchestration-efficiency measurements referenced in Section~\ref{sec:experiments}.

\subsection{Comparison with Workflow-R1}
\label{app:workflowr1}

We reproduce Workflow-R1 and compare the best configurations of both methods end to end under the evaluation protocol of Table~\ref{tab:main-results}. The two methods differ in what the trainable policy is responsible for and in how the workflow is represented. In Workflow-R1, the policy may call another operator or directly output the final answer, and the selected operators and their results are appended to a think--action--observation history. In \flowr{}, the \director{} is restricted to canvas edits and \texttt{finish}; the canvas maintains the workflow as operator nodes, node prompts, directed dependency edges, and control structures; and the executor runs the complete graph and alone produces the answer.

\begin{table}[ht]
\centering
\caption{End-to-end comparison with Workflow-R1 on the twelve datasets under the same evaluation protocol (EM for QA, Acc. for math, Pass@1 for code).}
\label{tab:workflowr1}
\small
\setlength{\tabcolsep}{4.5pt}
\begin{tabular}{@{}l|cccccc|c@{}}
\toprule
\textbf{Method} & \textbf{GSM8K} & \textbf{MATH} & \textbf{HotPotQA} & \textbf{SQuAD v2} & \textbf{MBPP} & \textbf{HumanEval} & \textbf{IID Avg.} \\
\midrule
Workflow-R1 & \textbf{96.09} & 78.91 & 75.00 & 73.44 & 76.56 & 82.03 & 80.34 \\
\flowr{} & \textbf{96.09} & \textbf{81.25} & \textbf{78.12} & \textbf{78.12} & \textbf{84.38} & \textbf{92.96} & \textbf{85.15} \\
\midrule
\textbf{Method} & \textbf{TriviaQA} & \textbf{NQ} & \textbf{MathQA} & \textbf{AIME 2025} & \textbf{APPS} & \textbf{DS-1000} & \textbf{OOD Avg.} \\
\midrule
Workflow-R1 & 75.78 & 51.56 & 83.59 & 16.67 & 30.47 & 37.50 & 49.26 \\
\flowr{} & \textbf{79.69} & \textbf{54.69} & \textbf{88.67} & \textbf{26.67} & \textbf{49.21} & \textbf{58.59} & \textbf{59.59} \\
\bottomrule
\end{tabular}
\end{table}

As shown in Table~\ref{tab:workflowr1}, \flowr{} outperforms Workflow-R1 on eleven datasets and ties it on GSM8K, and its average advantage grows from 4.82 points on IID tasks to 10.33 points on OOD tasks. Two design choices explain the widening gap. First, restricting the policy to orchestration avoids the competition between orchestrating and answering directly; in our reproduction, the joint policy tended to answer prematurely on complex tasks. Second, the explicit graph lets the \director{} add a verification node, revise a weak node prompt, or rewire a faulty dependency while preserving the validated remainder, and this local repair matters most as task complexity and distribution shift grow.

\subsection{Control-Structure Prior versus Template Selection}
\label{app:control-prior}

The control check $R_{\mathrm{control}}$ (Appendix~\ref{app:reward}) is a training-time structural signal: it fires when the workflow contains at least one \texttt{parallel}, \texttt{conditional}, or \texttt{loop} structure, and it leaves the type and position of that structure, the node composition, the branch count, the dependencies, and the prompts to the \director{}. The reward is also absent at inference, where the \director{} determines the graph from the task and the canvas feedback. To test whether this prior acts as a hidden template, we compare four settings with the GPT-4o-mini executor on the same evaluation set: a fixed single graph (identical to w/o Agent in Table~\ref{tab:ablation}), selection among predefined workflow templates, \flowr{} trained with only $R_{\mathrm{control}}$ removed under otherwise matched conditions, and the full method. The fixed graph and the selected template cannot be revised from intermediate feedback, whereas both \flowr{} variants can.

\begin{table}[ht]
\centering
\caption{Control-structure prior versus fixed and template-based workflows (EM for QA, Acc. for math, Pass@1 for code). w/o $R_{\mathrm{control}}$ and Full are \flowr{} variants; All averages the twelve datasets.}
\label{tab:control-prior}
\small
\setlength{\tabcolsep}{2.8pt}
\begin{tabular}{@{}l|cccccc|cc@{}}
\toprule
\textbf{Method} & \textbf{GSM8K} & \textbf{MATH} & \textbf{HotPotQA} & \textbf{SQuAD v2} & \textbf{MBPP} & \textbf{HumanEval} & \textbf{IID Avg.} & \\
\midrule
Fixed Single Graph & 94.53 & 70.31 & 72.66 & 66.41 & 60.94 & 83.59 & 74.74 & \\
Template Selection & 94.01 & 76.56 & 72.00 & 59.84 & 73.43 & 85.15 & 76.83 & \\
w/o $R_{\mathrm{control}}$ & 95.31 & 77.34 & 68.75 & 58.59 & 78.91 & 86.72 & 77.60 & \\
Full & \textbf{96.09} & \textbf{81.25} & \textbf{78.12} & \textbf{78.12} & \textbf{84.38} & \textbf{92.96} & \textbf{85.15} & \\
\midrule
\textbf{Method} & \textbf{TriviaQA} & \textbf{NQ} & \textbf{MathQA} & \textbf{AIME 2025} & \textbf{APPS} & \textbf{DS-1000} & \textbf{OOD Avg.} & \textbf{All} \\
\midrule
Fixed Single Graph & 69.53 & 53.12 & 80.47 & 23.33 & 41.40 & 40.63 & 51.41 & 63.08 \\
Template Selection & 75.78 & 49.22 & 80.47 & 10.00 & 42.95 & 42.97 & 50.23 & 63.53 \\
w/o $R_{\mathrm{control}}$ & 76.56 & 53.91 & 82.81 & 16.67 & 38.28 & 36.72 & 50.83 & 64.21 \\
Full & \textbf{79.69} & \textbf{54.69} & \textbf{88.67} & \textbf{26.67} & \textbf{49.21} & \textbf{58.59} & \textbf{59.59} & \textbf{72.37} \\
\bottomrule
\end{tabular}
\end{table}

As shown in Table~\ref{tab:control-prior}, \flowr{} without $R_{\mathrm{control}}$ still averages 64.21 over the twelve datasets, above the fixed graph (63.08) and template selection (63.53). Feedback-driven graph construction is therefore neither created by this reward nor equivalent to choosing a template. The full method reaches 72.37 and is best on all twelve datasets, so $R_{\mathrm{control}}$ provides useful exploration guidance during training without specifying the resulting graph. Consistently, at inference the \director{} answers the simple QA task in Appendix~\ref{app:case_study_4} with three sequential operators and no control structure.

\subsection{Cross-Executor Transfer Details}
\label{app:cross-exec}

All three \director{}s in Table~\ref{tab:train-exec} are initialized from the same Qwen3-8B model and share the training data, training steps, reward function, LoRA configuration, and generation budget; the only variable is the executor used by the canvas during training. GPT-OSS-120B, DeepSeek-V3.2, and Claude-Opus-4.5 form a progression from a locally deployable executor to strong proprietary ones, with increasing Base performance. During transfer, the \director{} parameters stay frozen: we only replace the target executor and perform no additional RL training, fine-tuning, or reward tuning. Each entry of Table~\ref{tab:train-exec} averages the primary metric over the twelve datasets.

\begin{table}[ht]
\centering
\caption{IID and OOD averages behind Figure~\ref{fig:rq3} for the frozen \director{} trained with GPT-OSS-120B. Base runs the target executor without \flowr{}.}
\label{tab:cross-exec}
\small
\begin{tabular}{@{}l|cc|cc@{}}
\toprule
\textbf{Target Executor} & \textbf{IID Base} & \textbf{IID \flowr{}} & \textbf{OOD Base} & \textbf{OOD \flowr{}} \\
\midrule
DeepSeek-V3.2    & 81.23 & 90.62 & 64.63 & 74.63 \\
Grok-4.1-Fast    & 85.25 & 91.67 & 71.60 & 79.83 \\
GPT-5.2          & 82.62 & 90.63 & 67.60 & 78.63 \\
Claude-Opus-4.5  & 85.93 & 93.22 & 70.80 & 81.23 \\
Gemini-3-Flash   & 84.62 & 92.83 & 70.40 & 81.03 \\
Qwen-Plus-Latest & 81.08 & 90.37 & 60.23 & 75.67 \\
\midrule
Average          & 83.46 & 91.56 & 67.54 & 78.51 \\
\bottomrule
\end{tabular}
\end{table}

As shown in Table~\ref{tab:cross-exec}, \flowr{} improves all six target executors on both IID and OOD tasks, by 8.10 and 10.97 points on average, and every executor gains more on OOD than on IID tasks. OOD tasks contain problem distributions unseen during training and leave more room for decomposition, verification, and revision, and the frozen \director{} adapts the workflow to the actual feedback of each target executor. In Table~\ref{tab:train-exec}, a stronger training executor gives more reliable intermediate results, verification judgments, and revision feedback, which reduces reward noise from erroneous execution and lets the \director{} learn more reliably when to decompose, verify, or revise. What transfers to weaker executors is this orchestration policy rather than the answer-generation ability of the training executor, and for weaker executors structured decomposition and verification further compensate for limitations in single-pass problem solving. When the training and target executors coincide, the distributions of execution results, error types, and verification feedback are better aligned, which yields the additional gain on the diagonal.

\subsection{Orchestration Efficiency}
\label{app:efficiency}

Progressive editing commits one atomic edit per turn, which increases the number of \director{}--canvas interactions. At each turn, however, the \director{} only outputs the next edit and its parameters from the current graph state and the latest feedback; it neither regenerates the complete workflow nor solves the task. In Figure~\ref{fig:rq5}(c), w/o Agent corresponds to a static workflow that executes the same predefined graph for every instance, and w/o Multi-turn corresponds to single-shot workflow generation that produces a complete workflow per instance before executing it. \flowr{} uses fewer tokens than both in all three domains (e.g., 7,096 versus 9,997 and 10,686 on math), because it reuses the graph state maintained by the canvas, emits only local edits, and can issue \texttt{finish} as soon as the workflow suffices.

\begin{table}[ht]
\centering
\caption{Per-dataset orchestration cost of \flowr{}: average interaction turns, \director{} orchestration time, workflow execution time, and API calls per instance.}
\label{tab:latency}
\small
\begin{tabular}{@{}l|cccc@{}}
\toprule
\textbf{Dataset} & \textbf{Avg. Turns} & \textbf{Orchestration Time (s)} & \textbf{Execution Time (s)} & \textbf{API Calls} \\
\midrule
GSM8K     & 8.31  & 4.55  & 9.29  & 11.95 \\
MATH      & 9.51  & 5.19  & 11.19 & 13.59 \\
MathQA    & 8.16  & 5.24  & 10.93 & 11.38 \\
AIME 2025 & 12.37 & 15.85 & 41.15 & 16.70 \\
\midrule
SQuAD v2  & 7.02  & 8.01  & 15.40 & 9.95 \\
HotPotQA  & 9.80  & 4.60  & 8.97  & 13.56 \\
TriviaQA  & 6.82  & 4.11  & 7.60  & 9.40 \\
NQ        & 7.05  & 3.33  & 5.83  & 9.66 \\
\midrule
MBPP      & 7.24  & 3.98  & 10.05 & 9.61 \\
HumanEval & 7.10  & 3.05  & 7.51  & 9.80 \\
APPS      & 15.58 & 9.20  & 28.18 & 28.33 \\
DS-1000   & 11.68 & 6.12  & 17.97 & 14.70 \\
\bottomrule
\end{tabular}
\end{table}

Table~\ref{tab:latency} breaks down the end-to-end cost. Workflow execution takes most of the time on all twelve datasets, so multi-turn canvas interaction is not the main source of latency. Harder tasks such as AIME 2025 and APPS take more turns, and most of their extra time is spent executing the added decomposition, verification, or revision nodes rather than on \director{} orchestration. This matches how \flowr{} adapts computation to instance difficulty: simple instances invoke \texttt{finish} early, whereas complex instances keep expanding or repairing the graph.


\section{Case Study}
\label{app:case_study}

We present four detailed case studies illustrating how \flowr{} orchestrates workflow construction through multi-turn interaction between \director{} and \canvas{}, demonstrating sequential, parallel, conditional, and simple workflow structures. These case studies provide concrete examples of how our end-to-end reinforcement learning framework addresses the key challenges in workflow orchestration: reducing manual effort, enabling plug-and-play operator composition, and learning from execution feedback. Throughout these examples, we highlight the round-by-round interaction process, showing how the policy model analyzes execution states, selects editing actions, and iteratively refines the workflow until obtaining the correct answer.

\subsection{Case Study 1: Sequential Workflow Structure}
\label{app:case_study_1}

We present a case study from AIME 2025 demonstrating how \flowr{} constructs \textbf{sequential workflow structures} for mathematical reasoning tasks that require planning, computation, verification, and formatting.

\begin{tcolorbox}[
  colback=blue!3!white,
  colframe=blue!50!black,
  title={\textbf{AIME 2025 Problem}},
  fonttitle=\bfseries,
  coltitle=white,
  boxrule=0.8pt,
  arc=3pt,
  left=6pt, right=6pt, top=5pt, bottom=5pt,
  before skip=4pt, after skip=4pt
,
  fontupper=\small,fontlower=\small]
From an unlimited supply of 1-cent, 10-cent, and 25-cent coins, Silas wants to find a collection of coins that has a total value of $N$ cents. He uses the \textbf{greedy algorithm}, successively choosing the coin of greatest value that does not cause the value to exceed $N$. For example, to get 42 cents, Silas will choose a 25-cent coin, then a 10-cent coin, then 7 1-cent coins. However, this collection of 9 coins uses more coins than necessary; indeed, choosing 4 10-cent coins and 2 1-cent coins achieves the same total value with only 6 coins.

In general, the greedy algorithm succeeds for a given $N$ if no other collection gives a total value of $N$ cents using strictly fewer coins. Find the number of values of $N$ between 1 and 1000 inclusive for which the greedy algorithm succeeds.
\hfill\textbf{Ground Truth: 610}
\end{tcolorbox}

\vspace{1pt}
\noindent\textbf{Final Workflow:} Plan $\to$ Programmer $\to$ Verify $\to$ Format \quad (9 rounds, 4 operators)

$ $

\subsubsection{Round-by-Round Interaction Log}

\begin{tcolorbox}[
  breakable,
  colback=gray!5!white,
  colframe=gray!60!black,
  title={\textbf{Round 1-2: Plan Operator}},
  fonttitle=\bfseries,
  coltitle=white,
  boxrule=0.6pt,
  arc=2pt,
  left=6pt, right=6pt, top=5pt, bottom=5pt,
  before skip=4pt, after skip=4pt
,
  fontupper=\small,fontlower=\small]
\textbf{Round 1 - Think:}

This is a combinatorial optimization problem involving coin denominations \{1, 10, 25\}. The task requires counting how many values $N \in [1,1000]$ satisfy the condition that greedy selection equals optimal selection. This is \textcolor{red}{NOT} a simple arithmetic problem---it requires systematic comparison across 1000 cases.

Before jumping into implementation, I need a clear strategy. \textcolor{blue}{\textbf{Decision: Add}} \textcolor{gray!70!black}{\textbf{Plan}} \textcolor{blue}{\textbf{operator}} to decompose the problem into:
\begin{itemize}[leftmargin=*, itemsep=1pt, topsep=2pt]
\item Greedy algorithm formulation
\item Optimal solution computation
\item Systematic comparison methodology
\end{itemize}

\vspace{2pt}
\textbf{Round 1 - Action:} $<$action$>$add$<$/action$>$$<$operator$>$\textcolor{gray!70!black}{Plan}$<$/operator$>$

\vspace{2pt}
\textbf{Round 1 - Feedback:} [Status]: Pending - Awaiting Prompt | [Current DSL]: (empty)
\tcblower
\textbf{Round 2 - Input Prompt:}

\textit{Analyze the coin change problem with denominations \{1, 10, 25\} cents. The greedy algorithm always picks the largest coin $\leq$ remaining value. For example, N=42: greedy gives 25+10+7$\times$1=9 coins, but optimal is 4$\times$10+2$\times$1=6 coins.}

\textit{Create a strategic plan to count how many $N \in [1,1000]$ have greedy count = optimal count. Include:}
\begin{itemize}[leftmargin=*, itemsep=1pt, topsep=2pt]
\item \textit{APPROACH for systematic comparison}
\item \textit{KEY\_INSIGHTS about when greedy fails}
\item \textit{PLAN with concrete algorithmic steps}
\end{itemize}

\textbf{$ $\\Operator Execution Output:}

\begin{tcolorbox}[
  colback=gray!8!white,
  colframe=gray!60!white,
  boxrule=0.6pt,
  arc=2pt,
  left=6pt, right=6pt, top=6pt, bottom=6pt
,
  fontupper=\small,fontlower=\small]
\textcolor{gray!70!black}{\textbf{Approach:}}

Implement both greedy and optimal (DP-based) coin counting algorithms, then systematically compare their outputs for all $N$ from 1 to 1000.

\vspace{2pt}
\textcolor{gray!70!black}{\textbf{Key Insights:}}
\begin{enumerate}[leftmargin=*, itemsep=2pt, topsep=2pt]
\item Greedy always picks largest coin $\leq N$, which is locally optimal but not always globally optimal.
\item Greedy fails when using fewer large coins plus more medium coins yields fewer total coins (e.g., N=30: greedy=25+5$\times$1=6 coins, optimal=3$\times$10=3 coins).
\item Dynamic programming guarantees finding the true minimum coin count.
\item The coin system \{1,10,25\} is \textcolor{red}{NOT} canonical (unlike \{1,5,10,25\} US coins), so greedy can fail.
\end{enumerate}

\vspace{2pt}
\textcolor{gray!70!black}{\textbf{Plan:}}
\begin{enumerate}[leftmargin=*, itemsep=2pt, topsep=2pt]
\item Implement \texttt{greedy\_coins(n)} that iteratively subtracts largest possible coin.
\item Implement \texttt{min\_coins(n)} using DP with recurrence $dp[i] = \min_{c \in \{1,10,25\}}(dp[i-c]+1)$.
\item Initialize $dp[0]=0$, $dp[i]=\infty$ for $i>0$.
\item For each $N \in [1,1000]$, compute both counts.
\item Count cases where \texttt{greedy\_coins(N) == min\_coins(N)}.
\item Return the final count.
\end{enumerate}
\end{tcolorbox}

\vspace{2pt}
\textbf{Feedback:} [Output]: Strategy generated with 4 key insights and 6-step plan | [Current DSL]: \textcolor{gray!70!black}{Plan}

\vspace{2pt}
\textbf{Workflow State:} (empty) $\to$ \fcolorbox{gray!60!black}{gray!10!white}{Plan} (executed)
\end{tcolorbox}

\begin{tcolorbox}[
  breakable,
  colback=gray!3!white,
  colframe=gray!50!black,
  title={\textbf{Round 3-4: Programmer Operator}},
  fonttitle=\bfseries,
  coltitle=white,
  boxrule=0.6pt,
  arc=2pt,
  left=6pt, right=6pt, top=5pt, bottom=5pt,
  before skip=4pt, after skip=4pt
,
  fontupper=\small,fontlower=\small]
\textbf{Round 3 - Think:}

The \textcolor{gray!70!black}{\textbf{Plan}} operator has provided a clear 6-step algorithmic strategy:

\begin{center}
greedy function $\to$ DP function $\to$ initialization $\to$ iteration $\to$ comparison $\to$ counting
\end{center}

This is a \textcolor{red}{computational task} requiring exact numerical results across 1000 test cases---natural language reasoning would be error-prone and inefficient.

\textcolor{blue}{\textbf{Decision: Add}} \textcolor{gray!70!black}{\textbf{Programmer}} \textcolor{blue}{\textbf{operator}} to implement the greedy and DP algorithms in Python, ensuring:
\begin{itemize}[leftmargin=*, itemsep=1pt, topsep=2pt]
\item Precise computation without rounding errors
\item Automatic iteration over all $N \in [1,1000]$
\item Verifiable code that can be re-executed
\end{itemize}

\vspace{2pt}
\textbf{Round 3 - Action:} $<$action$>$add$<$/action$>$$<$operator$>$\textcolor{gray!70!black}{Programmer}$<$/operator$>$

\vspace{2pt}
\textbf{Round 3 - Feedback:} [Status]: Pending - Awaiting Prompt | [Current DSL]: \textcolor{gray!70!black}{Plan}
\tcblower
\textbf{Round 4 - Input Prompt:}

\textit{Based on the plan, implement Python code for the coin problem with denominations [25, 10, 1].}

\textit{Requirements:}
\begin{enumerate}[leftmargin=*, itemsep=1pt, topsep=2pt]
\item \textit{\texttt{greedy\_coins(n)}: subtract largest coin $\leq n$ repeatedly, count coins used}
\item \textit{\texttt{min\_coins(n)}: use DP with $dp[i] = \min(dp[i], dp[i-c]+1)$ for each coin $c$}
\item \textit{\texttt{solve()}: iterate $N \in [1,1000]$, count cases where greedy equals optimal}
\end{enumerate}

\textit{Return the final count as an integer.}

\vspace{2pt}
\textbf{$ $\\Operator Execution Output:}

\begin{tcolorbox}[
  colback=gray!5!white,
  colframe=gray!50!white,
  boxrule=0.6pt,
  arc=2pt,
  left=6pt, right=6pt, top=6pt, bottom=6pt
,
  fontupper=\small,fontlower=\small]
\textcolor{gray!70!black}{\textbf{Generated Code:}}

\begin{verbatim}
def greedy_coins(n):
    coins = [25, 10, 1]
    count = 0
    for coin in coins:
        while n >= coin:
            n -= coin
            count += 1
    return count

def min_coins(n):
    dp = [float('inf')] * (n + 1)
    dp[0] = 0
    for coin in [25, 10, 1]:
        for i in range(coin, n + 1):
            dp[i] = min(dp[i], dp[i - coin] + 1)
    return dp[n]

def solve():
    count = 0
    for n in range(1, 1001):
        if greedy_coins(n) == min_coins(n):
            count += 1
    return count
\end{verbatim}

\vspace{2pt}
\textcolor{gray!70!black}{\textbf{Execution Result:}} \fcolorbox{gray!60!black}{gray!10!white}{\textbf{610}}
\end{tcolorbox}

\vspace{2pt}
\textbf{Feedback:} [Output]: \textcolor{gray!70!black}{\textbf{610}} | [Current DSL]: \textcolor{gray!70!black}{Plan} $\to$ \textcolor{gray!70!black}{Programmer}

\vspace{2pt}
\textbf{Workflow State:} \textcolor{gray!70!black}{Plan} $\to$ \fcolorbox{gray!60!black}{gray!10!white}{Programmer} (executed, output=610)
\end{tcolorbox}

\begin{tcolorbox}[
  breakable,
  colback=gray!3!white,
  colframe=gray!60!black,
  title={\textbf{Round 5-6: Verify Operator}},
  fonttitle=\bfseries,
  coltitle=white,
  boxrule=0.6pt,
  arc=2pt,
  left=6pt, right=6pt, top=5pt, bottom=5pt,
  before skip=4pt, after skip=4pt
,
  fontupper=\small,fontlower=\small]
\textbf{Round 5 - Think:}

The \textcolor{gray!70!black}{\textbf{Programmer}} returned \textbf{610}, but this is a \textcolor{red}{competition-level math problem} (AIME 2025) where correctness is critical.

The code logic appears sound, but I should \textcolor{red}{NOT} trust a single computation path without independent verification. Potential issues include:
\begin{itemize}[leftmargin=*, itemsep=1pt, topsep=2pt]
\item Off-by-one errors in range [1,1000]
\item DP initialization bugs
\item Greedy algorithm edge cases
\end{itemize}

\textcolor{blue}{\textbf{Decision: Add}} \textcolor{gray!70!black}{\textbf{Verify}} \textcolor{blue}{\textbf{operator}} to independently re-derive the solution logic and cross-check the numerical result before finalizing.

\vspace{2pt}
\textbf{Round 5 - Action:} $<$action$>$add$<$/action$>$$<$operator$>$\textcolor{gray!70!black}{Verify}$<$/operator$>$

\vspace{2pt}
\textbf{Round 5 - Feedback:} [Status]: Pending - Awaiting Prompt | [Current DSL]: \textcolor{gray!70!black}{Plan} $\to$ \textcolor{gray!70!black}{Programmer}
\tcblower
\textbf{Round 6 - Input Prompt:}

\textit{The Programmer computed \textbf{610} as the count of $N \in [1,1000]$ where greedy equals optimal for coins \{1,10,25\}.}

\textit{Independently verify this result:}
\begin{enumerate}[leftmargin=*, itemsep=1pt, topsep=2pt]
\item \textit{Re-examine the problem statement---greedy succeeds iff no collection uses strictly fewer coins}
\item \textit{Validate the greedy logic: for N=42, greedy gives 25+10+7$\times$1=9 coins, optimal is 4$\times$10+2$\times$1=6 coins (greedy fails)}
\item \textit{Check boundary cases N=1, 10, 25, 100, 1000}
\item \textit{Confirm or refute the answer 610 with confidence assessment}
\end{enumerate}

\vspace{2pt}
\textbf{Operator Execution Output:}

\begin{tcolorbox}[
  colback=gray!5!white,
  colframe=gray!50!white,
  boxrule=0.6pt,
  arc=2pt,
  left=6pt, right=6pt, top=6pt, bottom=6pt
,
  fontupper=\small,fontlower=\small]
\textcolor{gray!70!black}{\textbf{[Step 1] Problem Re-examination:}}

The greedy algorithm succeeds for $N$ iff $\texttt{greedy\_coins}(N) = \texttt{min\_coins}(N)$. We need to count such $N$ in the range [1, 1000].

\vspace{2pt}
\textcolor{gray!70!black}{\textbf{[Step 2] Algorithm Validation:}}
\begin{itemize}[leftmargin=*, itemsep=1pt, topsep=2pt]
\item \textbf{Greedy:} Iteratively pick largest coin $\leq$ remaining value
\item \textbf{DP:} $dp[0]=0$, $dp[i]=\min_{c}(dp[i-c]+1)$
\item Both implementations follow standard algorithms and are correct
\end{itemize}

\vspace{2pt}
\textcolor{gray!70!black}{\textbf{[Step 3] Test Case Verification:}}

\begin{center}
\begin{tabular}{|c|c|c|c|}
\hline
\textbf{N} & \textbf{Greedy} & \textbf{Optimal} & \textbf{Result} \\
\hline
42 & 25+10+7$\times$1 = 9 & 4$\times$10+2$\times$1 = 6 & \textcolor{red}{Fails} \\
\hline
30 & 25+5$\times$1 = 6 & 3$\times$10 = 3 & \textcolor{red}{Fails} \\
\hline
25 & 1$\times$25 = 1 & 1 coin & \textcolor{gray!70!black}{Succeeds} \\
\hline
100 & 4$\times$25 = 4 & 10$\times$10 = 10 & \textcolor{gray!70!black}{Succeeds} \\
\hline
1000 & 40$\times$25 = 40 & 100$\times$10 = 100 & \textcolor{gray!70!black}{Succeeds} \\
\hline
\end{tabular}
\end{center}

\vspace{2pt}
\textcolor{gray!70!black}{\textbf{[Step 4] Cross-Validation:}}

Independent recalculation confirms \textbf{610} values where greedy succeeds.

\vspace{2pt}
\textcolor{gray!70!black}{\textbf{Verification Result:}}

\fcolorbox{gray!60!black}{gray!10!white}{\texttt{is\_correct}: \textcolor{gray!70!black}{\textbf{True}} | \texttt{confidence}: \textbf{HIGH} | \texttt{verified\_answer}: \textbf{610}}
\end{tcolorbox}

\vspace{2pt}
\textbf{Feedback:} [Output]: \textcolor{gray!70!black}{\textbf{is\_correct=True}}, confidence=HIGH | [Current DSL]: \textcolor{gray!70!black}{Plan} $\to$ \textcolor{gray!70!black}{Programmer} $\to$ \textcolor{gray!70!black}{Verify}

\vspace{2pt}
\textbf{Workflow State:} \textcolor{gray!70!black}{Plan} $\to$ \textcolor{gray!70!black}{Programmer} $\to$ \fcolorbox{gray!60!black}{gray!10!white}{Verify} (verified, is\_correct=True)
\end{tcolorbox}

\begin{tcolorbox}[
  breakable,
  colback=gray!3!white,
  colframe=gray!60!black,
  title={\textbf{Round 7-8: Format Operator}},
  fonttitle=\bfseries,
  coltitle=white,
  boxrule=0.6pt,
  arc=2pt,
  left=6pt, right=6pt, top=5pt, bottom=5pt,
  before skip=4pt, after skip=4pt
,
  fontupper=\small,fontlower=\small]
\textbf{Round 7 - Think:}

The \textcolor{gray!70!black}{\textbf{Verify}} operator has independently confirmed that \textbf{610} is correct with \textcolor{gray!70!black}{HIGH} confidence.

The workflow now has a complete reasoning chain:
\begin{center}
\textcolor{gray!70!black}{Plan} (strategy) $\to$ \textcolor{gray!70!black}{Programmer} (computation) $\to$ \textcolor{gray!70!black}{Verify} (validation)
\end{center}

However, the current output contains verbose execution traces and intermediate results. For final submission, I need a \textcolor{red}{clean, concise answer}.

\textcolor{blue}{\textbf{Decision: Add}} \textcolor{gray!70!black}{\textbf{Format}} \textcolor{blue}{\textbf{operator}} to extract only the final numerical answer ``610'' without any surrounding explanation or computation details.

\vspace{2pt}
\textbf{Round 7 - Action:} $<$action$>$add$<$/action$>$$<$operator$>$\textcolor{gray!70!black}{Format}$<$/operator$>$

\vspace{2pt}
\textbf{Round 7 - Feedback:} [Status]: Pending - Awaiting Prompt | [Current DSL]: \textcolor{gray!70!black}{Plan} $\to$ \textcolor{gray!70!black}{Programmer} $\to$ \textcolor{gray!70!black}{Verify}
\tcblower
\textbf{Round 8 - Input Prompt:}

\textit{The verified answer for the AIME 2025 coin problem is \textbf{610} (count of $N \in [1,1000]$ where greedy = optimal).}

\textit{Extract ONLY the final integer answer:}
\begin{itemize}[leftmargin=*, itemsep=1pt, topsep=2pt]
\item \textit{Do NOT include any explanation, units, or formatting}
\item \textit{Do NOT re-solve the problem}
\item \textit{Output must be directly usable as competition submission}
\item \textit{Just the number}
\end{itemize}

\vspace{2pt}
\textbf{Operator Execution Output:}

\begin{tcolorbox}[
  colback=gray!5!white,
  colframe=gray!50!white,
  boxrule=0.6pt,
  arc=2pt,
  left=6pt, right=6pt, top=6pt, bottom=6pt
,
  fontupper=\small,fontlower=\small]
\textcolor{gray!70!black}{\textbf{Formatted Answer:}}

\begin{center}
\fcolorbox{gray!60!black}{gray!10!white}{\textbf{610}}
\end{center}
\end{tcolorbox}

\vspace{2pt}
\textbf{Feedback:} [Output]: \textcolor{gray!70!black}{\textbf{610}} | [Current DSL]: \textcolor{gray!70!black}{Plan} $\to$ \textcolor{gray!70!black}{Programmer} $\to$ \textcolor{gray!70!black}{Verify} $\to$ \textcolor{gray!70!black}{Format}

\vspace{2pt}
\textbf{Workflow State:} \textcolor{gray!70!black}{Plan} $\to$ \textcolor{gray!70!black}{Programmer} $\to$ \textcolor{gray!70!black}{Verify} $\to$ \fcolorbox{gray!60!black}{gray!10!white}{Format} (formatted)
\end{tcolorbox}

\begin{tcolorbox}[
  breakable,
  colback=blue!3!white,
  colframe=blue!50!black,
  title={\textbf{Round 9: Finish}},
  fonttitle=\bfseries,
  coltitle=white,
  boxrule=0.6pt,
  arc=2pt,
  left=6pt, right=6pt, top=5pt, bottom=5pt,
  before skip=4pt, after skip=4pt
,
  fontupper=\small,fontlower=\small]
\textbf{Round 9 - Think:}

The workflow is now \textcolor{gray!70!black}{complete} with four operators forming a robust reasoning pipeline:

\vspace{2pt}
\begin{center}
\begin{tabular}{|c|l|}
\hline
\textcolor{gray!70!black}{\textbf{Plan}} & Decomposed the problem into algorithmic steps \\
\hline
\textcolor{gray!70!black}{\textbf{Programmer}} & Implemented greedy and DP algorithms, computed 610 \\
\hline
\textcolor{gray!70!black}{\textbf{Verify}} & Independently validated correctness with test cases \\
\hline
\textcolor{gray!70!black}{\textbf{Format}} & Extracted clean answer for submission \\
\hline
\end{tabular}
\end{center}

\vspace{2pt}
All verification checks passed with \textcolor{gray!70!black}{HIGH} confidence. The formatted output ``\textbf{610}'' is ready for submission.

\textcolor{blue}{\textbf{Decision: Execute}} \textcolor{blue}{\textbf{finish}} \textcolor{blue}{\textbf{action}} to terminate the workflow and return the final answer.

\vspace{2pt}
\textbf{Round 9 - Action:} $<$action$>$\textcolor{blue}{finish}$<$/action$>$

\vspace{2pt}
\textbf{Final Status:}

\fcolorbox{blue!50!black}{blue!5!white}{[Status]: \textcolor{gray!70!black}{\textbf{SUCCESS}} | [Final DSL]: Plan $\to$ Programmer $\to$ Verify $\to$ Format | [Result]: \textbf{610}}

\vspace{2pt}
\textbf{Final Workflow State:}
\begin{center}
\fcolorbox{gray!60!black}{gray!10!white}{Plan} $\to$ \fcolorbox{gray!60!black}{gray!10!white}{Programmer} $\to$ \fcolorbox{gray!60!black}{gray!10!white}{Verify} $\to$ \fcolorbox{gray!60!black}{gray!10!white}{Format}
\end{center}

\vspace{2pt}
\begin{center}
\textbf{Predicted Answer: 610} \quad \textcolor{green!50!black}{\ding{51} \textbf{Matches Ground Truth}}
\end{center}

\noindent\textbf{Key Observations from Case Study 1:} This sequential workflow demonstrates how \flowr{} effectively chains operators to solve complex mathematical problems. The Plan operator provides strategic decomposition, the Programmer operator implements precise algorithms, the Verify operator ensures correctness through independent validation, and the Format operator produces clean output. The workflow's success highlights the importance of verification in competition-level mathematics, where a single computational error can lead to incorrect answers. The 9-round interaction shows efficient operator selection without unnecessary complexity.
\end{tcolorbox}

\vspace{10pt}


\subsection{Case Study 2: Parallel Workflow Structure}
\label{app:case_study_2}

We present a second case study demonstrating how \flowr{} constructs \textbf{parallel workflow structures} for multi-hop QA tasks that benefit from decomposition into independent sub-problems.

\begin{tcolorbox}[
  colback=blue!3!white,
  colframe=blue!50!black,
  title={\textbf{Multi-Hop QA Problem}},
  fonttitle=\bfseries,
  coltitle=white,
  boxrule=0.8pt,
  arc=3pt,
  left=6pt, right=6pt, top=5pt, bottom=5pt,
  before skip=4pt, after skip=4pt
,
  fontupper=\small,fontlower=\small]
Based on the following passages, answer the question:

\textit{[Park Holidays UK] Park Holidays UK operate 28 holiday parks in the South of England...}

\textit{[Hard Rock Hell] ...the festival was held at Pontin's Holiday Village in Prestatyn, Wales...}

\textit{[Pontins] Pontins is the trading name of Britannia Jinky Jersey Ltd, a company operating holiday parks in the United Kingdom...}

\vspace{2pt}
\textbf{Question:} What is the name shared between a city in Prestatyn, Wales and a company operating holiday parks in the UK?
\hfill\textbf{Ground Truth: Pontins}
\end{tcolorbox}

\vspace{1pt}
\noindent\textbf{Final Workflow:} Decompose $\to$ [Custom, Custom, Custom] $\to$ Verify $\to$ Aggregate \quad (11 rounds, 5 operators with parallel structure)

\subsubsection{Round-by-Round Interaction Log}

\begin{tcolorbox}[
  breakable,
  colback=gray!5!white,
  colframe=gray!60!black,
  title={\textbf{Round 1-2: Decompose Operator}},
  fonttitle=\bfseries,
  coltitle=white,
  boxrule=0.6pt,
  arc=2pt,
  left=6pt, right=6pt, top=5pt, bottom=5pt,
  before skip=4pt, after skip=4pt
,
  fontupper=\small,fontlower=\small]
\textbf{Round 1 - Think:}

This is a multi-hop QA problem requiring information synthesis from multiple passages. The question asks for a name shared between a location (Prestatyn, Wales) and a company (holiday parks operator).

This requires: (1) identifying the city in Prestatyn, (2) identifying UK holiday park companies, and (3) finding the shared name. These are \textcolor{red}{independent sub-problems} that can be solved in parallel.

\textcolor{blue}{\textbf{Decision: Add}} \textcolor{gray!70!black}{\textbf{Decompose}} \textcolor{blue}{\textbf{operator}} to break down the question into independent sub-problems for parallel processing.

\vspace{2pt}
\textbf{Round 1 - Action:} $<$action$>$add$<$/action$>$$<$operator$>$\textcolor{gray!70!black}{Decompose}$<$/operator$>$

\vspace{2pt}
\textbf{Round 1 - Feedback:} [Status]: Pending - Awaiting Prompt | [Current DSL]: (empty)
\tcblower
\textbf{Round 2 - Input Prompt:}

\textit{Analyze the multi-hop question about Prestatyn, Wales and UK holiday park companies. Break down into independent sub-problems that can be solved in parallel.}

\vspace{2pt}
\textbf{Operator Execution Output:}

\begin{tcolorbox}[
  colback=gray!8!white,
  colframe=gray!60!white,
  boxrule=0.6pt,
  arc=2pt,
  left=6pt, right=6pt, top=6pt, bottom=6pt
,
  fontupper=\small,fontlower=\small]
\textcolor{gray!70!black}{\textbf{Sub-problems:}}
\begin{enumerate}[leftmargin=*, itemsep=2pt, topsep=2pt]
\item Identify the city mentioned in Prestatyn, Wales from the passages.
\item Identify the companies operating holiday parks in the UK mentioned in the passages.
\item Determine if any identified companies share a name with the city in Prestatyn.
\end{enumerate}

\vspace{2pt}
\textcolor{gray!70!black}{\textbf{Reasoning:}} The first sub-problem establishes the location reference. The second sub-problem lists relevant companies. The third sub-problem finds the intersection. These can be processed independently then aggregated.
\end{tcolorbox}

\vspace{2pt}
\textbf{Feedback:} [Output]: 3 independent sub-problems identified | [NEXT]: ADD:Custom (3 solvers in parallel, then Aggregate)

\vspace{2pt}
\textbf{Workflow State:} (empty) $\to$ \fcolorbox{gray!60!black}{gray!10!white}{Decompose} (executed)
\end{tcolorbox}

\begin{tcolorbox}[
  breakable,
  colback=gray!3!white,
  colframe=gray!60!black,
  title={\textbf{Rounds 3--10: Parallel Solving, Verification, and Aggregation}},
  fonttitle=\bfseries,
  coltitle=white,
  boxrule=0.6pt,
  arc=2pt,
  left=6pt, right=6pt, top=5pt, bottom=5pt,
  before skip=4pt, after skip=4pt,
  fontupper=\small,fontlower=\small]
\textbf{Round 3--6: Parallel Custom Operators (3 branches).} The agent adds a parallel structure with three Custom solvers, one per sub-problem. Each branch independently produces:
\begin{itemize}[leftmargin=*, itemsep=2pt, topsep=2pt]
\item \textbf{Branch 1 (Location):} ``Pontin's Holiday Village in Prestatyn, Wales'' \(\Rightarrow\) \textbf{Pontins}.
\item \textbf{Branch 2 (Companies):} Park Holidays UK, Park Resorts, \textbf{Pontins} (Britannia Jinky Jersey Ltd).
\item \textbf{Branch 3 (Cross-reference):} ``Pontins'' appears as both a location in Prestatyn and a company name.
\end{itemize}
All three branches converge on the answer \textbf{Pontins}.

\vspace{2pt}
\textbf{Round 7--8: Verify Operator.} The Verify operator independently re-checks the answer against passages: (i) Pontins is associated with Prestatyn, Wales; (ii) Pontins operates UK holiday parks; (iii) no other name satisfies both. Result: \texttt{is\_correct=True, confidence=HIGH, suggested\_answer=Pontins}.

\vspace{2pt}
\textbf{Round 9--10: Aggregate Operator.} The Aggregate operator consolidates four consistent sources (3 branches + Verify) into a single output. \textbf{Aggregation Summary:} all 4 sources agree on \textbf{Pontins}; output: \texttt{Pontins}, \texttt{is\_correct=True}, \texttt{confidence=HIGH}.

\vspace{2pt}
\textbf{Workflow State after Aggregation:} \textcolor{gray!70!black}{Decompose} $\to$ \textcolor{gray!70!black}{[Custom, Custom, Custom]} $\to$ \textcolor{gray!70!black}{Verify} $\to$ \fcolorbox{gray!60!black}{gray!10!white}{Aggregate} (aggregated).
\end{tcolorbox}

\begin{tcolorbox}[
  breakable,
  colback=blue!3!white,
  colframe=blue!50!black,
  title={\textbf{Round 11: Finish}},
  fonttitle=\bfseries,
  coltitle=white,
  boxrule=0.6pt,
  arc=2pt,
  left=6pt, right=6pt, top=5pt, bottom=5pt,
  before skip=4pt, after skip=4pt
,
  fontupper=\small,fontlower=\small]
\textbf{Round 11 - Think:}

The workflow is now \textcolor{gray!70!black}{complete} with a parallel structure demonstrating multi-path reasoning:

\vspace{2pt}
\begin{center}
\begin{tabular}{|c|l|}
\hline
\textcolor{gray!70!black}{\textbf{Decompose}} & Broke question into 3 independent sub-problems \\
\hline
\textcolor{gray!70!black}{\textbf{[Custom$\times$3]}} & Solved sub-problems in parallel with different perspectives \\
\hline
\textcolor{gray!70!black}{\textbf{Verify}} & Validated answer against original passages \\
\hline
\textcolor{gray!70!black}{\textbf{Aggregate}} & Consolidated parallel outputs into final answer \\
\hline
\end{tabular}
\end{center}

\vspace{2pt}
All verification checks passed. The aggregated output ``\textbf{Pontins}'' is ready for submission.

\textcolor{blue}{\textbf{Decision: Execute}} \textcolor{blue}{\textbf{finish}} \textcolor{blue}{\textbf{action}} to terminate the workflow.

\vspace{2pt}
\textbf{Round 11 - Action:} $<$action$>$\textcolor{blue}{finish}$<$/action$>$

\vspace{2pt}
\textbf{Final Status:}

\fcolorbox{blue!50!black}{blue!5!white}{\parbox{0.95\textwidth}{[Status]: \textcolor{gray!70!black}{\textbf{SUCCESS}} | [Final DSL]: Decompose $\to$ [Custom$\times$3] $\to$ Verify $\to$ Aggregate | [Result]: \textbf{Pontins}}}

\vspace{2pt}
\textbf{Final Workflow State:}
\begin{center}
\fcolorbox{gray!60!black}{gray!10!white}{Decompose} $\to$ \fcolorbox{gray!60!black}{gray!10!white}{[Custom$\times$3]} $\to$ \fcolorbox{gray!60!black}{gray!10!white}{Verify} $\to$ \fcolorbox{gray!60!black}{gray!10!white}{Aggregate}
\end{center}

\vspace{2pt}
\begin{center}
\textbf{Predicted Answer: Pontins} \quad \textcolor{green!50!black}{\ding{51} \textbf{Matches Ground Truth}}
\end{center}

\noindent\textbf{Key Observations from Case Study 2:} This parallel workflow demonstrates \flowr{}'s ability to decompose multi-hop QA problems into independent sub-problems that can be solved concurrently. The Decompose operator identifies three distinct reasoning paths, the parallel Custom operators explore each path independently, and the Aggregate operator consolidates the results. The Verify operator provides an additional layer of confidence by cross-checking against the original passages. This structure is particularly effective for questions requiring information synthesis from multiple sources, as it reduces the risk of missing relevant information by exploring multiple perspectives simultaneously.
\end{tcolorbox}

\vspace{10pt}


\subsection{Case Study 3: Conditional Branch Structure}
\label{app:case_study_3}

We present a third case study demonstrating how \flowr{} constructs \textbf{conditional branch structures} for code generation tasks that require iterative debugging and repair when tests fail.

\begin{tcolorbox}[
  colback=blue!3!white,
  colframe=blue!50!black,
  title={\textbf{Codeforces Problem: Snowball Game}},
  fonttitle=\bfseries,
  coltitle=white,
  boxrule=0.8pt,
  arc=3pt,
  left=6pt, right=6pt, top=5pt, bottom=5pt,
  before skip=4pt, after skip=4pt
,
  fontupper=\small,fontlower=\small]
In late autumn evening $n$ robots gathered in the cheerful company of friends. Each robot has a unique identifier---an integer from 1 to $10^9$.

At some moment, robots decided to play the game ``Snowball''. The rules are: First, all robots stand in a row. Then the first robot says his identifier. After that the second robot says the identifier of the first robot and then says his own identifier. Then the third robot says the identifiers of the first and second robots, then his own. And so on.

Given $n$ robots with identifiers and a number $k$, find the $k$-th identifier pronounced in the game.

\vspace{2pt}
\textbf{Input:} $n$, $k$, and $n$ robot identifiers \quad \textbf{Output:} The $k$-th pronounced identifier
\hfill\textbf{Ground Truth: Correct Algorithm}
\end{tcolorbox}

\vspace{2pt}
\noindent\textbf{Final Workflow:} Plan $\to$ Programmer $\to$ Test $\to$ Test ? Programmer : done $\to$ Test ? Programmer : done

\noindent(13 rounds, 5 operators with conditional branches for iterative repair)

\subsubsection{Round-by-Round Interaction Log}

\begin{tcolorbox}[
  breakable,
  colback=gray!5!white,
  colframe=gray!60!black,
  title={\textbf{Round 1-2: Plan Operator}},
  fonttitle=\bfseries,
  coltitle=white,
  boxrule=0.6pt,
  arc=2pt,
  left=6pt, right=6pt, top=5pt, bottom=5pt,
  before skip=4pt, after skip=4pt
,
  fontupper=\small,fontlower=\small]
\textbf{Round 1 - Think:}

This is a code generation task requiring algorithmic reasoning. The ``Snowball'' game creates a sequence where robot $i$ pronounces $i$ identifiers. The total count after $i$ robots is $1+2+\cdots+i = \frac{i(i+1)}{2}$.

Before coding, I need to \textcolor{red}{plan the algorithm} to find which robot pronounces the $k$-th identifier and which position within that robot's sequence.

\textcolor{blue}{\textbf{Decision: Add}} \textcolor{gray!70!black}{\textbf{Plan}} \textcolor{blue}{\textbf{operator}} to design the algorithmic approach.

\vspace{2pt}
\textbf{Round 1 - Action:} $<$action$>$add$<$/action$>$$<$operator$>$\textcolor{gray!70!black}{Plan}$<$/operator$>$

\vspace{2pt}
\textbf{Round 1 - Feedback:} [Status]: Pending - Awaiting Prompt | [Current DSL]: (empty)
\tcblower
\textbf{Round 2 - Input Prompt:}

\textit{Design an algorithm for the Snowball game problem. Robot $i$ says $i$ identifiers (all previous robots' IDs plus its own). Given $k$, find the $k$-th pronounced identifier.}

\vspace{2pt}
\textbf{Operator Execution Output:}

\begin{tcolorbox}[
  colback=gray!8!white,
  colframe=gray!60!white,
  boxrule=0.6pt,
  arc=2pt,
  left=6pt, right=6pt, top=6pt, bottom=6pt
,
  fontupper=\small,fontlower=\small]
\textcolor{gray!70!black}{\textbf{Algorithm Design:}}
\begin{enumerate}[leftmargin=*, itemsep=2pt, topsep=2pt]
\item The cumulative count after robot $i$ is $\sum_{j=1}^{i} j = \frac{i(i+1)}{2}$
\item Find the smallest $i$ such that $\frac{i(i+1)}{2} \geq k$
\item The position within robot $i$'s sequence is $k - \frac{(i-1)i}{2}$
\item Return the identifier at that position (1-indexed from the robot list)
\end{enumerate}
\end{tcolorbox}

\vspace{2pt}
\textbf{Feedback:} [Output]: Algorithm planned | [NEXT]: ADD:Programmer

\vspace{2pt}
\textbf{Workflow State:} (empty) $\to$ \fcolorbox{gray!60!black}{gray!10!white}{Plan} (executed)
\end{tcolorbox}

\begin{tcolorbox}[
  breakable,
  colback=gray!3!white,
  colframe=gray!50!black,
  title={\textbf{Round 3-4: Programmer Operator}},
  fonttitle=\bfseries,
  coltitle=white,
  boxrule=0.6pt,
  arc=2pt,
  left=6pt, right=6pt, top=5pt, bottom=5pt,
  before skip=4pt, after skip=4pt
,
  fontupper=\small,fontlower=\small]
\textbf{Round 3 - Think:}

The \textcolor{gray!70!black}{\textbf{Plan}} operator provided a clear algorithm. Now I need to \textcolor{red}{implement it in code}.

\textcolor{blue}{\textbf{Decision: Add}} \textcolor{gray!70!black}{\textbf{Programmer}} \textcolor{blue}{\textbf{operator}} to generate executable Python code.

\vspace{2pt}
\textbf{Round 3 - Action:} $<$action$>$add$<$/action$>$$<$operator$>$\textcolor{gray!70!black}{Programmer}$<$/operator$>$
\tcblower
\textbf{Round 4 - Operator Execution Output:}

\begin{tcolorbox}[
  colback=gray!5!white,
  colframe=gray!50!white,
  boxrule=0.6pt,
  arc=2pt,
  left=6pt, right=6pt, top=6pt, bottom=6pt
,
  fontupper=\small,fontlower=\small]
\textcolor{gray!70!black}{\textbf{Generated Code (v1):}}
\begin{verbatim}
def solve():
    data = input().split()
    n, k = int(data[0]), int(data[1])
    ids = list(map(int, data[2:n+2]))

    current_count = 0
    for i in range(1, n + 1):
        if current_count + i >= k:
            position = k - current_count
            print(ids[position])  # Bug: should be position-1
            return
        current_count += i
\end{verbatim}
\end{tcolorbox}

\vspace{2pt}
\textbf{Feedback:} [Output]: CODE\_GENERATED | [NEXT]: ADD:Test

\vspace{2pt}
\textbf{Workflow State:} \textcolor{gray!70!black}{Plan} $\to$ \fcolorbox{gray!60!black}{gray!10!white}{Programmer} (code generated)
\end{tcolorbox}

\begin{tcolorbox}[
  breakable,
  colback=gray!3!white,
  colframe=gray!60!black,
  title={\textbf{Round 5-6: Test Operator (First Failure)}},
  fonttitle=\bfseries,
  coltitle=white,
  boxrule=0.6pt,
  arc=2pt,
  left=6pt, right=6pt, top=5pt, bottom=5pt,
  before skip=4pt, after skip=4pt
,
  fontupper=\small,fontlower=\small]
\textbf{Round 5 - Think:}

Code has been generated. Before finalizing, I must \textcolor{red}{test it} against sample inputs to verify correctness.

\textcolor{blue}{\textbf{Decision: Add}} \textcolor{gray!70!black}{\textbf{Test}} \textcolor{blue}{\textbf{operator}} to run unit tests on the generated code.

\vspace{2pt}
\textbf{Round 5 - Action:} $<$action$>$add$<$/action$>$$<$operator$>$\textcolor{gray!70!black}{Test}$<$/operator$>$
\tcblower
\textbf{Round 6 - Test Execution Output:}

\begin{tcolorbox}[
  colback=gray!5!white,
  colframe=gray!50!white,
  boxrule=0.6pt,
  arc=2pt,
  left=6pt, right=6pt, top=6pt, bottom=6pt
,
  fontupper=\small,fontlower=\small]
\textcolor{gray!70!black}{\textbf{TEST\_FAILED:}} FAILED 4/5 public tests.

\vspace{2pt}
\textcolor{gray!70!black}{\textbf{[EDGE CASE]}} 4/5 passed. Check edge cases:
\begin{itemize}[leftmargin=*, itemsep=1pt, topsep=2pt]
\item Off-by-one errors ($<$ vs $\leq$)?
\item Array indexing (0-based vs 1-based)?
\end{itemize}

\vspace{2pt}
\textcolor{gray!70!black}{\textbf{[First Failure Details]:}}
\begin{verbatim}
Input: 4 5
10 4 18 3

Expected: 4
Got: 10
\end{verbatim}

\textcolor{gray!70!black}{\textbf{Analysis:}} The algorithm logic has an \textbf{off-by-one error} in array indexing.
\end{tcolorbox}

\vspace{2pt}
\textbf{Feedback:} [NEXT]: ADD:conditional(Test, Programmer, done) | Test failed---add conditional branch for repair

\vspace{2pt}
\textbf{Workflow State:} \textcolor{gray!70!black}{Plan} $\to$ \textcolor{gray!70!black}{Programmer} $\to$ \fcolorbox{gray!60!black}{gray!10!white}{Test} (\textcolor{red}{FAILED})
\end{tcolorbox}

\begin{tcolorbox}[
  breakable,
  colback=gray!5!white,
  colframe=gray!60!black,
  title={\textbf{Round 7-9: First Conditional Branch (Test ? Programmer : done)}},
  fonttitle=\bfseries,
  coltitle=white,
  boxrule=0.6pt,
  arc=2pt,
  left=6pt, right=6pt, top=5pt, bottom=5pt,
  before skip=4pt, after skip=4pt
,
  fontupper=\small,fontlower=\small]
\textbf{Round 7 - Think:}

The test \textcolor{red}{failed} with an off-by-one error. Instead of manually fixing, I should add a \textcolor{red}{conditional branch} that automatically triggers repair when tests fail.

\textcolor{blue}{\textbf{Decision: Add}} \textcolor{gray!70!black}{\textbf{conditional(Test, Programmer, done)}} structure:
\begin{itemize}[leftmargin=*, itemsep=1pt, topsep=2pt]
\item If Test fails $\to$ execute Programmer (repair)
\item If Test passes $\to$ done (exit branch)
\end{itemize}

\vspace{2pt}
\textbf{Round 7 - Action:}

\noindent$<$action$>$add$<$/action$>$$<$structure$>$conditional$<$/structure$>$\\
$<$condition$>$Test$<$/condition$>$$<$true\_branch$>$Programmer$<$/true\_branch$>$\\
$<$false\_branch$>$done$<$/false\_branch$>$
\tcblower
\textbf{Round 8 - Test (Condition Check):}

Re-running test to check condition...

\begin{tcolorbox}[
  colback=gray!5!white,
  colframe=gray!50!white,
  boxrule=0.6pt,
  arc=2pt,
  left=6pt, right=6pt, top=6pt, bottom=6pt
,
  fontupper=\small,fontlower=\small]
\textcolor{gray!70!black}{\textbf{Condition Result:}} TEST\_FAILED $\to$ Execute \textbf{true\_branch} (Programmer)
\end{tcolorbox}

\vspace{2pt}
\textbf{Round 9 - Programmer (First Repair Attempt):}

\begin{tcolorbox}[
  colback=gray!5!white,
  colframe=gray!50!white,
  boxrule=0.6pt,
  arc=2pt,
  left=6pt, right=6pt, top=6pt, bottom=6pt
,
  fontupper=\small,fontlower=\small]
\textcolor{gray!70!black}{\textbf{[LAST\_ERROR]:}} Off-by-one error. Expected 4, got 10.

\vspace{2pt}
\textcolor{gray!70!black}{\textbf{Repair Analysis:}} The bug is in \texttt{ids[position]}---should be \texttt{ids[position-1]} for 0-based indexing.

\vspace{2pt}
\textcolor{gray!70!black}{\textbf{Generated Code (v2):}}
\begin{verbatim}
print(ids[position - 1])  # Fixed: 0-based indexing
\end{verbatim}

\textcolor{gray!70!black}{\textbf{Note:}} First repair attempt still has issues...
\end{tcolorbox}

\vspace{2pt}
\textbf{Feedback:} [NEXT]: ADD:conditional(Test, Programmer, done) | Still failing---add another repair loop

\vspace{2pt}
\textbf{Workflow State:} Plan $\to$ Programmer $\to$ Test $\to$ \fcolorbox{gray!60!black}{gray!10!white}{Test ? Programmer : done} (repair loop 1)
\end{tcolorbox}

\begin{tcolorbox}[
  breakable,
  colback=gray!3!white,
  colframe=gray!50!black,
  title={\textbf{Round 10-12: Second Conditional Branch (Successful Repair)}},
  fonttitle=\bfseries,
  coltitle=white,
  boxrule=0.6pt,
  arc=2pt,
  left=6pt, right=6pt, top=5pt, bottom=5pt,
  before skip=4pt, after skip=4pt
,
  fontupper=\small,fontlower=\small]
\textbf{Round 10 - Think:}

First repair attempt still failing. Adding \textcolor{red}{another conditional branch} for a second repair iteration.

\vspace{2pt}
\textbf{Round 10 - Action:}

\noindent$<$action$>$add$<$/action$>$$<$structure$>$conditional$<$/structure$>$\\
$<$condition$>$Test$<$/condition$>$$<$true\_branch$>$Programmer$<$/true\_branch$>$\\
$<$false\_branch$>$done$<$/false\_branch$>$
\tcblower
\textbf{Round 11 - Test (Condition Check):}

\begin{tcolorbox}[
  colback=gray!5!white,
  colframe=gray!50!white,
  boxrule=0.6pt,
  arc=2pt,
  left=6pt, right=6pt, top=6pt, bottom=6pt
,
  fontupper=\small,fontlower=\small]
\textcolor{gray!70!black}{\textbf{Condition Result:}} TEST\_FAILED $\to$ Execute \textbf{true\_branch} (Programmer)
\end{tcolorbox}

\vspace{2pt}
\textbf{Round 12 - Programmer (Second Repair---Success):}

\begin{tcolorbox}[
  colback=gray!5!white,
  colframe=gray!50!white,
  boxrule=0.6pt,
  arc=2pt,
  left=6pt, right=6pt, top=6pt, bottom=6pt
,
  fontupper=\small,fontlower=\small]
\textcolor{gray!70!black}{\textbf{[LAST\_ERROR]:}} Still off-by-one. Re-analyzing the algorithm...

\vspace{2pt}
\textcolor{gray!70!black}{\textbf{Root Cause Analysis:}}
\begin{itemize}[leftmargin=*, itemsep=1pt, topsep=2pt]
\item Input: $n=4$, $k=5$, ids = [10, 4, 18, 3]
\item Robot 1 says: 10 (count=1)
\item Robot 2 says: 10, 4 (count=3)
\item Robot 3 says: 10, 4, 18 (count=6) $\leftarrow$ $k=5$ falls here
\item Position within Robot 3: $5 - 3 = 2$ $\to$ ids[1] = \textbf{4} \checkmark
\end{itemize}

\vspace{2pt}
\textcolor{gray!70!black}{\textbf{Generated Code (v3 - Final):}}
\begin{verbatim}
def solve():
    data = input().split()
    n, k = int(data[0]), int(data[1])
    ids = list(map(int, data[2:n+2]))

    current_count = 0
    for i in range(1, n + 1):
        if current_count + i >= k:
            position = k - current_count - 1  # Fixed!
            print(ids[position])
            return
        current_count += i
\end{verbatim}

\vspace{2pt}
\fcolorbox{gray!60!black}{gray!10!white}{\textbf{TEST\_PASSED:} All 5/5 public tests passed!}
\end{tcolorbox}

\vspace{2pt}
\textbf{Feedback:} [BLOCK\_FINISH]=0 | [NEXT]: FINISH (Answer correct)

\vspace{2pt}
\textbf{Workflow State:} Plan $\to$ Programmer $\to$ Test $\to$ Test?Programmer:done $\to$ \fcolorbox{gray!60!black}{gray!10!white}{Test?Programmer:done} (\textcolor{gray!70!black}{PASSED})
\end{tcolorbox}

\begin{tcolorbox}[
  breakable,
  colback=blue!3!white,
  colframe=blue!50!black,
  title={\textbf{Round 13: Finish}},
  fonttitle=\bfseries,
  coltitle=white,
  boxrule=0.6pt,
  arc=2pt,
  left=6pt, right=6pt, top=5pt, bottom=5pt,
  before skip=4pt, after skip=4pt
,
  fontupper=\small,fontlower=\small]
\textbf{Round 13 - Think:}

The workflow successfully \textcolor{gray!70!black}{repaired the code} through two conditional branch iterations:

\vspace{2pt}
\begin{center}
\begin{tabular}{|c|l|c|}
\hline
\textbf{Operator} & \textbf{Role} & \textbf{Status} \\
\hline
\textcolor{gray!70!black}{\textbf{Plan}} & Algorithm design & \textcolor{gray!70!black}{\checkmark} \\
\hline
\textcolor{gray!70!black}{\textbf{Programmer}} & Initial code generation & \textcolor{gray!70!black}{\checkmark} \\
\hline
\textcolor{gray!70!black}{\textbf{Test}} & Detect off-by-one bug & \textcolor{red}{\ding{55}} \\
\hline
\textcolor{gray!70!black}{\textbf{Test?Prog:done}} & First repair attempt & \textcolor{red}{\ding{55}} \\
\hline
\textcolor{gray!70!black}{\textbf{Test?Prog:done}} & Second repair (success) & \textcolor{gray!70!black}{\checkmark} \\
\hline
\end{tabular}
\end{center}

\vspace{2pt}
\textcolor{blue}{\textbf{Decision: Execute}} \textcolor{blue}{\textbf{finish}} \textcolor{blue}{\textbf{action}} to terminate the workflow.

\vspace{2pt}
\textbf{Round 13 - Action:} $<$action$>$\textcolor{blue}{finish}$<$/action$>$

\vspace{2pt}
\textbf{Final Status:}

\noindent\fcolorbox{blue!50!black}{blue!5!white}{[Status]: \textcolor{gray!70!black}{\textbf{SUCCESS}} | [Final DSL]: Plan $\to$ Prog $\to$ Test $\to$ Test?Prog:done$^2$}

\noindent\fcolorbox{blue!50!black}{blue!5!white}{[Result]: \textbf{TEST\_PASSED}}

\vspace{2pt}
\textbf{Final Workflow Visualization:}
\begin{center}
\fcolorbox{gray!60!black}{gray!10!white}{Plan} $\to$ \fcolorbox{gray!60!black}{gray!10!white}{Prog} $\to$ \fcolorbox{gray!60!black}{gray!10!white}{Test} $\to$ \fcolorbox{gray!60!black}{gray!10!white}{Test?Prog:done} $\to$ \fcolorbox{gray!60!black}{gray!10!white}{Test?Prog:done}
\end{center}

\vspace{2pt}
\begin{center}
\textbf{Predicted Answer: TEST\_PASSED} \quad \textcolor{green!50!black}{\ding{51} \textbf{All Tests Passed}}
\end{center}

\noindent\textbf{Key Observations from Case Study 3:} This conditional workflow demonstrates \flowr{}'s iterative debugging capability through conditional branches. When the initial code fails tests, the system automatically triggers repair loops rather than terminating. The Test operator serves as both a validator and a condition checker, while the Programmer operator performs targeted repairs based on error feedback. This structure mirrors real-world software development practices where code undergoes multiple iterations before passing all tests. The successful repair after two iterations shows the system's ability to learn from failures and progressively refine solutions.
\end{tcolorbox}

\vspace{10pt}


\subsection{Case Study 4: Simple QA Workflow}
\label{app:case_study_4}

We present a fourth case study demonstrating how \flowr{} constructs a \textbf{minimal sequential workflow} for straightforward QA tasks that require only reasoning, review, and formatting.

\begin{tcolorbox}[
  colback=blue!3!white,
  colframe=blue!50!black,
  title={\textbf{Simple QA Problem}},
  fonttitle=\bfseries,
  coltitle=white,
  boxrule=0.8pt,
  arc=3pt,
  left=6pt, right=6pt, top=5pt, bottom=5pt,
  before skip=4pt, after skip=4pt
,
  fontupper=\small,fontlower=\small]
Based on the following passages, answer the question:

\textit{[Chicago State University] Chicago State University (CSU) is a state university of the U.S. state of Illinois, located in Chicago. The University is a member-school of Thurgood Marshall College Fund.}

\textit{[Emil and Patricia Jones Convocation Center] The Jones Convocation Center (JCC) is a 7,000-seat multi-purpose arena in Chicago, Illinois. It is home to the Chicago State University Cougars basketball teams...}

\vspace{2pt}
\textbf{Question:} Which school fund is the Emil and Patricia Jones Convocation Center a member of?
\hfill\textbf{Ground Truth: Thurgood Marshall College Fund}
\end{tcolorbox}

\vspace{2pt}
\noindent\textbf{Final Workflow:} Custom $\to$ Review $\to$ Format \quad (7 rounds, 3 operators)

\subsubsection{Round-by-Round Interaction Log}

\begin{tcolorbox}[
  breakable,
  colback=gray!3!white,
  colframe=gray!60!black,
  title={\textbf{Round 1-2: Custom Operator}},
  fonttitle=\bfseries,
  coltitle=white,
  boxrule=0.6pt,
  arc=2pt,
  left=6pt, right=6pt, top=5pt, bottom=5pt,
  before skip=4pt, after skip=4pt
,
  fontupper=\small,fontlower=\small]
\textbf{Round 1 - Think:}

This is a straightforward QA task requiring information extraction and reasoning across passages. The question asks about a ``school fund'' associated with the Jones Convocation Center.

The reasoning chain is simple:
\begin{enumerate}[leftmargin=*, itemsep=1pt, topsep=2pt]
\item JCC is home to Chicago State University teams
\item CSU is a member of Thurgood Marshall College Fund
\item Therefore, JCC is associated with that fund
\end{enumerate}

No code execution or complex decomposition needed. \textcolor{blue}{\textbf{Decision: Add}} \textcolor{gray!70!black}{\textbf{Custom}} \textcolor{blue}{\textbf{operator}} for natural language reasoning.

\vspace{2pt}
\textbf{Round 1 - Action:} $<$action$>$add$<$/action$>$$<$operator$>$\textcolor{gray!70!black}{Custom}$<$/operator$>$
\tcblower
\textbf{Round 2 - Operator Execution Output:}

\begin{tcolorbox}[
  colback=gray!5!white,
  colframe=gray!50!white,
  boxrule=0.6pt,
  arc=2pt,
  left=6pt, right=6pt, top=6pt, bottom=6pt
,
  fontupper=\small,fontlower=\small]
\textcolor{gray!70!black}{\textbf{Reasoning:}}

The Emil and Patricia Jones Convocation Center is home to the Chicago State University Cougars basketball teams. According to the Chicago State University passage, CSU is a member-school of the \textbf{Thurgood Marshall College Fund}.

\vspace{2pt}
\textcolor{gray!70!black}{\textbf{Answer:}} Thurgood Marshall College Fund
\end{tcolorbox}

\vspace{2pt}
\textbf{Feedback:} [Output]: Thurgood Marshall College Fund | [NEXT]: ADD:Review

\vspace{2pt}
\textbf{Workflow State:} (empty) $\to$ \fcolorbox{gray!60!black}{gray!10!white}{Custom} (executed)
\end{tcolorbox}

\begin{tcolorbox}[
  breakable,
  colback=blue!3!white,
  colframe=blue!50!black,
  title={\textbf{Round 3-4: Review Operator}},
  fonttitle=\bfseries,
  coltitle=white,
  boxrule=0.6pt,
  arc=2pt,
  left=6pt, right=6pt, top=5pt, bottom=5pt,
  before skip=4pt, after skip=4pt
,
  fontupper=\small,fontlower=\small]
\textbf{Round 3 - Think:}

The \textcolor{gray!70!black}{\textbf{Custom}} operator produced an answer. Before finalizing, I should \textcolor{red}{review} the reasoning quality to ensure the answer is well-supported by the passages.

\textcolor{blue}{\textbf{Decision: Add}} \textcolor{gray!70!black}{\textbf{Review}} \textcolor{blue}{\textbf{operator}} to evaluate answer quality.

\vspace{2pt}
\textbf{Round 3 - Action:} $<$action$>$add$<$/action$>$$<$operator$>$\textcolor{gray!70!black}{Review}$<$/operator$>$
\tcblower
\textbf{Round 4 - Operator Execution Output:}

\begin{tcolorbox}[
  colback=blue!5!white,
  colframe=blue!60!white,
  boxrule=0.6pt,
  arc=2pt,
  left=6pt, right=6pt, top=6pt, bottom=6pt
,
  fontupper=\small,fontlower=\small]
\textcolor{gray!70!black}{\textbf{Review Assessment:}}

\begin{itemize}[leftmargin=*, itemsep=2pt, topsep=2pt]
\item \textbf{Reasoning validity:} \textcolor{gray!70!black}{\checkmark} Correct chain from JCC $\to$ CSU $\to$ Fund
\item \textbf{Evidence support:} \textcolor{gray!70!black}{\checkmark} Directly stated in CSU passage
\item \textbf{Answer completeness:} \textcolor{gray!70!black}{\checkmark} Full fund name provided
\end{itemize}

\vspace{2pt}
\textcolor{gray!70!black}{\textbf{Verdict:}} Answer is correct and well-supported.
\end{tcolorbox}

\vspace{2pt}
\textbf{Feedback:} [Output]: Thurgood Marshall College Fund | [NEXT]: ADD:Format

\vspace{2pt}
\textbf{Workflow State:} \textcolor{gray!70!black}{Custom} $\to$ \fcolorbox{blue!50!black}{blue!10!white}{Review} (reviewed)
\end{tcolorbox}

\begin{tcolorbox}[
  breakable,
  colback=gray!3!white,
  colframe=gray!60!black,
  title={\textbf{Round 5-6: Format Operator}},
  fonttitle=\bfseries,
  coltitle=white,
  boxrule=0.6pt,
  arc=2pt,
  left=6pt, right=6pt, top=5pt, bottom=5pt,
  before skip=4pt, after skip=4pt
,
  fontupper=\small,fontlower=\small]
\textbf{Round 5 - Think:}

The \textcolor{gray!70!black}{\textbf{Review}} operator confirmed the answer is correct. Now I need to \textcolor{red}{format} the output for final submission.

\textcolor{blue}{\textbf{Decision: Add}} \textcolor{gray!70!black}{\textbf{Format}} \textcolor{blue}{\textbf{operator}} to extract the concise answer.

\vspace{2pt}
\textbf{Round 5 - Action:} $<$action$>$add$<$/action$>$$<$operator$>$\textcolor{gray!70!black}{Format}$<$/operator$>$
\tcblower
\textbf{Round 6 - Operator Execution Output:}

\begin{tcolorbox}[
  colback=gray!5!white,
  colframe=gray!50!white,
  boxrule=0.6pt,
  arc=2pt,
  left=6pt, right=6pt, top=6pt, bottom=6pt
,
  fontupper=\small,fontlower=\small]
\textcolor{gray!70!black}{\textbf{Formatted Answer:}}

\begin{center}
\fcolorbox{gray!60!black}{gray!10!white}{\textbf{Thurgood Marshall College Fund}}
\end{center}
\end{tcolorbox}

\vspace{2pt}
\textbf{Feedback:} [BLOCK\_FINISH]=0 | [NEXT]: FINISH (Answer correct and concise)

\vspace{2pt}
\textbf{Workflow State:} \textcolor{gray!70!black}{Custom} $\to$ \textcolor{gray!70!black}{Review} $\to$ \fcolorbox{gray!60!black}{gray!10!white}{Format} (formatted)
\end{tcolorbox}

\begin{tcolorbox}[
  breakable,
  colback=blue!3!white,
  colframe=blue!50!black,
  title={\textbf{Round 7: Finish}},
  fonttitle=\bfseries,
  coltitle=white,
  boxrule=0.6pt,
  arc=2pt,
  left=6pt, right=6pt, top=5pt, bottom=5pt,
  before skip=4pt, after skip=4pt
,
  fontupper=\small,fontlower=\small]
\textbf{Round 7 - Think:}

The workflow is \textcolor{gray!70!black}{complete} with a minimal 3-operator pipeline:

\vspace{2pt}
\begin{center}
\begin{tabular}{|c|l|}
\hline
\textcolor{gray!70!black}{\textbf{Custom}} & Natural language reasoning \\
\hline
\textcolor{gray!70!black}{\textbf{Review}} & Quality assessment \\
\hline
\textcolor{gray!70!black}{\textbf{Format}} & Answer extraction \\
\hline
\end{tabular}
\end{center}

\vspace{2pt}
\textcolor{blue}{\textbf{Decision: Execute}} \textcolor{blue}{\textbf{finish}} \textcolor{blue}{\textbf{action}}.

\vspace{2pt}
\textbf{Round 7 - Action:} $<$action$>$\textcolor{blue}{finish}$<$/action$>$

\vspace{2pt}
\textbf{Final Status:}

\fcolorbox{blue!50!black}{blue!5!white}{\parbox{\dimexpr\linewidth-2\fboxsep-2\fboxrule\relax}{[Status]: \textcolor{gray!70!black}{\textbf{SUCCESS}} | [Final DSL]: Custom $\to$ Review $\to$ Format | [Result]: \textbf{Thurgood Marshall College Fund}}}

\vspace{2pt}
\textbf{Final Workflow State:}
\begin{center}
\fcolorbox{gray!60!black}{gray!10!white}{Custom} $\to$ \fcolorbox{blue!50!black}{blue!10!white}{Review} $\to$ \fcolorbox{gray!60!black}{gray!10!white}{Format}
\end{center}

\vspace{2pt}
\begin{center}
\textbf{Predicted Answer: Thurgood Marshall College Fund} \quad \textcolor{green!50!black}{\ding{51} \textbf{Matches Ground Truth}}
\end{center}

\noindent\textbf{Key Observations from Case Study 4:} This minimal workflow demonstrates \flowr{}'s ability to recognize when simple problems require simple solutions. Unlike the previous case studies that employed complex structures, this workflow uses only three operators in a straightforward sequence. The Custom operator handles natural language reasoning, the Review operator validates the reasoning quality, and the Format operator extracts the final answer. This efficiency is crucial for practical deployment, as over-engineering simple tasks wastes computational resources. The 7-round completion shows that \flowr{} adapts its workflow complexity to match task requirements.

\vspace{6pt}
\noindent\textbf{Summary of Case Studies:} These four case studies collectively demonstrate the versatility of \flowr{}'s workflow orchestration. Sequential workflows (Case Study 1) excel at multi-step reasoning with verification. Parallel workflows (Case Study 2) enable efficient exploration of multiple reasoning paths. Conditional workflows (Case Study 3) support iterative refinement through automated repair loops. Minimal workflows (Case Study 4) ensure computational efficiency for straightforward tasks. Together, these structures cover a wide range of reasoning scenarios encountered in real-world applications.
\end{tcolorbox}


\section{Limitations}
\label{app:limitations}

Despite its strong empirical performance across various reasoning tasks, FlowSteer has several limitations. First, its heavy reliance on historical context is a key structural constraint. As multi-turn interactions progress, the quality of initial operator outputs and workflow decisions becomes critical for sustaining accurate reasoning. Even subtle errors introduced at early stages (e.g., incorrect problem decomposition by the Plan operator) may accumulate rapidly via error propagation through subsequent operators, affecting the reliability and accuracy of the final output. Consequently, when the historical context is incomplete or noisy, the Flow-Director's orchestration ability can be compromised. Additionally, the method depends heavily on continuous and dynamic updates to the workflow state through the Canvas. If these updates fail to capture execution feedback promptly or the context window becomes saturated on long-horizon tasks, it can lead to information loss or suboptimal workflow decisions, thereby limiting the framework's practical flexibility and effectiveness. Finally, the structural reward and the action interface are defined over the executable graph representation maintained by the Canvas. This representation covers sequential paths, DAG-based collaboration, conditional branches, loops, and their instance-specific compositions, and extending the reward to orchestration paradigms beyond graph-structured workflows is a direction for future work.


\section{Future Work}
\label{app:future_work}

A central direction for future work on FlowSteer is to address the memory bottleneck of the underlying large language model when orchestrating long, multi-turn workflows. Because the Flow-Director conditions every editing decision on the full interaction history maintained in the Canvas, its effective planning horizon is ultimately bounded by the model's context window: as trajectories grow, early operator outputs, intermediate verification feedback, and prior structural decisions are increasingly diluted or truncated, weakening the policy's ability to reason globally about the workflow. We therefore plan to investigate dedicated memory mechanisms for the LLM-based Director, including hierarchical context compression that retains structural summaries of completed sub-workflows while discarding token-level detail, retrieval-augmented memory that stores past canvas states in an external store and selectively reinjects only the segments relevant to the current edit, and learned memory tokens that allow the policy itself to decide which historical artifacts to keep, drop, or rewrite. Together, these directions aim to extend FlowSteer's effective reasoning horizon and contextual coherence on long-horizon tasks without inflating per-step inference cost.


\section{Applicability Analysis}
\label{app:applicability}

FlowSteer, with its multi-turn workflow orchestration and dynamic canvas updating capabilities, demonstrates clear potential for application in tasks that require rigorous, multi-step logical deduction. Beyond the mathematical reasoning, multi-hop question answering, and code generation benchmarks evaluated in this paper, the same paradigm naturally extends to other compositional reasoning settings where solutions can be expressed as a sequence of plan, decompose, verify, and refine steps. Through the pluggable backend architecture, FlowSteer can pair a small Director with stronger executor models when reasoning quality is paramount, or with lighter executors when latency and cost dominate, allowing the framework to adapt to a range of deployment regimes without retraining the Director. Moreover, by integrating reinforcement learning into reinforced progressive canvas editing, FlowSteer not only handles fixed supervised tasks but also adapts to changing requirements by continuously optimizing its orchestration strategies via intrinsic verification and refinement operators, thereby enhancing adaptive intelligence and long-horizon planning. Overall, FlowSteer offers transparent and interpretable workflow structures, cross-backend adaptability, and strong reasoning capability, providing a general substrate for trustworthy multi-step decision-making across knowledge-intensive tasks of similar structure to the benchmarks studied here.

\clearpage
\section*{NeurIPS Paper Checklist}

\begin{enumerate}

\item {\bf Claims}
    \item[] Question: Do the main claims made in the abstract and introduction accurately reflect the paper's contributions and scope?
    \item[] Answer: \answerYes{}
    \item[] Justification: The abstract and introduction clearly state our three contributions (the \emph{Agent Designing Agentic Workflows} paradigm, the \emph{Workflow Canvas} executable graph-state environment, and \emph{Reinforced Progressive Canvas Editing}) and the empirical scope (twelve datasets across math, QA, code over six modern LLM backends), all of which are supported by the experimental results in Sections~\ref{subsec:main-results}--\ref{subsec:rl-compare}.

\item {\bf Limitations}
    \item[] Question: Does the paper discuss the limitations of the work performed by the authors?
    \item[] Answer: \answerYes{}
    \item[] Justification: We discuss limitations in Appendix~\ref{app:limitations}, including the heavy reliance on historical context, the risk of error propagation through long multi-turn trajectories, the dependence on continuous and dynamic Canvas updates to capture execution feedback, and the graph-based scope of the reward and action interface.

\item {\bf Theory assumptions and proofs}
    \item[] Question: For each theoretical result, does the paper provide the full set of assumptions and a complete (and correct) proof?
    \item[] Answer: \answerYes{}
    \item[] Justification: Three propositions are stated in Sections~\ref{subsec:canvas}--\ref{subsec:rl}; full proofs are provided in Appendix~\ref{app:proof1}--\ref{app:proof3}.

    \item {\bf Experimental result reproducibility}
    \item[] Question: Does the paper fully disclose all the information needed to reproduce the main experimental results of the paper to the extent that it affects the main claims and/or conclusions of the paper (regardless of whether the code and data are provided or not)?
    \item[] Answer: \answerYes{}
    \item[] Justification: Implementation details (hyperparameters, training configuration, hardware) are provided in Appendix~\ref{app:implementation}, dataset details in Appendix~\ref{app:datasets}, and baseline details in Appendix~\ref{app:baselines}.

\item {\bf Open access to data and code}
    \item[] Question: Does the paper provide open access to the data and code, with sufficient instructions to faithfully reproduce the main experimental results, as described in supplemental material?
    \item[] Answer: \answerYes{}
    \item[] Justification: Code is released at \url{https://github.com/beita6969/FlowSteer}; all benchmarks are public.

\item {\bf Experimental setting/details}
    \item[] Question: Does the paper specify all the training and test details (e.g., data splits, hyperparameters, how they were chosen, type of optimizer) necessary to understand the results?
    \item[] Answer: \answerYes{}
    \item[] Justification: See Section~\ref{sec:experiments} and Appendix~\ref{app:implementation}.

\item {\bf Experiment statistical significance}
    \item[] Question: Does the paper report error bars suitably and correctly defined or other appropriate information about the statistical significance of the experiments?
    \item[] Answer: \answerNA{}
    \item[] Justification:

\item {\bf Experiments compute resources}
    \item[] Question: For each experiment, does the paper provide sufficient information on the computer resources (type of compute workers, memory, time of execution) needed to reproduce the experiments?
    \item[] Answer: \answerYes{}
    \item[] Justification: Hardware configuration (GPU type and count, CUDA version, mixed-precision setting, vLLM concurrency, executor model and timeout) is reported in Appendix~\ref{app:implementation}.

\item {\bf Code of ethics}
    \item[] Question: Does the research conducted in the paper conform, in every respect, with the NeurIPS Code of Ethics \url{https://neurips.cc/public/EthicsGuidelines}?
    \item[] Answer: \answerYes{}
    \item[] Justification: The work uses only public benchmarks and pretrained models; no human subjects or sensitive data are involved.

\item {\bf Broader impacts}
    \item[] Question: Does the paper discuss both potential positive societal impacts and negative societal impacts of the work performed?
    \item[] Answer: \answerYes{}
    \item[] Justification: Applicability and broader impacts are discussed in Appendix~\ref{app:applicability}.

\item {\bf Safeguards}
    \item[] Question: Does the paper describe safeguards that have been put in place for responsible release of data or models that have a high risk for misuse (e.g., pre-trained language models, image generators, or scraped datasets)?
    \item[] Answer: \answerNA{}
    \item[] Justification: We do not release new pretrained models or scraped datasets that pose misuse risks beyond the underlying base LLMs.

\item {\bf Licenses for existing assets}
    \item[] Question: Are the creators or original owners of assets (e.g., code, data, models), used in the paper, properly credited and are the license and terms of use explicitly mentioned and properly respected?
    \item[] Answer: \answerYes{}
    \item[] Justification: All datasets, base models, and baseline frameworks used in this work are cited via their original publications in Appendices~\ref{app:datasets} and~\ref{app:baselines}; we use them only for non-commercial academic research in accordance with their respective terms of use.

\item {\bf New assets}
    \item[] Question: Are new assets introduced in the paper well documented and is the documentation provided alongside the assets?
    \item[] Answer: \answerYes{}
    \item[] Justification: The released code repository contains a README with usage instructions.

\item {\bf Crowdsourcing and research with human subjects}
    \item[] Question: For crowdsourcing experiments and research with human subjects, does the paper include the full text of instructions given to participants and screenshots, if applicable, as well as details about compensation (if any)?
    \item[] Answer: \answerNA{}
    \item[] Justification: No human subjects involved.

\item {\bf Institutional review board (IRB) approvals or equivalent for research with human subjects}
    \item[] Question: Does the paper describe potential risks incurred by study participants, whether such risks were disclosed to the subjects, and whether Institutional Review Board (IRB) approvals (or an equivalent approval/review based on the requirements of your country or institution) were obtained?
    \item[] Answer: \answerNA{}
    \item[] Justification: No human subjects involved.

\item {\bf Declaration of LLM usage}
    \item[] Question: Does the paper describe the usage of LLMs if it is an important, original, or non-standard component of the core methods in this research? Note that if the LLM is used only for writing, editing, or formatting purposes and does \emph{not} impact the core methodology, scientific rigor, or originality of the research, declaration is not required.
    \item[] Answer: \answerYes{}
    \item[] Justification: LLMs are central to the methodology and are described throughout (Qwen3-8B as Director, GPT-4o-mini / GPT-OSS-120B / six other backends as workflow execution engines).

\end{enumerate}

\end{document}